\documentclass{article} 
\usepackage{iclr2026_conference,times}
\usepackage[table]{xcolor}
\usepackage{colortbl}  
\usepackage{algorithm}
\usepackage{algpseudocode}

\usepackage{latexsym,mathtools,mathrsfs,amssymb,amsmath,amsbsy,amsopn,amstext,amsthm,amsxtra,bm,bbm,calc,ifpdf,cancel,extarrows,manfnt,cancel}

\usepackage{hyperref}
\usepackage[capitalise]{cleveref}

\def\eqref#1{equation~\ref{#1}}
\def\Eqref#1{Equation~\ref{#1}}

\def\1{\bm{1}}

\def\rvc{{\mathbf{c}}}

\def\rvo{{\mathbf{o}}}

\def\rvx{{\mathbf{x}}}
\def\rvy{{\mathbf{y}}}

\DeclareMathAlphabet{\mathsfit}{\encodingdefault}{\sfdefault}{m}{sl}
\SetMathAlphabet{\mathsfit}{bold}{\encodingdefault}{\sfdefault}{bx}{n}

\newcommand{\E}{\mathbb{E}}

\newcommand{\R}{\mathbb{R}}

\newcommand{\mask}{\mathbbm{m}}

\newcommand{\dd}{\text{d}}
\newcommand{\bsl}{\backslash}

\newcommand{\norm}[1]{\|{#1}\|} 
\newcommand{\ltwo}[1]{\norm{#1}_2} 

\newcommand{\abs}[1]{{| #1 |}}
\newcommand{\paren}[1]{{( #1 )}}
\newcommand{\brac}[1]{{[ #1 ]}}
\newcommand{\set}[1]{{\{ #1 \}}}
\newcommand{\<}{\langle}
\renewcommand{\>}{\rangle}

\newcommand{\parenBig}[1]{{\Big( #1 \Big)}}
\newcommand{\bracBig}[1]{{\Big[ #1 \Big]}}
\newcommand{\setBig}[1]{{\Big\{ #1 \Big\}}}

\renewcommand{\P}{\mathbb{P}}

\newcommand{\mc}[1]{\mathcal{#1}}

\usepackage{pifont}

\newtheorem{theorem}{Theorem}

\newtheorem{remark}{Remark}

\newtheorem{proposition}{Proposition}

\newtheorem{assumption}{Assumption}

\crefname{theorem}{Theorem}{Theorems}
\crefname{lemma}{Lemma}{Lemmas}
\crefname{remark}{Remark}{Remarks}
\crefname{corollary}{Corollary}{Corollaries}
\crefname{observation}{Observation}{Observations}
\crefname{proposition}{Proposition}{Propositions}
\crefname{definition}{Definition}{Definitions}
\crefname{claim}{Claim}{Claims}
\crefname{fact}{Fact}{Facts}
\crefname{assumption}{Assumption}{Assumptions}
\crefname{example}{Example}{Examples}
\crefname{conjecture}{Conjecture}{Conjectures}

\usepackage{booktabs}
\usepackage{threeparttable}
\usepackage{array}
\usepackage{multirow}
\usepackage{caption}
\usepackage{graphicx}
\usepackage{booktabs}        
\usepackage{array}             
\usepackage{threeparttable} 
\usepackage{tabularx}
\usepackage{wrapfig}
\usepackage{subcaption}
\usepackage{makecell}

\title{Discrete Guidance Matching: Exact Guidance for Discrete Flow Matching} 

\author{Zhengyan Wan\textsuperscript{\textbf{*}1}\And
Yidong Ouyang\textsuperscript{\textbf{*}2} \And
Liyan~Xie\textsuperscript{3} \And
Fang~Fang\textsuperscript{\rm \dag 1} \And
Hongyuan~Zha\textsuperscript{\rm \dag 4} \And
Guang~Cheng\textsuperscript{2} \vspace{5pt}\\ 
\textbf{\textsuperscript{*} Equal contribution} \quad \textbf{\textsuperscript{\dag} Corresponding author} \\
\textsuperscript{1} School of Statistics, East China Normal University \\
\textsuperscript{2} Department of Statistics, University of California, Los Angeles \\
\textsuperscript{3} Department of Industrial and Systems Engineering, University of Minnesota \\
\textsuperscript{4} School of Data Science, The Chinese University of Hong Kong, Shenzhen \\
}

\iclrfinalcopy 
\begin{document}

\maketitle


\begin{abstract}
Guidance provides a simple and effective framework for posterior sampling by steering the generation process towards the desired distribution. When modeling discrete data, existing approaches mostly focus on guidance with the first-order approximation to improve the sampling efficiency. However, such an approximation is inappropriate in discrete state spaces since the approximation error could be large. A novel guidance framework for discrete data is proposed to address this problem: we derive the exact transition rate for the desired distribution given a learned discrete flow matching model, leading to guidance that only requires a single forward pass in each sampling step, significantly improving efficiency. This unified novel framework is general enough, encompassing existing guidance methods as special cases, and it can also be seamlessly applied to the masked diffusion model. We demonstrate the effectiveness of our proposed guidance on energy-guided simulations and preference alignment on text-to-image generation and multimodal understanding tasks. The code is available at \url{https://github.com/WanZhengyan/Discrete-Guidance-Matching}.


\end{abstract}
\section{Introduction}
Discrete diffusion models \citep{austin2021structured,campbell2022continuous,sun2022score,lou2023discrete} and discrete flow-based models \citep{campbell2024generative,gat2024discrete,shaul2024flow,yimingdefog} have received significant attention for generating samples in discrete state spaces, providing effective alternatives to autoregressive (AR) models. Additionally, several works aim to study how to guide a pre-trained discrete model towards a desired distribution \citep{vignac2022digress,schiff2024simple,nisonoff2024unlocking}. 
However, there are some challenges to developing guidance mechanisms in discrete cases. The main challenge is that discrete guidance is often associated with a transition probability or transition rate, which requires extra computation for all possible target positions after the transition. To reduce the number of forward passes through guidance models, existing methods treat the model as a continuous function and use the first-order approximation to compute guidance efficiently \citep{vignac2022digress,schiff2024simple,nisonoff2024unlocking}. Unfortunately, this approximation might introduce non-negligible approximation errors. Moreover, existing literature only consider some particular cases, like class conditional guidance or energy-weighted guidance, which might not be general enough for various tasks.



We seek to develop a general guidance framework for discrete diffusion and flow-based models, which is illustrated in \cref{fig:sampling_framework}. Leveraging the Continuous-Time Markov Chain (CTMC) framework, given a pretrained model for sampling from a source distribution and the density ratio between the target distribution and source distribution, our method identifies the exact transition rate for sampling the target distribution. To the best of our knowledge, our guidance framework is the {\it first} exact discrete guidance in general form. Since our guidance formulation is expressed by the density ratio between the target distribution and source distribution, it also encompasses existing guidance methods as special cases. This is because the density ratio can be calculated by the energy function in energy-guided sampling, the ratio of the classifier in class-conditional generation, and preference alignment in reinforcement learning from human feedback (RLHF). In Table \ref{table:comparison of guidance}, we compare the formulation of the guidance and the number of function evaluations in each sampling step. Our framework only requires one forward pass in each sampling step for any initial distribution. 

To compute the proposed guidance, we employ the Bregman divergence to train a network for estimating the conditional expectation. To further utilize potentially available samples from the target distribution, we further propose a regularization term.



\begin{figure}[t]
    \centering    \includegraphics[width=\linewidth]{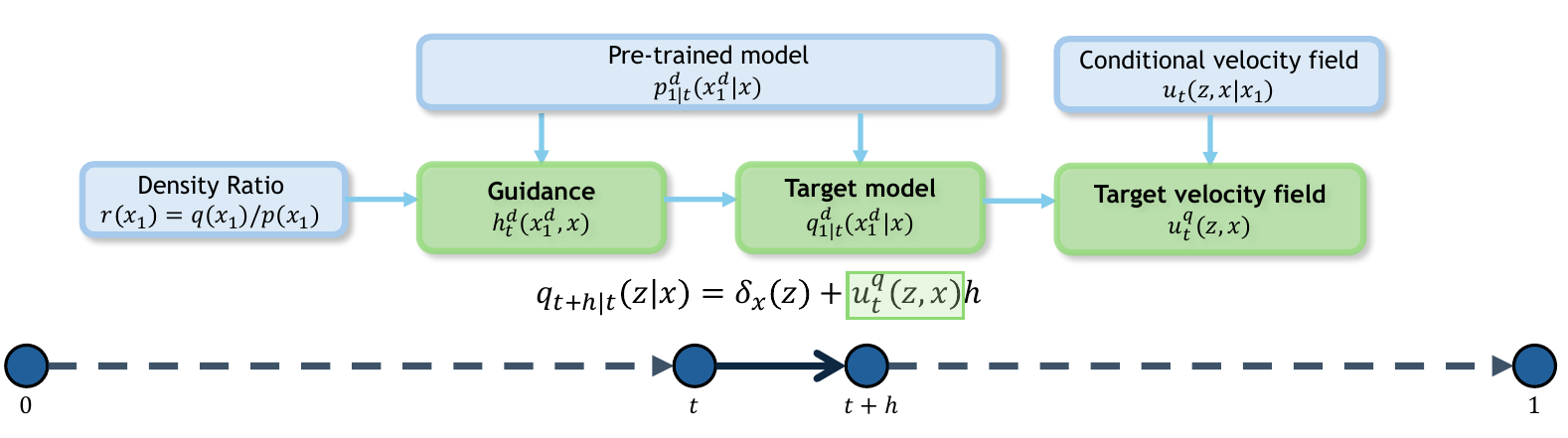}
    \vspace{-0.25in}
    \caption{{\small Framework of the proposed guidance on discrete flow matching. Given a pretrained discrete flow matching on source distribution $p_1(x)$ and a known density ratio $r(x)=q_1(x)/p_1(x)$, our framework calculates the target velocity field to generate target distribution $q_1(x)$. Our framework is general enough to deal with energy-guided sampling with $p_1^{(\gamma)}(x)\propto p_1(x)e^{-\gamma \mathcal{E}(x)}$ shown in \cref{sec:simulation}, classifier guidance shown in \cref{thm:guidance}, and preference alignment in \cref{exp:multmodal}.}}
    \label{fig:sampling_framework}
\end{figure}

In summary, our main contributions are as follows.
\begin{enumerate}
    \item We introduce a general guidance framework for discrete flow matching, achieving efficient sampling without approximation. 
    \item We propose to learn the guidance network by minimizing  Bregman divergence. We further propose a regularization technique to utilize samples from the target distribution.
    \item We verify the effectiveness of the proposed framework on energy-based sampling through simulations, preference alignment on text-to-image, and multimodal understanding benchmark. 
\end{enumerate}
\begin{table}[t]
    \centering
     \resizebox{0.99\textwidth}{!}{%
    \begin{threeparttable}
    \begin{small}
    \caption{\small {Overall comparison with existing discrete guidance. Our posterior-based guidance provides an exact formulation of guidance for general target distributions with a sampling-efficient implementation. Existing methods are either limited to class-conditional generation or suffer from non-negligible approximation errors with multiple function evaluations in each sampling step. A more detailed discussion is provided in \cref{sec:discrete_guidance}.}}\label{table:comparison of guidance}
    \begin{tabular}{llcc}
    \toprule
    Guidance Method & Formula & Exact Guidance & 
    \# of function evaluations\\
    \midrule
    \vspace{0.03in}
    Posterior-Based (Ours) & 
    \Eqref{eq:guidance} & \color{green}{\ding{51}} & 1\\
    \vspace{0.03in}
    Rate-Based \cite{nisonoff2024unlocking} & 
    \Eqref{eq:diffusion classifier guidance}& \color{green}{\ding{51}} & 
    
      $\mc{D}+1$ ; $\mc{D}\times(\abs{\mc{S}}-1)+1$
  \\
\makecell[l]{First-Order Approximated \cite{nisonoff2024unlocking} \\
\cite{vignac2022digress, schiff2024simple}}
& \Eqref{eq:first-order approximation}
& \color{red}{\ding{55}}
& 2 \\
    \bottomrule
    \end{tabular}
        \begin{tablenotes}
\item \textit{Note:}  In \cite{nisonoff2024unlocking}, guidance requires 
$\mc{D}\times(\abs{\mc{S}}-1)+1$ function evaluations. 
When implemented using Algorithm~\ref{alg:sampling with rate-based guidance}, 
the number of function evaluations reduces to $\mc{D}+1$.
\end{tablenotes}
\end{small}
\end{threeparttable}
}

\end{table}

\textsc{Notation.} Let $\brac{N}=\set{1,2,\dots,N}$ for a positive integer $N$. We use $\dot{\kappa}_t$ to denote the time derivative of a function $\kappa_t$ with respect to $t$. For a $\mc{D}$-dimensional vector $z$, let $z^d$ and $z^{\bsl d}$ denote the $d$-th element of the vector $z$ and the $(\mc{D}-1)$-dimensional vector $(z^1,\dots,z^{d-1},z^{d+1},\dots,z^\mc{D})^\top$. For a positive integer $p$, let $\mathbbm{1}_p$ denote $p$-dimensional all-one vector. We write $f(h)=o(h)$ if $f:\R\to\R$ satisfies $f(h)/h\to 0$ as $h\to 0^+$. For two quantities $x,z$, define the Kronecker delta $\delta_x(z)$ satisfying $\delta_x(z)=1$ if $x=z$ and $\delta_x(z)=0$ if $x\neq z$. Let $\<x,z\>$ be the inner product of two vectors $x,z$ in Euclidean space. We denote $\ltwo{x}$ as the Euclidean norm of a vector $x$.

\section{Discrete Flow Models}
Discrete flow models (DFMs) aim to generate samples from a $\mc{D}$-dimensional discrete state-space with probability mass function (pmf) $q_1$, defined on the domain $\mc{S}^{\mc{D}}$, where $\mc{S}$ denotes the finite state space of each dimension. A natural example of this setting is text generation, where each coordinate corresponds to a token drawn from the vocabulary $\mc{S}$. 
To begin with, we first introduce some background on discrete flow models \citep{campbell2024generative,gat2024discrete,shaul2024flow} under the framework of CTMC \citep{norris1998markov}, which we describe in detail in \cref{sec:background CTMC}.

\subsection{Transition Rate and Transition Probability}
Given a CTMC with the marginal pmfs $\set{q_t}_{t\in[0,1]}$, let $u_t^q$ be the associated transition rate of this CTMC at time $t$, which is a $|\mc{S}|^{\mc{D}}\times |\mc{S}|^{\mc{D}}$ matrix. 
For $z,x\in\mc{S}^\mc{D}$ with $z\neq x$, we write $u_t^q(z,x)$ for the corresponding entry in this matrix representing the intensity of the transition from state $x$ to state $z$ at time $t$. 
If $u_t^q(z,x)$ satisfy the following rate-properties:
\begin{align*}
    u_t^q(z,x)\ge 0, \text{ for any }z\neq x, \text{ and }~  \sum_{z\in\mc{S}^\mc{D}}u_t^q(z,x)=0,
\end{align*}
and the Kolmogorov forward equation (also known as the continuity equation):
\begin{align}\label{eq:Kolmogorov equation}    \dot{q}_t(x)=\sum_{z\in\mc{S}^\mc{D}}u_t^q(x,z)q_t(z)=\underbrace{\sum_{z\neq x}u_t^q(x,z)q_t(z)}_{\text{incoming flux}}-\underbrace{\sum_{z\neq x}u_t^q(z,x)q_t(x)}_{\text{outgoing flux}},
\end{align}
we say $u_t^q$ can generate the probability path $q_t$. Based on the above Kolmogorov forward equation, for a sufficiently small time step $h$, we have
\begin{align*}
    q_{t+h}(x)=q_t(x)+\dot{q}_t(x)h+o(h)=\sum_{z\in\mc{S}^\mc{D}}\setBig{\delta_{z}(x)+u_t^q(x,z)h+o(h)}q_t(z),
\end{align*}
which implies that the transition probability from state $z$ to state $x$, for any $x,z \in \mc{S}$, is 
\begin{align}\label{eq:transition probability}
    q_{t+h|t}(x|z)=\delta_{z}(x)+u_t^q(x,z)h+o(h).
\end{align}
Thus, the transition rate $u_t^q$ is also the generator of the CTMC $\set{q_t}_{t\in[0,1]}$ \citep{holderrieth2024generator}.

\subsection{Modeling Transition Rate via Marginalization Trick}
Given the transition rate $u_t^q(z,x)$ and current state $x$ at time $t$, we can generate a sample from $q_{t+h}$ using the \eqref{eq:transition probability}, with a sufficiently small time step $h$. Therefore, we can learn a transition rate $u_t^q$ that can transport from an initial noise distribution $q_0$ to a target data distribution $q_1$. To obtain such a transition rate for sampling, a natural method is to learn the conditional expectation of the conditional transition rate $u^q_t(z,x|\rvx_1)$ that generates the conditional probability path $q_{t|1}(\cdot|\rvx_1)$ interpolating from noise to the datapoint $\rvx_1$. Then $u^q_t(z,x)=\E_{q_{1|t}(\rvx_1|x)}\brac{u^q_t(z,x|\rvx_1)}$ can generate the target probability path $q_t$ \citep[see Proposition 3.1 of][]{campbell2024generative}, i.e., it satisfies the Kolmogorov forward \eqref{eq:Kolmogorov equation}. For completeness, we include the proof in Appendix \ref{appendix:prop}. We are free to define the conditional probability path $q_{t|1}(\cdot|\rvx_1)$ and with the corresponding conditional transition rate as needed.

It is worth noting that learning a $\abs{\mc{S}}^\mc{D}$-dimensional vector-valued function $\paren{u_t(z,x)}_{z\in\mc{\mc{S}}^\mc{D}}$ of current time $t$ and state $x$ is intractable when $\mc{D}$ is relatively large. To handle such a high-dimensional setting, a common approach is to construct a coordinate-wise conditional probability 
path and transition rate, 
\begin{align}\label{eq:coordinate-wise conditional path}
   q_{t|1}(x|\rvx_1)=\prod_{d=1}^{\mc{D}}q^d_{t|1}(x^d|x_1^d),\text{ and }u_t^{q}(z,x|\rvx_1)=\sum_{d=1}^\mc{D}\delta_{x^{\bsl d}}(z^{\bsl d})u_t^{q,d}(z^d,x^d|x_1^d),
\end{align}
which means that the elements of the vector $\rvx_t$ are independent conditional on $\rvx_1$. Here, $u_t^{q,d}(z^d,x^d|x_1^d)$ is the conditional transition rate that generates the conditional probability path $q_{t|1}^d$. A popular choice of probability path and the associated conditional transition rate used in the previous works \citep{campbell2024generative,gat2024discrete} is
\begin{align}\label{eq:mixture path}
    q_{t|1}^d(x^d|x_1^d)=(1-\kappa_t)q_0^d(x^d)+\kappa_t\delta_{x_1^d}(x^d) ;~ u_t^{q,d}(z^d,x^d|x_1^d)=\frac{\dot{\kappa}_t}{1-\kappa_t}(\delta_{x_1^d}(z^d)-\delta_{x^d}(z^d)),
\end{align}
where $\kappa_t:[0,1]\to[0,1]$ is a non-decreasing function satisfying $ \kappa_0=0$ and $\kappa_1=1$.

After defining the conditional path and rate, the marginal transition rate is given by
\begin{align*}    u_t^q(z,x)=\sum_{\rvx_1}u_t^q(z,x|\rvx_1)q_{1|t}(\rvx_1|x)=&~\sum_{d=1}^\mc{D}\delta_{x^{\bsl d}}(z^{\bsl d})\sum_{x_1^d}u_t^{q,d}(z^d,x^d|x_1^d)q^d_{1|t}(x_1^d|x)\\
    \overset{\triangle}{=}&~\sum_{d=1}^\mc{D}\delta_{x^{\bsl d}}(z^{\bsl d})u_t^{q,d}(z^d,x),
\end{align*}
where $q_{1|t}^d(x_1^d|x)=\sum_{x^{\bsl d}_1}q_{1|t}(x_1|x)$ is the posterior. Therefore, we only need to learn a $\mc{D}\times\abs{\mc{S}}$-matrix-valued function $U_t^q(x)=(u_t^{q,d}(s,x))_{d,s}$. Then we can update each token in parallel during sampling.

\noindent {\bf Training Objective and Sampling Algorithm.} 
Given the conditional transition rate, the unconditional transition rate can be parameterized by learning the posterior $q_{1|t}^d(x_1^d|x)$. Let $\mc{U}([0,1])$ be the uniform distribution on $[0,1]$. A simple learning objective 
for discrete flow models is given by
\begin{align}\label{eq:training via cross-entropy}
    \mc{L}_q=\E_{t\sim\mc{U}([0,1]),\rvx_1\sim q_1(\rvx_1),\rvx_t\sim q_{t|1}(\rvx_t|\rvx_1)}\bracBig
    {-\sum_{d=1}^\mc{D}\log q_{1|t}^{d,\theta}(\rvx_1^d|\rvx_t)},
\end{align}
which is the cross-entropy between the true posterior $q_{1|t}$ and the estimated posterior $q^\theta_{1|t}$ \citep{campbell2024generative,gat2024discrete,wang2025fudoki}.

Additionally, there are some alternative approaches to training the transition rate, including directly maximizing the ELBO in terms of the unconditional transition rate \citep{shaul2024flow} and minimizing a Bregman divergence-type conditional generator matching loss \citep{holderrieth2024generator}. It turns out that parameterization through the posterior gives us a unified framework to formulate the discrete guidance not only for discrete flow-based models, but also for other generative models. We include more detailed discussions in \cref{discuss:unifiying guidance}. Moreover, details of the sampling implementation for \eqref{eq:transition probability} in \cref{alg:sampling without guidance} are provided in the \cref{app:sampling}.
\section{Discrete Flow Guidance}
\label{sec:discrete_guidance}
In this section, we introduce our general guidance framework for modeling discrete data. 
It  can be applied to both discrete flow models and discrete diffusion models \citep{campbell2022continuous,sun2022score,austin2021structured,lou2023discrete,shi2024simplified,sahoo2024simple}. 
More detailed discussions can be found in \cref{ap:indepth}.

\subsection{Problem Formulation}
Consider a CTMC with a probability path $p_t$ associated with a transition rate $u^p_t$. Suppose that the conditional transition rate and the conditional probability path are prespecified (see \eqref{eq:mixture path} for an example), and we are given a pre-trained posterior $p_{1|t}$ such that the unconditional transition rate $u^p_t(z,x)=\E_{\rvx_1\sim p_{1|t}(\rvx_1|x)}\brac{u_t^p(z,x|\rvx_1)}$ can generate the probability path $p_t$. We will call the distribution with the pmf $p_1$ the source distribution. Our goal is to generate samples from the target distribution $q_1$ by rectifying the posterior or the unconditional transition rate corresponding to the probability path $p_t$.

\subsection{Exact Discrete Guidance}
To begin with, we first impose the following assumption on the target distribution.
\begin{assumption}\label{con:support}
    The target distribution $q_1$ is absolutely continuous with respect to the source distribution $p_1$, i.e., the support of the target probability mass function is a subset of that of the source probability mass function: $\set{x\in\mc{S}^\mc{D}|~q_1(x)>0}\subseteq\set{x\in\mc{S}^\mc{D}|~p_1(x)>0}$.
\end{assumption}
This assumption admits a well-defined density ratio $r(x)=q_1(x)/p_1(x)$ on the support of $p_1$, which is crucial for our formulation of discrete guidance. In the special case setting of conditional generation, the target distribution is the conditional distribution $q_1(x)=p_1(x|y)$ where $y$ denotes the conditioning variable, and the Assumption \ref{con:support} holds by construction.

Now, assume that the density ratio $r(x)=q_1(x)/p_1(x)$ is known. The following theorem yields the rectified posterior \(q_{1| t}\) by reweighting the source posterior \(p_{1| t}\) with a guidance term.

\begin{theorem}[Posterior-based guidance]\label{thm:guidance}
    Suppose that the conditional probability path of the source distribution $p_{t|1}$ and that of the target distribution $q_{t|1}$ are chosen to be the same for any $t\in[0,1]$. Under \cref{con:support}, the posterior regarding the target distribution has the following form:    \begin{align}\label{eq:guidance}
        q_{1|t}(z^d|x)=\frac{\E_{\rvx^{\bsl d}_1\sim p(\rvx^{\bsl d}_1|\rvx^d_1=z^d,\rvx_t=x)}\brac{r(\rvx_1)}}{\E_{\rvx_1\sim p_{1|t}(\rvx_1|x)}\brac{r(\rvx_1)}}p_{1|t}(z^d|x).
    \end{align}
    In particular, if $q_1(x)=p_1(x|y)$, we have
    \begin{align}\label{eq:classifier guidance}
        q_{1|t}(z^d|x)=p_{1|t}(z^d|x,y)=\frac{p(y|\rvx_1^d=z^d,\rvx_t=x)}{p(y|\rvx_t=x)}p_{1|t}(z^d|x).
    \end{align}
\end{theorem}
Here, we only assume that the conditional probability paths are the same. \cref{thm:guidance} provides a general guidance framework for both discrete diffusion and flow models, because both of them can be modeled via parameterizing posterior densities. In discrete diffusion models, if we further assume that the transition rates of both forward noising processes are equal, then we can generalize and recover the predictor guidance in \cite{nisonoff2024unlocking}, which we state in the following.

\begin{theorem}[Rate-based guidance]\label{thm:diffusion guidance}
Suppose that the transition rates of the forward noising processes $Q^p_t(x,z)$ and $Q^q_t(x,z)$ are equal for any $t\in[0,1]$. Under \cref{con:support}, the backward transition rate $u^q_t(z,x)$ of probability $q_t$ has the following form:
\begin{align}\label{eq:diffusion guidance}    u_t^q(z,x)=\frac{\E_{\rvx_1\sim p_{1|t}(\rvx_1|z)}\brac{r(\rvx_1)}}{\E_{\rvx_1\sim p_{1|t}(\rvx_1|x)}\brac{r(\rvx_1)}}u_t^p(z,x),
\end{align}
where $u_t^q(z,x)$ and $u_t^p(z,x)$ is the time reversal of $Q^p_t(x,z)$ and $Q^q_t(x,z)$, respectively.

Furthermore, if $q_1(x)=p_1(x|y)$, then we have
\begin{align}\label{eq:diffusion classifier guidance}
    u_t^p(z,x|y)=\frac{p(y|\rvx_t=z)}{p(y|\rvx_t=x)}u_t^p(z,x),
\end{align}
which recovers Equation 2 in \cite{nisonoff2024unlocking}.
\end{theorem}
\cref{thm:diffusion guidance} requires the source and target corruption rate matrices to be the same, which is stronger than \cref{thm:guidance}. Although \cref{thm:diffusion guidance} gives a transition rate that can generate the target probability path, the transition rate we get is not necessarily the same as what we want ($u^q_t(z,x)=\E_{q_{1|t}(\rvx_1|x)}\brac{u^q_t(z,x|\rvx_1)}$). However, \cref{thm:guidance} provides a guidance term for the posterior and leads to the desired transition rate $u^q_t(z,x)$.  See further discussion in \cref{appendix:comparison of guidance}. 

In general cases, the guidance is time-dependent according to \cref{thm:guidance} and \cref{thm:diffusion guidance}. However, in the masked diffusion and flow models with $x_1$-independent scheduler, the guidance is time-independent. To see this, note that in those cases the posterior is time-independent (see Theorem 1 of \cite{ou2024your} and Proposition 6 of \cite{gat2024discrete} for details). Since the guidance is the integral of the product of the density ratio and the posterior, the time-independent posterior can lead to a time-independent guidance.

\subsection{Training Objective}\label{sec:training}
According to \cref{thm:guidance}, given the density ratio $r(x)$, we can obtain the exact guidance term by learning the conditional expectation of the density ratio $h_t^d(z^d,x)\overset{\triangle}{=}\E_{\rvx^{\bsl d}_1\sim p(\rvx^{\bsl d}_1|\rvx^d_1=z^d,\rvx_t=x)}\brac{r(\rvx_1)}$. We parameterize the guidance $h_t$ using a neural network. A natural training objective for learning conditional expectation is the Bregman divergence \citep{Banerjee2005}, which is defined as
\begin{align}\label{eq:Bregman divergence}
D_F(x\|y)=F(x)-F(y)-\<\nabla F(y),x-y\>,
\end{align}
where $F(\cdot)$ is a convex function. The most common choice of $F$ is $F(x)=\ltwo{x}^2/2$, and then the Bregman divergence is the $\ell^2$-loss. Unfortunately, as mentioned in \cite{lou2023discrete}, the $\ell^2$-loss does not work well for learning the conditional expectation of a density ratio due to the positive nature of the density ratio. In our work, we take $F(x)=\<x,\log x\>$, and the training objective is given by the following.
\begin{equation}\label{eq:training objective}
\begin{aligned}
    &~\E_{t\sim \mc{U}([0,1]),\rvx_1\sim p_1(\rvx_1),\rvx_t\sim p_{t|1}(\rvx_t|\rvx_1)}\bracBig{D_F\parenBig{\mathbf{r}(\rvx_1)\Big\|\mathbf{h}_t^{\theta}(\rvx_1,\rvx_t)}}\\
    =&~\E_{t\sim \mc{U}([0,1]),\rvx_1\sim p_1(\rvx_1),\rvx_t\sim p_{t|1}(\rvx_t|\rvx_1)}\bracBig{\sum_{d=1}^\mc{D}h_t^{d,\theta}(\rvx_1^d,\rvx_t)-r(\rvx_1)\log h_t^{d,\theta}(\rvx_1^d,\rvx_t)}+C\\
    \overset{\triangle}{=}&~\mc{L}_{h,p}(\theta)+C,
\end{aligned}
\end{equation}
where $\mathbf{r}(\rvx_1)=r(\rvx_1)\mathbbm{1}_\mc{D}$, $\mathbf{h}_t^\theta(\rvx_1,\rvx_t)=(h_t^{1,\theta}(\rvx^1_1,\rvx_t),\dots,h_t^{\mc{D},\theta}(\rvx^\mc{D}_1,\rvx_t))^\top$ and $C$ is a constant independent of $\theta$.

The loss $\mc{L}_{h,p}$ only requires the source distribution data. Additionally, we can utilize samples from the target distribution if they are available to achieve better performance. Similar to \cite{ouyang2024transfer}, we introduce a regularization term allowing us to learn the guidance $h_t$ with sampling $\rvx_1\sim q_1$ during training. Notice that the minimizer of the objective in \eqref{eq:training via cross-entropy} is the underlying posterior $q_{1|t}$. Combining with \cref{thm:guidance}, the exact guidance $h_t$ is the minimizer of the following objective. 
\begin{align}\label{eq:regularization}
    \mc{L}_{h,q}(\theta)=\E_{t\sim \mc{U}([0,1]),\rvx_1\sim q_1(\rvx_1),\rvx_t\sim q_{t|1}(\rvx_t|\rvx_1)}\bracBig{-\sum_{d=1}^\mc{D}\log\frac{h^{d,\theta}_t(\rvx_1^d,\rvx_t)}{\sum_{s\in\mc{S}}h^{d,\theta}_t(s,\rvx_t)p_{1|t}^d(s|\rvx_t)}}. 
\end{align}
Thus, our final training objective can be written as $\mc{L}_h(\theta)=\mc{L}_{h,p}(\theta)+\lambda\mc{L}_{h,q}(\theta)$, where $\lambda$ denotes the hyperparameter controlling the strength of the regularization.

\subsection{Sampling}
After the training stage, we can obtain the learned posterior-based guidance, which is a matrix-valued function $H_t^\theta:\mc{S}^\mc{D}\times [0,1]\to\R_+^{\mc{D}\times \abs{\mc{S}}}$ with $[H_t^\theta(x)]_{d,s}=h_t^{d,\theta}(s,x)$. At the sampling stage, given the current state $\rvx_t$, we first sample $\rvx_1^d$ from 
\begin{align*}
    q^{d,\theta}_{1|t}(\cdot|\rvx_t)= \frac{h_t^{d,\theta}(\cdot,\rvx_t)}{h_t^{d,\theta}(\rvx_t)}p_{1|t}^d(\cdot|\rvx_t)
\end{align*}
for each $d\in\brac{\mc{D}}$, where $h_t^{d,\theta}(x)=\sum_{s\in\mc{S}}h_t^{d,\theta}(s,x)p_{1|t}^d(s,x)$. Then, we sample $\rvx^d_{t+h}$ by always-valid sampling procedure (see \eqref{eq:valid sampling} in Appendix \ref{app:sampling} for details). {Specifically, at each step, we input the current state $\rvx_t$ into both the guidance network and pretrained posterior network, and then multiply the output to get the final transition probability in the target distribution $q^{d,\theta}_{1|t}(\cdot|\rvx_t)$ to update the current state $\rvx_t$ like \cref{alg:sampling without guidance}.} The overall sampling algorithm can be found in \cref{alg:sampling with posterior-based guidance} in Appendix \ref{app:sampling} and the detailed design for the general rate-based method proposed in \cref{thm:diffusion guidance} can be found in \cref{ap:rate_based}.

\begin{remark}\label{remark:first order approximation}
The rate-based guidance involves multiple forward passes, which depend on the number of nonzero elements in the transition rate $u_t^p(z,\rvx_t),~z\neq \rvx_t$, given the current state $\rvx_t$. If one use the always-valid sampling algorithm with rate-based guidance \cref{alg:sampling with rate-based guidance}, the guidance entails $\mc{D}+1$ function calls in each sampling step, which is still computationally inefficient; see \cref{fig:toy_selected_and_speed} (d). For predictor guidance, \cite{nisonoff2024unlocking} proposed to use the first-order approximation to estimate the log-ratio $\log\frac{p(y|\rvx_t=z)}{p(y|\rvx_t=x)}$ for efficient sampling. To be specific, we give the general form based on our proposed rate-based guidance:
\begin{align*}
    \log \E_{\rvx_1\sim p_{1|t}(\rvx_1|z)}\brac{r(\rvx_1)}\approx\log \E_{\rvx_1\sim p_{1|t}(\rvx_1|x)}\brac{r(\rvx_1)}+\Big\<z-x,\nabla_x\log\E_{\rvx_1\sim p_{1|t}(\rvx_1|x)}\brac{r(\rvx_1)}\Big\>.
\end{align*}
Plugging this into \eqref{thm:diffusion guidance}, the approximated guidance is
\begin{align}\label{eq:first-order approximation}    u_t^q(z,x)=\exp\parenBig{\Big\<z-x,\nabla_x\log\E_{\rvx_1\sim p_{1|t}(\rvx_1|x)}\brac{r(\rvx_1)}\Big\>}u_t^p(z,x).
\end{align}
However, this approximation not only introduces an approximation error but is also unreasonable for discrete data modeling, since the value of the right-hand side depends mainly on the location of $z$ and $x$ in Euclidean space. Fortunately, our proposed posterior-based guidance enables sampling at low computational cost without approximation, requiring only a single forward pass. \cref{table:comparison of guidance} summarizes three types of discrete guidance.
\end{remark}

\subsection{Guidance for Preference Alignment}
{
Given the general form of guidance for discrete flow matching, our framework is able to deal with energy-guided sampling \citep{zhang2025energy, lu2023contrastive} and also preference alignment in reinforcement learning from human feedback (RLHF) \citep{RectorBrooks2024SteeringMD}. In the following, we provide the formulation for RLHF.
Let $\rvc$ denote the prompt with distribution $p_\rvc$, $\pi_{ref}(\rvo_1|\rvc)$ denote the reference policy associated with pretrained model $p^d_{1|t}(\rvo_1^d|\rvc,\rvo_t)$. In the RL stage of RLHF \citep{ouyang2022training}, we consider to maximize
\begin{align*}
    \E_{\rvc\sim p_\rvc,\rvo_1|\rvc\sim\pi}\bracBig{\mc{R}(\rvc,\rvo_1)-\tau\log(\pi(\rvo_1|\rvc)/\pi_{ref}(\rvo_1|\rvc))},
\end{align*}
where $\tau$ is the temperature that controls the deviation from the reference policy.
As mentioned in \cite{rafailov2023direct}, the above objective has the closed-form maximizer
\begin{align*}
   \pi^*(\rvo_1|\rvc)=\frac{1}{\mc{Z}(\rvc)}\pi_{ref}(\rvo_1|\rvc)\exp\parenBig{\frac{\mc{R}(\rvc,\rvo_1)}{\tau}},
\end{align*}
where $\mc{Z}(\rvc)=\int \pi_{ref}(\rvo_1|\rvc)\exp\parenBig{\frac{\mc{R}(\rvc,\rvo_1)}{\tau}}\dd \rvo_1$ and $\mc{R}(\cdot,\cdot)$ is the reward function which can be rule-based or obtained by training a neural network on the a comparison dataset. Given the reward function, we aim to generate samples $\rvo_1\sim\pi^*(\rvo_1|\rvc)$ conditional on $\rvc$. If we take $p_1(\cdot)=\pi_{ref}(\cdot|\rvc)$ and $q_1(\cdot)=\pi^*(\cdot|\rvc)$ in \cref{thm:guidance}, we can obtain a rectified posterior for sampling from $\pi^*$, i.e., we would like to use $\mathbf{h}_t^{\theta}(\rvo_1,\rvo_t, c)$ to approximate $\exp\parenBig{\frac{\mc{R}(\rvc,\rvo_1)}{\tau}}$. We leave the detailed training objective in \eqref{eq_for_RLHF} in \cref{ap:rlhf}.}

\section{Experiments}\label{sec:numerical}
In this section, we present empirical evidence to verify the efficacy of the proposed framework. In Section \ref{sec:simulation}, we conduct proof-of-concept experiments using energy-guided sampling. In Section \ref{exp:multmodal}, we illustrate the effectiveness of our method on reinforcement learning from human feedback. Built upon a multimodal discrete flow-based model, we demonstrate the effectiveness of our method on
text-to-image generation and multimodal understanding benchmark.

\subsection{Simulation Results} \label{sec:simulation}

\noindent {\bf Experimental Setup and Baselines.} 
We consider a 2-D setting similar to that in \cite{zhang2025energy, lu2023contrastive}. 
The sample space is $\mc{S}^\mc{D}=\set{0,1,\dots,32}^2$. We denote the source distribution as $p_1(x)$ and set the target distribution as its energy-guided version $p_1^{(\gamma)}(x)\propto p_1(x)e^{-\gamma \mathcal{E}(x)}$, where $\gamma\ge 0$ is the guidance strength and $\mathcal{E}(x)=-\log p(y=1|x)$ is the energy function defined by a given classifier $p(y=1|x)$.
Our goal is to sample from the guided distribution $p_1^{(\gamma)}$ by using either the source transition rate $u^p_t$ or the source posterior $p_{1|t}$. We compare our proposed exact discrete guidance with the predictor guidance introduced in \cite{nisonoff2024unlocking}. Note that the density ratio in this case is $r(x_1)=\frac{p_1^{(\gamma)}(x_1)}{p_1(x_1)}=\frac{p^\gamma(y=1|x_1)}{\mc{Z}(\gamma)}$, where $\mc{Z}(\gamma)$ is a constant that does not depend on $x_1$. Thus, we obtain the posterior-based and rate-based guidance based on \cref{thm:guidance} and \cref{thm:diffusion guidance}, respectively. For the predictor guidance proposed by \cite{nisonoff2024unlocking}, we use the predictor-guided rate $u_t^{(\gamma)}(z,x)=\bracBig{\frac{\E_{\rvx_1\sim p_{1|t}(\rvx_1|z)}p(y=1|\rvx_1)}{\E_{\rvx_1\sim p_{1|t}(\rvx_1|x)}p(y=1|\rvx_1)}}^\gamma u_t^p(z,x),~z\neq x$, which is different from our proposed rate-based guidance; see \cref{discuss:guidance strenth} for further discussion.



\noindent {\bf Experimental Results.}
We sample from the 2-D target distribution $p_1^{(\gamma)}$ using different guidance schemes, as shown in \cref{fig:toy_selected_and_speed} (a-c). For $\gamma=0,3$, both rate-based and posterior-based guidance produce distributions close to the ground truth. When $\gamma=10,20$, the distribution generated by the predictor guidance \citep{nisonoff2024unlocking} differs from the ground-truth distribution. Overall, our posterior-based guidance achieves better performance than its rate-based counterpart, and also yields faster sampling, as illustrated in \cref{fig:toy_selected_and_speed} (d). We include the experimental details in \cref{ap:sim_detail} and more experimental results can be found in Appendix \ref{ap:more_exp}.

\begin{figure}[htbp]
\centering
\vspace{-0.05in}
\begin{minipage}[t]{0.48\linewidth}
    \begin{minipage}[t]{\linewidth}
        \centering
        \includegraphics[width=\linewidth]{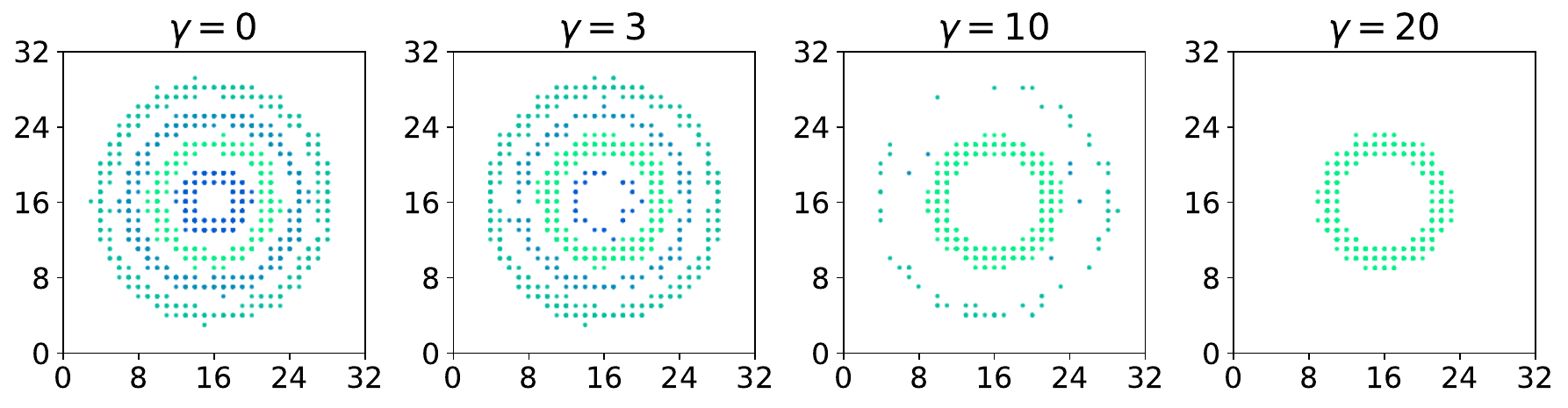}\\[-2pt]
        {\scriptsize (a) \textbf{Ground Truth (rings)}}
    \end{minipage}
    
    \vspace{12pt}
    
    \begin{minipage}[t]{\linewidth}
        \centering
        \includegraphics[width=\linewidth]{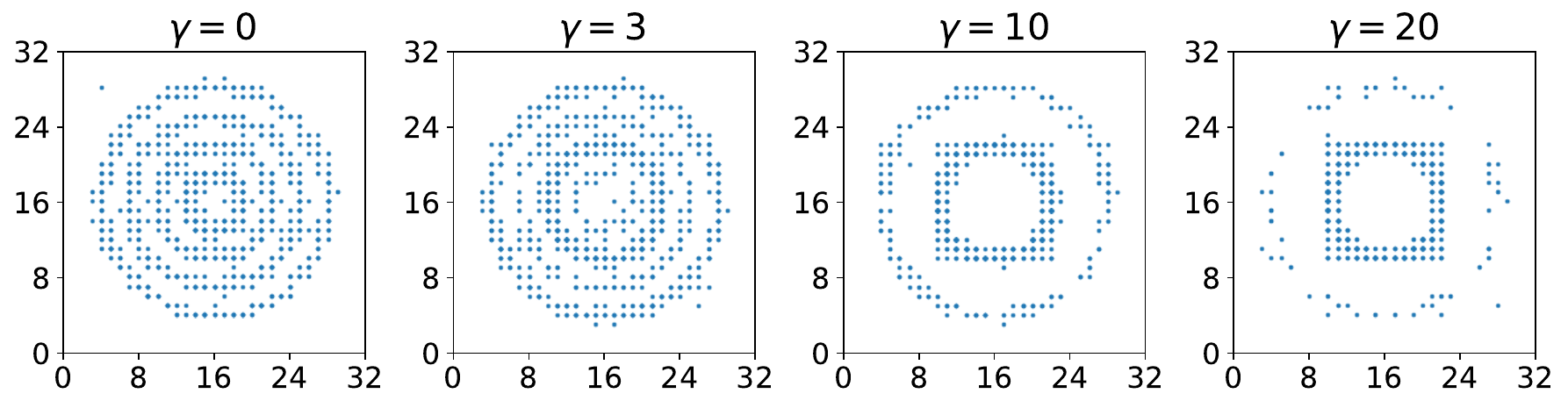}\\[-2pt]
        {\scriptsize (b) Rate-Based (Masked, Nisonoff et al. (2025))}
    \end{minipage}
\end{minipage}
\hfill
\begin{minipage}[t]{0.48\linewidth}
    \begin{minipage}[t]{\linewidth}
        \centering
        \includegraphics[width=\linewidth]{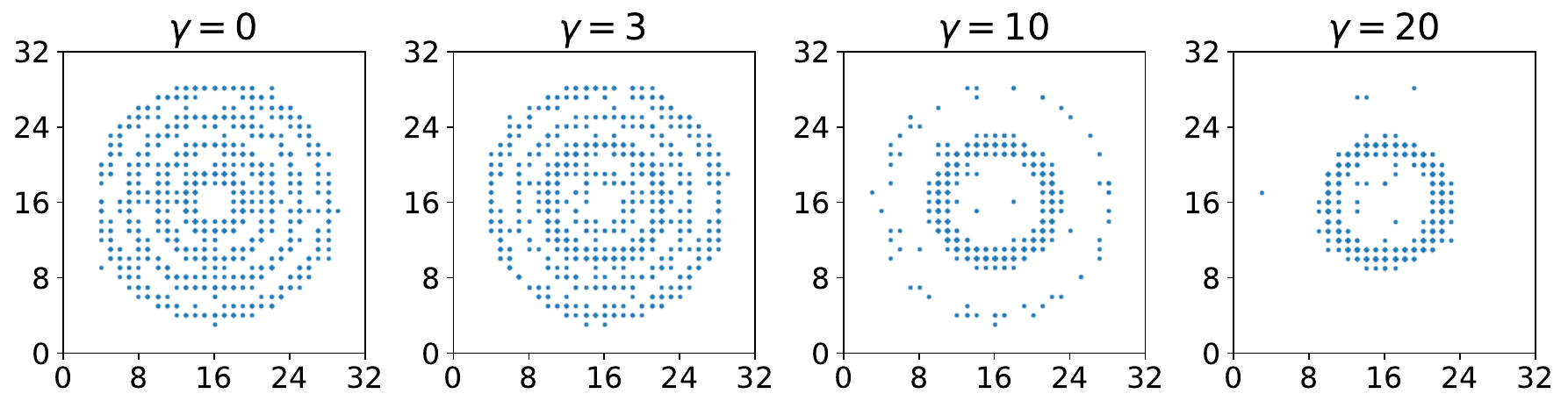}\\[-2pt]
        {\scriptsize (c) Posterior-Based (Uniform, Ours)}
    \end{minipage}
    
    \vspace{6.5pt}
    
    \begin{minipage}[t]{\linewidth}
        \centering
        \includegraphics[width=0.75\linewidth]{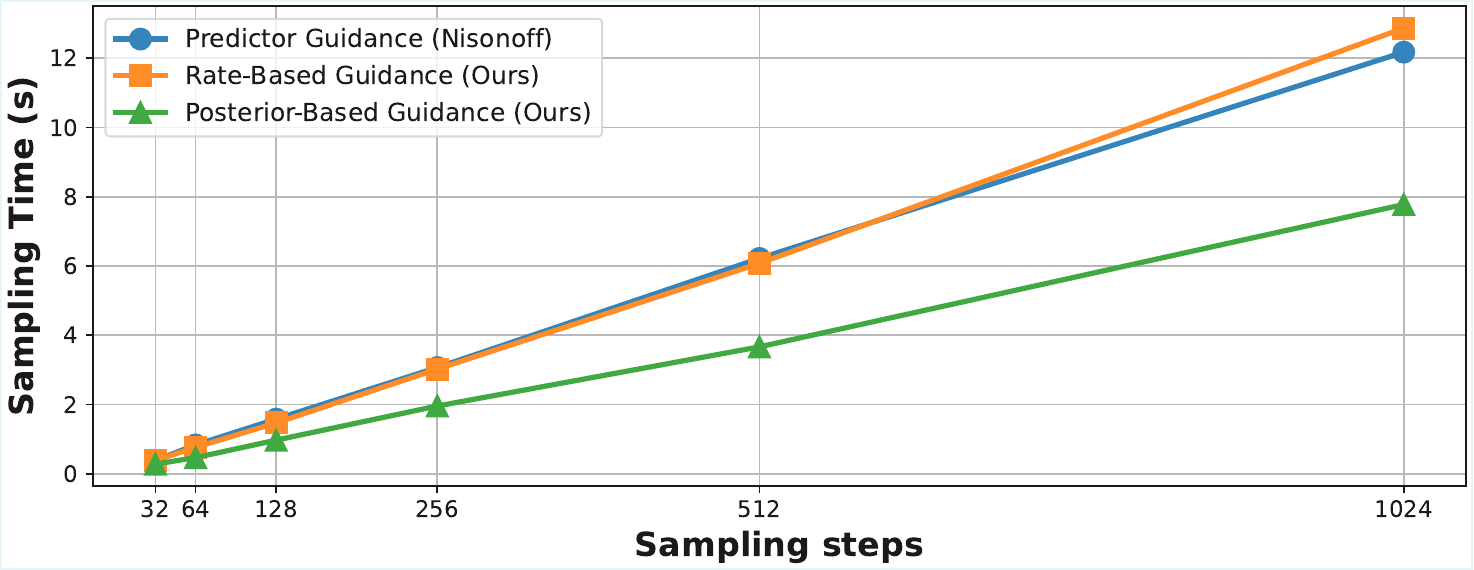}\\[-2pt]
        {\scriptsize (d) Sampling Time Comparison}
    \end{minipage}
\end{minipage}
\vspace{-0.05in}
\caption{\small Comparison of selected sampling results (64 steps) in 2-D experiments and sampling efficiency. (a) Ground truth distribution; (b) rate-based guidance with masked initial distribution \citep{nisonoff2024unlocking}; (c) posterior-based guidance with uniform initial distribution (ours); (d) comparison of sampling times showing posterior-based guidance achieves 1.6x higher sampling speed than rate-based guidance.}
\label{fig:toy_selected_and_speed}
\vspace{-0.2in}
\end{figure}

\subsection{RLHF on Multimodal Benchmark}\label{exp:multmodal}
\noindent {\bf Experimental Setup and Baselines.}
We develop the guidance based on a state-of-the-art discrete flow matching model on multimodal tasks, FUDOKI \citep{wang2025fudoki}. {The detailed probability path for FUDOKI can be found in \eqref{eq:metric-induced path and rate} and the experimental details can be found in \cref{sec:ap_imple_multimodal}.}
For text-to-image generation tasks, we adopt the same training prompts and reward function as \cite{Liu2025FlowGRPOTF}, which utilizes a weighted sum of Pickscore \citep{kirstain2023pick} and GenEval reward \citep{Ghosh2023GenEvalAO} as the reward function. We evaluate the performance on the widely used GenEval Benchmark \citep{Ghosh2023GenEvalAO}. For multimodal understanding tasks, we adopt the question and ground-truth answers from SEED \citep{Li2023SEEDBenchBM} and the widely used LLM-as-a-judge for assigning reward. We adopt VLMEvalKit \citep{Duan2024VLMEvalKitAO} to evaluate on several multimodal understanding datasets, including POPE \citep{Li2023EvaluatingOH}, MME-P \citep{Fu2023MMEAC}, MMB \citep{Liu2023MMBenchIY}, GQA \citep{Hudson2019GQAAN}, MMMU \citep{Yue2023MMMUAM}, and MM-Vet \citep{Yu2023MMVetEL}.


\noindent {\bf Experimental Results.}
Results on the GenEval benchmark for text-to-image generation are presented in \cref{tab:t2i}. Compared to the setting without guidance, our method achieves improvements on four out of six sub-tasks, highlighting the benefit of guidance in utilizing reward signals. We also provide the ablation studies of the regularization strength in \cref{ap:abl}. The performance on multimodal understanding benchmarks is demonstrated in \cref{tab:multimodal}. Our guidance framework consistently improves performance.
Qualitative results can be found in \cref{fig:qualitative_comparison}. Overall, the results confirm the effectiveness of our method for preference alignment across multimodal tasks.

\begin{figure}[h]
\centering
\begin{minipage}[t]{0.48\textwidth}
    \centering
    \includegraphics[width=\textwidth]{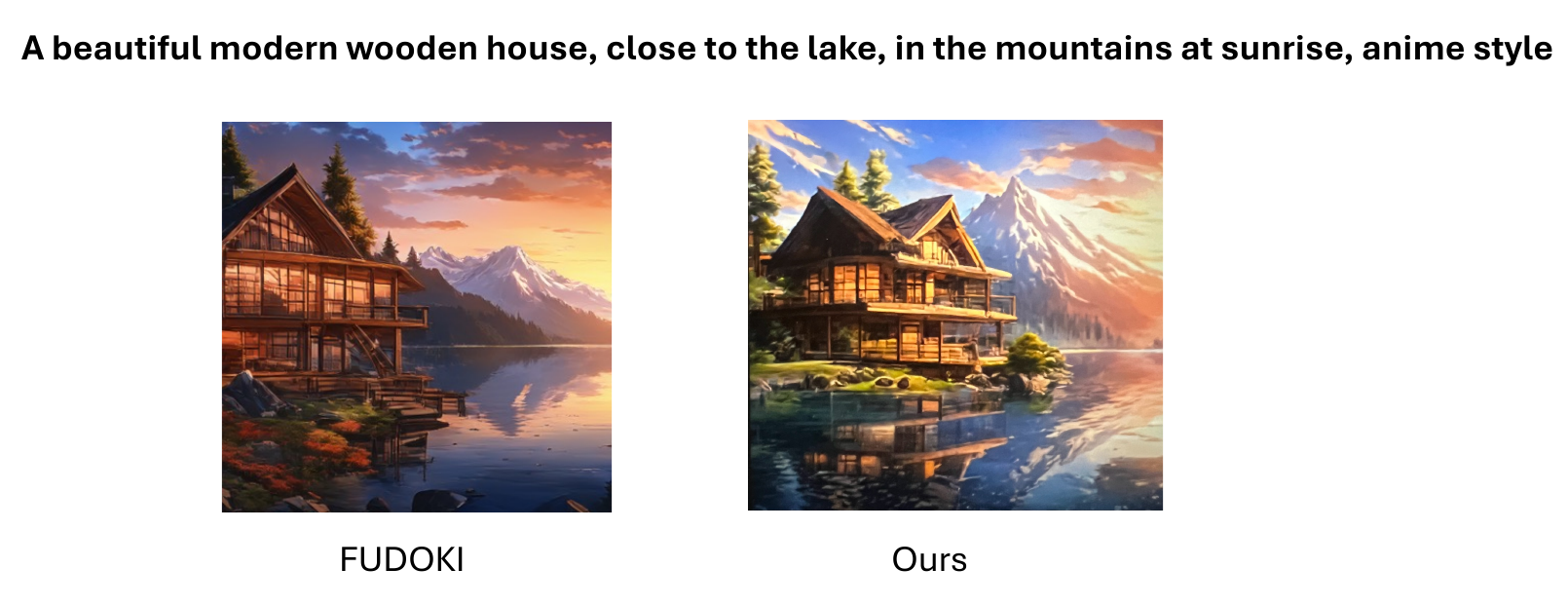}
    \subcaption{Text-to-Image Generation Results}
    \label{fig:qual_t2i}
\end{minipage}
\hfill
\begin{minipage}[t]{0.48\textwidth}
    \centering
    \includegraphics[width=\textwidth]{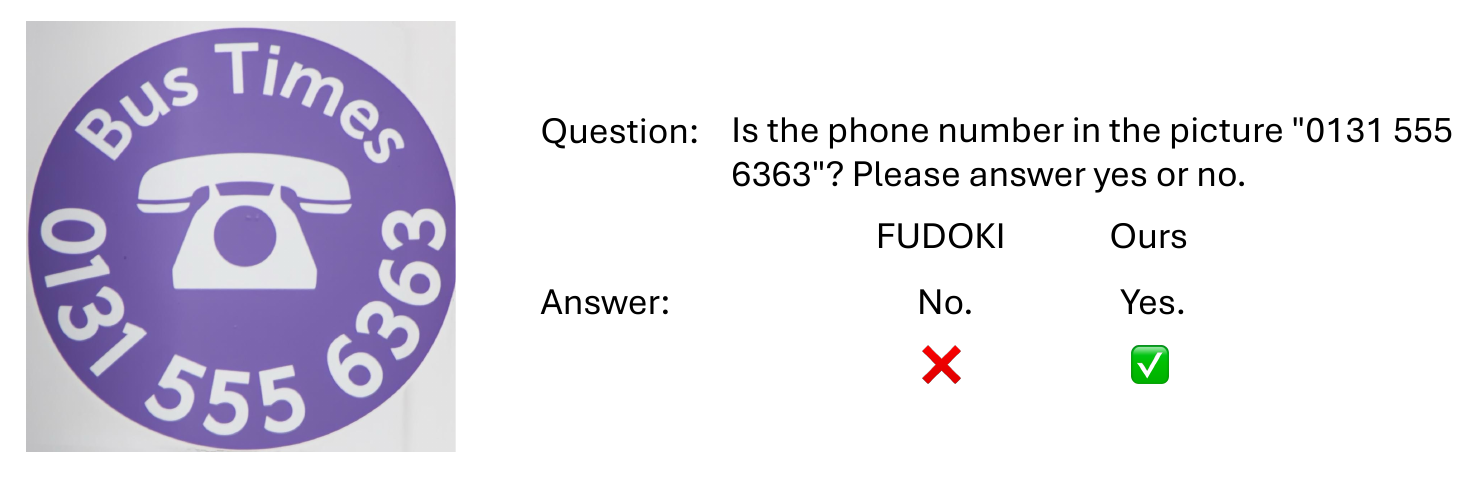}
    \subcaption{Multimodal Understanding Results}
    \label{fig:qual_mmu}
\end{minipage}
\caption{Qualitative comparison on text-to-image generation and multimodal understanding.}
\label{fig:qualitative_comparison}
\end{figure}

\begin{table}[ht!]
\centering
\begin{threeparttable}
\caption{{\small Visual Generation Performance on the GenEval Benchmark.}}\label{tab:t2i}
\tiny
\begin{tabular}{@{}llccccccc@{}}
\toprule
Type & Method & 
\shortstack{Single Obj.} & 
\shortstack{Two Obj.} & 
Counting & Colors & Position & 
\shortstack{Color Attri.} & 
Overall $\uparrow$ \\
\midrule
\multirow{10}{*}{Gen. Only}
& LlamaGen \citep{Sun2024AutoregressiveMB} & 0.71 & 0.34 & 0.21 & 0.58 & 0.07 & 0.04 & 0.32 \\
& Emu3-Gen \citep{Wang2024Emu3NP} & 0.98 & 0.71 & 0.34 & 0.81 & 0.17 & 0.21 & 0.54 \\
& LDM \citep{Rombach2021HighResolutionIS} & 0.92 & 0.29 & 0.23 & 0.70 & 0.02 & 0.05 & 0.37 \\
& SDv1.5 \citep{Rombach2021HighResolutionIS} & 0.97 & 0.38 & 0.35 & 0.76 & 0.04 & 0.06 & 0.43 \\
& PixArt-alpha \citep{Chen2023PixArtFT} & 0.98 & 0.50 & 0.44 & 0.80 & 0.08 & 0.07 & 0.48 \\
& SDv2.1 \citep{Rombach2021HighResolutionIS} & 0.98 & 0.51 & 0.44 & 0.85 & 0.07 & 0.17 & 0.50 \\
& DALL-E 2 \citep{Ramesh2022HierarchicalTI} & 0.94 & 0.66 & 0.49 & 0.77 & 0.10 & 0.19 & 0.52 \\
& SDXL \citep{Podell2023SDXLIL} & 0.98 & 0.74 & 0.39 & 0.85 & 0.15 & 0.23 & 0.55 \\
& DALL-E 3 \citep{BetkerImprovingIG} & 0.96 & 0.87 & 0.47 & 0.83 & 0.43 & 0.45 & 0.67 \\
& SD3-Medium \citep{Esser2024ScalingRF} & 0.99 & 0.94 & 0.72 & 0.89 & 0.33 & 0.60 & 0.74 \\
\midrule
\multirow{13}{*}{Und. + Gen.}
& SEED-$^{\dagger}$ \citep{Ge2024SEEDXMM} & 0.97 & 0.58 & 0.26 & 0.80 & 0.19 & 0.14 & 0.49 \\
& LWM \citep{Liu2024WorldMO} & 0.93 & 0.41 & 0.46 & 0.79 & 0.09 & 0.15 & 0.47 \\
& ILLUME \citep{Wang2024ILLUMEIY} & 0.99 & 0.86 & 0.45 & 0.71 & 0.39 & 0.28 & 0.61 \\
& TokenFlow-XL \citep{Qu2024TokenFlowUI}& 0.95 & 0.60 & 0.41 & 0.81 & 0.16 & 0.24 & 0.55 \\
& Chameleon \citep{Team2024ChameleonME} & -- & -- & -- & -- & -- & -- & 0.39 \\
& Janus \citep{Wu2024JanusDV} & 0.97 & 0.68 & 0.30 & 0.84 & 0.46 & 0.42 & 0.61 \\
& Janus-Pro-1B \citep{Chen2025JanusProUM} & 0.98 & 0.82 & 0.51 & 0.89 & 0.65 & 0.56 & 0.73 \\
& Show-o \citep{Xie2024ShowoOS} & 0.95 & 0.52 & 0.49 & 0.82 & 0.11 & 0.28 & 0.53 \\
& Transfusion \citep{Zhou2024TransfusionPT} & -- & -- & -- & -- & -- & -- & 0.63 \\
& UniDisc \citep{Swerdlow2025UnifiedMD} & 0.92 & 0.47 & 0.15 & 0.67 & 0.13 & 0.19 & 0.42 \\
& D-DiT \citep{Li2024DualDF} & 0.97 & 0.80 & 0.54 & 0.76 & 0.32 & 0.50 & 0.65 \\
& FUDOKI \citep{wang2025fudoki} & 0.96 & 0.85 & 0.56 & 0.88 & 0.68 & 0.67 & 0.77 \\
& Ours & 0.94 & \cellcolor{gray!25}0.86 & 0.53 &\cellcolor{gray!25} 0.89 &\cellcolor{gray!25} 0.70 & \cellcolor{gray!25}0.77 & \cellcolor{gray!25}0.78 \\
\bottomrule
\end{tabular}
\begin{tablenotes}
\item \textit{Note:} ``Und.'' = Understanding, ``Gen.'' = Generation. $^{\dagger}$ = models integrating an external pretrained diffusion model. Values exceeding the baseline FUDOKI are highlighted in \colorbox{gray!30}{gray}.
\end{tablenotes}
\end{threeparttable}
\end{table}

\begin{table}[ht!]
\centering
\begin{threeparttable}
\caption{{\small Multimodal Understanding Performance on Various Benchmarks.} }\label{tab:multimodal}
\tiny
\begin{tabular}{@{}lp{2.9cm}lcccccc@{}}
\toprule
Type & Model & \# Params & POPE $\uparrow$ & MME-P $\uparrow$ & MMB $\uparrow$ & GQA $\uparrow$ & MMMU $\uparrow$ & MM-Vet $\uparrow$ \\
\midrule
\multirow{13}{*}{Und. Only}
& LLaVA-Phi-1.5 \citep{Liu2023VisualIT} & 1.3B & 84.1 & 1128.0 & - & 56.5 & 30.7 & - \\
& MobileVLM \citep{Chu2023MobileVLMA} & 1.4B & 84.5 & 1196.2 & 53.2 & 56.1 & - & - \\
& MobileVLM-V2 \citep{Chu2024MobileVLMVF} & 1.4B & 84.3 & 1302.8 & 57.7 & 59.3 & - & - \\
& MobileVLM \citep{Chu2023MobileVLMA} & 2.7B & 84.9 & 1288.9 & 59.6 & 59.0 & - & - \\
& MobileVLM-V2 \citep{Chu2024MobileVLMVF} & 2.7B & 84.7 & 1440.5 & 63.2 & 61.1 & - & - \\
& LLaVA-Phi \citep{Zhu2024LLaVAPhiEM} & 2.7B & 85.0 & 1335.1 & 59.8 & - & - & 28.9 \\
& LLaVA \citep{Liu2023VisualIT} & 7B & 76.3 & 809.6 & 38.7 & - & - & 25.5 \\
& LLaVA-v1.5 \citep{Liu2023ImprovedBW} & 7B & 85.9 & 1510.7 & 64.3 & 62.0 & 35.4 & 31.1 \\
& InstructBLIP \citep{Dai2023InstructBLIPTG} & 7B & - & - & 36.0 & 49.2 & - & 26.2 \\
& Qwen-VL-Chat \citep{Bai2023QwenVLAF} & 7B & - & 1487.5 & 60.6 & 57.5 & - & - \\
& IDEFICS \citep{laurencon2023idefics} & 8B & - & - & 48.2 & 38.4 & - & - \\
& Emu3-Chat \citep{Wang2024Emu3NP} & 8B & 85.2 & 1244.0 & 58.5 & 60.3 & 31.6 & 37.2 \\
& InstructBLIP \citep{Dai2023InstructBLIPTG} & 13B & 78.9 & 1212.8 & - & 49.5 & - & 25.6 \\
\midrule
\multirow{15}{*}{Und. \& Gen.}
& LaVIT$^{\dagger}$ \citep{Jin2023UnifiedLP} & 7B & - & - & - & 46.8 & - & - \\
& MetaMorph$^{\dagger}$  \citep{Tong2024MetaMorphMU} & 8B & - & - & 75.2 & - & - & - \\
& Gemini-Nano-1 \citep{Reid2024Gemini1U} & 1.8B & - & - & - & - & 26.3 & - \\
& ILLUME \citep{Wang2024ILLUMEIY} & 7B & 88.5 & 1445.3 & 65.1 & - & 38.2 & 37.0 \\
& TokenFlow-XL \citep{Qu2024TokenFlowUI} & 13B & 86.8 & 1545.9 & 68.9 & 62.7 & 38.7 & 40.7 \\
& LWM \citep{Liu2024WorldMO} & 7B & 75.2 & - & - & 44.8 & - & 9.6 \\
& VILA-U \citep{Wu2024VILAUAU} & 7B & 85.8 & 1401.8 & - & 60.8 & - & 33.5 \\
& Chameleon \citep{Team2024ChameleonME} & 7B & - & - & - & - & 22.4 & 8.3 \\
& Janus \citep{Wu2024JanusDV} & 1.5B & 87.0 & 1338.0 & 69.4 & 59.1 & 30.5 & 34.3 \\
& Janus-Pro-1B \citep{Chen2025JanusProUM} & 1.5B & 86.2 & 1444.0 & 75.5 & 59.3 & 36.3 & 39.8 \\
& Show-o-256 \citep{Xie2024ShowoOS} & 1.3B & 73.8 & 948.4 & - & 48.7 & 25.1 & - \\
& Show-o-512 \citep{Xie2024ShowoOS} & 1.3B & 80.0 & 1097.2 & - & 58.0 & 26.7 & - \\
& D-Dit \citep{Li2024DualDF} & 2.0B & 84.0 & 1124.7 & - & 59.2 & - & - \\
& FUDOKI \citep{wang2025fudoki} & 1.5B & 86.1 & 1485.4 & 73.9 & 57.6 & 34.3 & 38.0 \\
& Ours & 1.5B & \cellcolor{gray!25}86.8 & \cellcolor{gray!25}1492.7 & \cellcolor{gray!25}74.2 & \cellcolor{gray!25}58.2 & \cellcolor{gray!25}35.4 & \cellcolor{gray!25}38.6 \\
\bottomrule
\end{tabular}
\begin{tablenotes}
\item \textit{Note:} ``Und.'' = Understanding, ``Gen.'' = Generation. $^{\dagger}$ = models integrating an external pretrained diffusion model. Values exceeding the baseline FUDOKI are highlighted in \colorbox{gray!30}{gray}.
\end{tablenotes}
\end{threeparttable}
\end{table}


\section{Related Works}

In this section, we review the main related works to ours. A comprehensive coverage of related literature can be found in \cref{appendix:additional related works}.

\noindent {\bf Guidance.}
A line of work has studied guidance for diffusion and flow-matching models in continuous space \citep{dhariwal2021diffusion,lu2023contrastive,ouyang2024transfer,feng2025guidance}. Subsequent efforts extend guidance to discrete space \citep{vignac2022digress,schiff2024simple,li2024derivative,nisonoff2024unlocking}, but these methods typically rely on a first-order approximation to reduce the computational cost. However, this approximation is not appropriate in the discrete setting, which our work seeks to overcome.

\noindent {\bf Discrete Flow-based Models.} Our work primarily builds on a line of discrete flow-based models \citep{campbell2024generative,gat2024discrete,shaul2024flow,yimingdefog}. These models construct transition rates that generate the target probability path through marginalizing conditional transition rate, yielding a more comprehensive design space of conditional probability path than discrete diffusion models \citep{austin2021structured,campbell2022continuous,sun2022score,vignac2022digress,lou2023discrete} including masked diffusion models \citep{shi2024simplified,sahoo2024simple,ou2024your,nie2025large,zhu2025llada,zhao2025d1,yang2025mmada}. This motivates us to propose posterior-based guidance, which provides a flexible and unified perspective for generative modeling through marginalization.

\section{Conclusion}

In this work, we introduce a novel guidance framework for discrete flow-based models. Our approach provides exact guidance with improved sampling efficiency, offering a unified perspective for constructing guidance across generative models. We further propose to train the guidance network by minimizing the Bregman divergence with regularization. The effectiveness of our framework is validated through experiments on energy-guided simulations, preference alignment for text-to-image generation, and multimodal understanding tasks.

\section*{Acknowledgments}

Hongyuan Zha's research is supported in part by Shenzhen Stability Science Program 2023, and National Natural Science Foundation of China (72495131). Fang Fang gratefully acknowledges research support from National Natural Science Foundation of China (72331005). Guang Cheng gratefully acknowledges financial support from the National Science Foundation (NSF–CNS 2247795).

\section*{Ethics statement}
This work adheres to the ICLR Code of Ethics. Our research does not involve human subjects, sensitive personal data, or experiments that may pose harm to individuals or communities. Publicly available datasets and benchmarks were used in all experiments.
\section*{Reproducibility Statement}
The code for our implementation can be found in the supplementary materials. For our theoretical result in \cref{thm:guidance} and \cref{thm:diffusion guidance}, the assumption is stated in \cref{con:support}, and the proof can be found in \cref{ap:proof_thm1} and \cref{ap:proof_thm2}.

\bibliography{ref}
\bibliographystyle{iclr2026_conference}

\clearpage
\appendix
\section{Sampling Algorithms}\label{app:sampling}

Note that the transition probability (\Eqref{eq:transition probability}) is valid only if $h\leq\frac{1}{\abs{u_t^q(x,x)}}$ for each $x\in\mc{S}^\mc{D}$, when we just use its Euler discretization to generate samples. To remove this constraint, following the always-valid sampling procedure introduced in \cite{shaul2024flow}, given the current state $\rvx_t$, for each $d\in\brac{\mc{D}}$, we first sample $\rvx_1^d$ from the learned posterior $q_{1|t}^{d,\theta}(\cdot|\rvx_t)$, and then sample $\rvx_{t+h}^d$ from
\begin{align}\label{eq:valid sampling}
    e^{hu_t^{q,d}(\rvx_t^d,\rvx_t^d|\rvx_1^d)}\delta_{\rvx_t^d}(\cdot)-(1-e^{hu_t^{q,d}(\rvx_t^d,\rvx_t^d|\rvx_1^d)})\frac{u_t^{q,d}(\cdot,\rvx_t^d|\rvx_1^d)}{u_t^{q,d}(\rvx_t^d,\rvx_t^d|\rvx_1^d)}(1-\delta_{\rvx_t^d}(\cdot)),
\end{align}
where we use the first-order approximation
\begin{align*}
    e^{hu_t^{q,d}(\rvx_t^d,\rvx_t^d|\rvx_1^d)}=1+hu_t^{q,d}(\rvx_t^d,\rvx_t^d|\rvx_1^d)+o(h).
\end{align*}
\begin{algorithm}[htbp]
\caption{Sampling without Guidance}
\label{alg:sampling without guidance}
\begin{algorithmic}[1]
\Require pretrained posterior $q_{1|t}$, conditional transition rate $u_t^q(z,x|x_1)$, initial value $x_0$, step size $h$
\State $t \gets 0$
\State $\rvx_t \gets x_0$
\While{$t+h < 1$}
  \For{$d = 1,\dotsc,\mc{D}$} \Comment{\textcolor{brown}{in parallel}}
  \State Calculate $q^{d,\theta}_{1|t}(x_1^d|\rvx_t)$
    \State Sample $\rvx^{d}_{1} \sim q^{d,\theta}_{1\mid t}\!\left(\,\cdot \mid \rvx_t\right)$
    \State $\lambda^{d} \gets \sum_{s\neq \rvx_t^d}u^{q,d}_{t}(s,\rvx^{d}_{t}| \rvx^{d}_{1})$
    \State Sample $Z^{d}_{\text{jump}} \sim U[0,1]$
    \If{$Z^{d}_{\text{jump}} \le 1 - e^{-h\lambda^{d}}$}
      \State Sample $\rvx^{d}_{t} \sim
      \dfrac{u^{q,d}_{t}(\cdot,\rvx^{d}_{t}| \rvx^{d}_{1})}{\lambda^{d}}
      \left(1-\delta_{\rvx^{d}_{t}}(\cdot)\right)$
    \EndIf
  \EndFor
  \State $t \gets t + h$
\EndWhile
\State $t \gets t - h$
\For{$d = 1,\dotsc,\mc{D}$} \Comment{\textcolor{brown}{in parallel}}

    \State Sample $\rvx^{d}_{1} \sim q^{d,\theta}_{1\mid t}\!\left(\,\cdot \mid \rvx_t\right)$
\EndFor
\State \Return $\rvx_1$
\end{algorithmic}
\end{algorithm}

\begin{algorithm}[htbp]
\caption{Sampling with Posterior-Based Guidance}
\label{alg:sampling with posterior-based guidance}
\begin{algorithmic}[1]
\Require pretrained posterior $p_{1|t}$, conditional transition rate $u_t^{p,d}(z,x|x_1)$, initial value $x_0$, posterior-based guidance $H_t^\theta$, step size $h$
\State $t \gets 0$
\State $\rvx_t \gets x_0$
\While{$t+h < 1$}
  \For{$d = 1,\dotsc,\mc{D}$} \Comment{\textcolor{brown}{in parallel}}
  \State \textcolor{orange}{Calculate the guided posterior $q^{d,\theta}_{1|t}(x_1^d|\rvx_t)\propto H_t^\theta(\rvx_t)_{d,x_1^d}p^d_{1|t}(x_1^d|\rvx_t)$}
    \State Sample $\rvx^{d}_{1} \sim q^{d,\theta}_{1\mid t}\!\left(\,\cdot \mid \rvx_t\right)$
    \State $\lambda^{d} \gets \sum_{s\neq \rvx_t^d}u^{p,d}_{t}(s,\rvx^{d}_{t}| \rvx^{d}_{1})$
    \State Sample $Z^{d}_{\text{jump}} \sim U[0,1]$
    \If{$Z^{d}_{\text{jump}} \le 1 - e^{-h\lambda^{d}}$}
      \State Sample $\rvx^{d}_{t} \sim
      \dfrac{u^{p,d}_{t}(\cdot,\rvx^{d}_{t}| \rvx^{d}_{1})}{\lambda^{d}}
      \left(1-\delta_{\rvx^{d}_{t}}(\cdot)\right)$
    \EndIf
  \EndFor
  \State $t \gets t + h$
\EndWhile
\State $t \gets t - h$
\For{$d = 1,\dotsc,\mc{D}$} \Comment{\textcolor{brown}{in parallel}}
\State \textcolor{orange}{Calculate the guided posterior $q^{d,\theta}_{1|t}(x_1^d|\rvx_t)\propto H_t^\theta(\rvx_t)_{d,x_1^d}p^d_{1|t}(x_1^d|\rvx_t)$}
    \State Sample $\rvx^{d}_{1} \sim q^{d,\theta}_{1\mid t}\!\left(\,\cdot \mid \rvx_t\right)$
\EndFor
\State \Return $\rvx_1$
\end{algorithmic}
\end{algorithm}

\section{Additional Related Works}\label{appendix:additional related works}

\noindent{\bf Discrete Flow-based Models.}
Flow-based models for modeling discrete data were introduced by \cite{campbell2024generative,gat2024discrete,shaul2024flow,holderrieth2024generator,yimingdefog}. Compared to discrete diffusion models \citep{austin2021structured,campbell2022continuous,sun2022score,lou2023discrete}, the main advantage of discrete flow models is the flexibility of designing the conditional probability path and the conditional transition rate without specifying corruption transition rate; while discrete diffusion models require choosing a forward transition rate with a simple form such that the conditional probability path can be computed easily. Similar to the framework of continuous flow models \citep{albergo2022building,liu2022flow,lipman2022flow}, discrete flow models focus on the conditional path with probability interpolation, which is a convex combination of conditional probabilities. \cite{shaul2024flow} proposed the mixture probability path with $x_1$-dependent schedulers and the metric-induced path; the training objective they used is the negative ELBO instead of the cross entropy used in \cite{campbell2024generative} and \cite{gat2024discrete}. Recently, \cite{havasi2025editflowsflowmatching} introduced a novel discrete flow matching framework with auxiliary process, which supports variable-length sequence generation. 

\noindent{\bf Discrete Diffusion Models.}
Motivated by studies of discrete-time diffusion models in continuous state spaces \citep{sohl2015deep,ho2020denoising,song2020denoising,karras2022elucidating}, \cite{austin2021structured} and \cite{hoogeboom2021argmax} proposed a discrete-time framework for training diffusion models in discrete state spaces. Similarly, several authors have developed continuous-time discrete diffusion models through continuous-time Markov chains \citep{campbell2022continuous,benton2024denoising}. \cite{sun2022score} proposed discrete score matching based on the idea of continuous score matching for modeling continuous-time diffusion models \citep{song2020score}. \cite{meng2022concrete,lou2023discrete} introduced concrete score matching, which aims to identify the time-reversal of the noising process like \cite{sun2022score}. Empirically, the masked diffusion models \citep{shi2024simplified,sahoo2024simple,ou2024your} outperform the diffusion models with other noise distributions, such as uniform distribution. Moreover, for masked diffusion models, \cite{ou2024your} showed that the training objective is equivalent to that of any-order autoregressive models \citep{hoogeboom2021autoregressive,shih2022training}, which is also equivalent to the ELBO for discrete flow \citep[see Appendix D.1 in][]{shaul2024flow}. The main advantage of discrete diffusion models over (any-order) autoregressive models is parallel sampling. However, supporting variable-length sequence generation is considerably more challenging than in autoregressive models. To overcome this limitation, semi-autoregressive models \citep{nie2025large,arriola2025block} perform block-wise sampling and enable parallel sampling within each block, thus supporting flexible-length sampling.

\noindent{\bf Guidance.}
Classifier guidance \citep{dhariwal2021diffusion} and classifier-free guidance \citep{ho2022classifier,zheng2023guided} are proposed for conditional generation. \cite{schiff2024simple} explored the classifier guidance tailored to discrete diffusion models by decomposing the reversed transition probability for practical purposes. \cite{nie2024scaling} applied classifier-free guidance to masked diffusion models for unsupervised pretraining. In this paper, the posterior-based guidance provides a unified perspective for the existing guidance in previous work, including classifier guidance \cite{dhariwal2021diffusion}, energy-weighted guidance \cite{lu2023contrastive}, guidance for transfer learning \citep{ouyang2024transfer}, and flow matching guidance \citep{feng2025guidance} in continuous state space. See discussion in \cref{discuss:unifiying guidance}. Moreover, thanks to the token-wise transition rate, our discrete guidance only requires training one network, unlike the training-based continuous guidance proposed by \citep{feng2025guidance}, which requires estimating a normalizing constant by learning a surrogate model.

\noindent{\bf Diffusion-based and Flow-based Large Language Models.}
Recently, a number of studies have focused on developing large language diffusion models. Several post-training approaches have been proposed to fine-tune discrete diffusion models aligned with human preference in a way similar to autoregressive models via RLHF \citep{ouyang2022training} or DPO \citep{rafailov2023direct} techniques. LLaDa \citep{nie2025large} leverages masked diffusion models and employs negative ELBO for supervised fine-tuning. It supports variable-length sentence generation by semi-autoregressive, and utilizes a deterministic low-confidence remasking technique to enhance performance similar to MaskGIT \citep{chang2022maskgit}. LLaDa 1.5 \citep{zhu2025llada} applied DPO to LLaDa, using an empirical estimator with variance reduction techniques to approximate the log-probability. To extend GRPO to masked diffusion models, d1-LLaDa \citep{zhao2025d1} proposes to estimate log-probability via mean-field approximation and to do one-step unmasking to estimate token-wise log-probability with partially masked prompts. To learn from multi-step denoising information, MMaDa \cite{yang2025mmada} perturbs sentence outcomes instead of prompts, and conducts policy optimization with a GRPO-type objective, which supports multi-modality. For the large language model based on discrete flow matching, \cite{wang2025fudoki} developed a multi-modal model with a metric-induced probability path and a kinetic-optimal conditional rate \citep{shaul2024flow}, achieving highly competitive performance compared to advanced AR-based multi-modal models.

\section{Background of Continuous-Time Markov Chains}
\label{sec:background CTMC}
Continuous-time Markov chains (CTMCs) are a class of continuous-time Markov processes on discrete state spaces. Given a discrete state space $\mc{S}^\mc{D}$, a stochastic process $\rvx_t$ is CTMC if (a) it has Markov property, that is, given the information of current time, the future and history are independent; (b) given the current time $t$ and state $x\in\mc{S}^\mc{D}$, the transition probability is
\begin{align*}
    \P(\rvx_{t+h}=z|\rvx_{t}=x)=\delta_x(z)+u_t(z,x)h+o(h),
\end{align*}
where $h$ is the step size, and $(u_t(z,x))_{(z,x)\in\mc{S}^\mc{D}\times\mc{S}^\mc{D}}$ is a (time-dynamic) transition rate matrix satisfying
\begin{align*}
    u_t(z,x)\ge0,~\text{ for any }~z\neq x,~\text{ and }~\sum_{z\in\mc{S}^\mc{D}}u_t(z,x)=0~\text{ for any }~x\in\mc{S}^\mc{D}.
\end{align*}
Intuitively, the transition rate $u_t(z,x)$ measures the intensity of transition from the current state $x$ to the target state $z$ at time $t$. This CTMC framework underlies recent advances in generative modeling. Building on the framework of CTMCs, the discrete diffusion models \citep{austin2021structured,campbell2022continuous,sun2022score,lou2023discrete} and the discrete flow models \citep{campbell2024generative,gat2024discrete,shaul2024flow,yimingdefog} can transport a simple distribution like uniform distribution or just a point mass to a complex data distribution through a sequence of jumps.

\section{Derivation of the guidance}
\subsection{Proof of Proposition 3.1 of \citep{campbell2024generative}}
\label{appendix:prop}
We include the proof of Proposition 3.1 of \citep{campbell2024generative} here for completion. We begin by taking the expectation over $p_{1}$ on both sides of the Kolmogorov equation for $p_{t| 1}(x | x_1)$ and $u_t(x_t,z | x_1)$. 

\begin{align*}
\partial_t p_t(x) 
&= \mathbb{E}_{\rvx\sim p_{1}(\rvx_1)}\big[\partial_t p_{t|1}(x |\rvx_1)\big] 
\quad{\small\text{(definition of $p_t$ as marginal over $x_1$)}} \\
&= \mathbb{E}_{\rvx_1\sim p_{1}(x_1)}\big[\sum_z u_t(x, z | \rvx_1) \, p_{t|1}(z | \rvx_1)\big]
\quad{\small\text{(Kolmogorov forward equation)}} \\
&= \sum_z \sum_{x_1} p_{1}(x_1)\,p_{t|1}(z | x_1)\,u_t(x,z| x_1) \\
&= \sum_z \sum_{x_1} p_t(z)\,p_{1|t}(x_1| z)\,u_t(x,z| x_1) 
\quad{\small\text{(Bayes' rule)}} \\
&= \sum_z \mathbb{E}_{\rvx\sim p_{1|t}(\rvx_1 | z)}\big[u_t(x,z| \rvx_1)\big]\,p_t(z).
\end{align*}

We observe that the final expression corresponds to the Kolmogorov equation for a continuous-time Markov chain (CTMC) with marginals $p_t(x)$ and transition rate 
$\mathbb{E}_{\rvx_1\sim p_{1 | t}(\rvx_1 | x)}[u_t(z,x| \rvx_1)]$. 
This demonstrates that $\mathbb{E}_{\rvx_1\sim p_{1 | t}(\rvx_1 | x)}[u_t(z,x| \rvx_1)]$ generates $p_t(x)$.

\subsection{\texorpdfstring{Proof of \cref{thm:guidance}}{}} \label{ap:proof_thm1}
\begin{proof}
By Bayes' formula and the condition $q_{t|1}=p_{t|1}$, we have
\begin{equation}\label{eq:basic formula}
\begin{aligned}
    q_{1|t}(z|x)=&~\frac{q_{t|1}(x|z)q_1(z)}{q_t(x)}=\frac{p_{t|1}(x|z)q_1(z)}{q_t(x)}=\frac{p_{t|1}(x|z)p_1(z)q_1(z)p_t(x)}{p_t(x)q_t(x)p_1(z)}=\frac{q_1(z)p_t(x)}{q_t(x)p_1(z)}p_{1|t}(z|x)\\
    =&~\frac{\frac{q_1(z)}{p_1(z)}}{\sum_{x_1}\frac{q_t(x)}{p_t(x)}q_{1|t}(x_1|x)} p_{1|t}(z|x)=\frac{\frac{q_1(z)}{p_1(z)}}{\sum_{x_1}\frac{q_1(x_1)}{p_1(x_1)}p_{1|t}(x_1|x)} p_{1|t}(z|x)\\
    =&~\frac{r(z)}{\E_{\rvx_1\sim p_{1|t}(\rvx_1|x)}\brac{r(\rvx_1)}}p_{1|t}(z|x),
\end{aligned}
\end{equation}
which implies that
\begin{align*}
    q_{1|t}(z|x)=\frac{r(z)}{\E_{\rvx_1\sim p_{1|t}(\rvx_1|x)}\brac{r(\rvx_1)}}p(\rvx_1^{\bsl d}=z^{\bsl d}|\rvx_t=x,\rvx_1^d=z^d)p^d_{1|t}(z^d|x).
\end{align*}
The result \eqref{eq:guidance} follows by taking summation over $z^{\bsl d}\in\mc{S}^{\mc{D}-1}$ for each $d\in\brac{\mc{D}}$.

For classifier guidance \eqref{eq:classifier guidance}, setting $q_1(\cdot)=p_1(\cdot|y)$ in the previous equation, we can obtain that
\begin{align*}
    p_{1|t}(z|x,y)=&~\frac{p_1(z|y)p_t(x)}{p_t(x|y)p_1(z)}p_{1|t}(z|x)=\frac{p(y|\rvx_1=z)}{p(y|\rvx_t=x)}p_{1|t}(z|x)\\
    =&~\frac{p(y|\rvx_1=z)}{p(y|\rvx_t=x)}p(\rvx_1^{\bsl d}=z^{\bsl d}|\rvx_t=x,\rvx_1^d=z^d)p^d_{1|t}(z^d|x).
\end{align*}
Thus, it suffices to show that $p(y|\rvx_1=z)=p(y|\rvx_1=z,\rvx_t=x)$. To see this, note that
\begin{align*}
    p(y|\rvx_1=z)=&~\frac{ p(\rvx_1=z,\rvy=y)}{p_1(z)}=\frac{p(\rvx_t=x|\rvx_1=z)p(\rvx_1=z,\rvy=y)}{p(\rvx_1=z,\rvx_t=x)}\\
    =&~\frac{p(\rvx_t=x|\rvx_1=z,\rvy=y)p(\rvx_1=z,\rvy=y)}{p(\rvx_1=z,\rvx_t=x)}=p(y|\rvx_1=z,\rvx_t=x),
\end{align*}
which completes the proof.
\end{proof}

\subsection{\texorpdfstring{Proof of \cref{thm:diffusion guidance}}{}}\label{ap:proof_thm2}
\begin{proof}
By Bayes' formula and the condition $q_{t|t+h}=p_{t|t+h}$, we have
\begin{equation}\label{eq:rate basic formula}
    \begin{aligned}
    q_{t+h|t}(z|x)=&~\frac{q_{t|t+h}(x|z)q_{t+h}(z)}{q_t(x)}=\frac{p_{t|t+h}(x|z)q_{t+h}(z)}{q_t(x)}=\frac{p_{t|t+h}(x|z)p_{t+h}(z)q_{t+h}(z)p_t(x)}{p_t(x)q_t(x)p_{t+h}(z)}\\
    =&~\frac{q_{t+h}(z)p_t(x)}{q_t(x)p_{t+h}(z)}p_{t+h|t}(z|x).
\end{aligned}
\end{equation}
By \eqref{eq:transition probability}, we can obtain that
 \begin{align*}
    q_{t+h|t}(z|x)=&~\frac{\sum_{x_t}\set{\delta_{z}(x_t)+u_t^q(z,x_t)h+o(h)}q_t(x_t)p_t(x)}{\sum_{x_t}\set{\delta_{z}(x_t)+u_t^p(z,x_t)h+o(h)}p_t(x_t)q_t(x)}\parenBig{\delta_x(z)+u_t^p(z,x)h+o(h)}\\
    =&~\frac{q_t(z)p_t(x)}{p_t(z)q_t(x)}\parenBig{\delta_x(z)+u_t^p(z,x)h+o(h)}\\
    =&~\delta_x(z)+\frac{q_t(z)p_t(x)}{p_t(z)q_t(x)}u_t^p(z,x)h+o(h)\\
    =&~\delta_x(z)+\frac{\E_{\rvx_1\sim p_{1|t}(\rvx_1|z)}\brac{r(\rvx_1)}}{\E_{\rvx_1\sim p_{1|t}(\rvx_1|x)}\brac{r(\rvx_1)}}u_t^p(z,x)h+o(h),
\end{align*}
where the last equation we use the same trick as the proof of \cref{thm:guidance}. Then, the transition rate of the probability path $q_t$ is
\begin{align*}
    u^q(z,x)=\lim_{h\to0^+}\frac{q_{t+h|t}(z|x)-\delta_x(z)}{h}=\frac{\E_{\rvx_1\sim p_{1|t}(\rvx_1|z)}\brac{r(\rvx_1)}}{\E_{\rvx_1\sim p_{1|t}(\rvx_1|x)}\brac{r(\rvx_1)}}u_t^p(z,x).
\end{align*}

For classifier guidance, considering $q_1(x_1)=p_1(x|y)$, we have
\begin{align*}
    u^q(z,x)=u^p(z,x|y)=\frac{p_t(z|y)p_t(x)}{p_t(z)p_t(x|y)}u_t^p(z,x)=\frac{p(y|\rvx_t=z)}{p(y|\rvx_t=x)}u_t^p(z,x),
\end{align*}
which completes the proof.
\end{proof}


\section{In-Depth Discussions of the Proposed Guidance Framework} \label{ap:indepth}

\subsection{Unifying Guidance through Marginalization Trick}\label{discuss:unifiying guidance}
In this subsection, we provide the guidance formulation for continuous score matching, continuous flow matching, concrete score matching, and discrete flow matching. We assume that the conditional probability path of the source data distribution is the same as that of the target data distribution.

\subsubsection{Continuous Score Matching}
We will use the convention of using $t=0$ for data distribution and $t=T$ for noise in continuous diffusion models. With a little abuse of notation, denote $r(z)=\frac{q_0(z)}{p_0(z)}$. From \eqref{eq:basic formula}, the score function for sampling from $q_0$ is
\begin{align*}
    \nabla_{x}\log q_t(x)=&~\E_{\rvx_0\sim q_{0|t}(\rvx_0|x)}\brac{\nabla_x \log q_{t|0}(x|\rvx_0)}\\
    =&~\int_{x_0}q_{0|t}(x_0|x)\nabla_x\log q_{t|0}(x|x_0)\dd x_0\\
    =&~\int_{x_0}\frac{r(x_0)}{\E_{\rvx_0\sim p_{0|t}(\rvx_0|x)}\brac{r(\rvx_0)}}p_{0|t}(x_0|x)\nabla_x\log p_{t|0}(x|x_0)\dd x_0\\
    =&~\nabla_x\log p_t(x)+\int_{x_0}\frac{r(x_0)-\E_{\rvx_0\sim p_{0|t}(\rvx_0|x)}\brac{r(\rvx_0)}}{\E_{\rvx_0\sim p_{0|t}(\rvx_0|x)}\brac{r(\rvx_0)}}p_{0|t}(x_0|x)\nabla_x\log p_{t|0}(x|x_0)\dd x_0\\
    =&~\nabla_x\log p_t(x)+\int_{x_0}\frac{\frac{q_0(x_0)}{p_0(x_0)}-\frac{q_t(x)}{p_t(x)}}{\E_{\rvx_0\sim p_{0|t}(\rvx_0|x)}\brac{r(\rvx_0)}}\cdot\frac{p_0(x_0)\nabla_x p_{t|0}(x|x_0)}{p_t(x)}\dd x_0\\
    =&~\nabla_x\log p_t(x)+\frac{\frac{\nabla_x p_q(x)}{p_t(x)}-\frac{q_t(x)\nabla_x p_t(x)}{p_t(x)^2}}{\E_{\rvx_0\sim p_{0|t}(\rvx_0|x)}\brac{r(\rvx_0)}}\\
    =&~\nabla_x\log p_t(x)+\nabla_x\log \E_{\rvx_0\sim p_{0|t}(\rvx_0|x)}\brac{r(\rvx_0)},
\end{align*}
where we use the facts $q_{t|0}=p_{t|0}$ and $\E_{\rvx_0\sim p_{0|t}(\rvx_0|x)}\brac{r(\rvx_0)}=\frac{q_t(x)}{p_t(x)}$. This result recovers the guidance given by \cite{ouyang2024transfer}. By setting $q_t(x)=p_t(x|y)$, we also recover the classifier guidance $\nabla_x\log q_t(x)=\nabla_x \log p_t(x)+\nabla_x\log p(y|\rvx_t=x)$.

\subsubsection{Continuous Flow Matching}
Denote $v_t(x)$ and $v_t(x\mid x_1)$ as the marginal and conditional velocity fields, respectively. 
Assuming the conditional kernels satisfy $q_{t\mid 1}(x\mid x_1)=p_{t\mid 1}(x\mid x_1)$, 
and defining $r(x_1):=q_1(x_1)/p_1(x_1)$, we have
\begin{align*}
    v_t^q(x)
    &= \E_{\rvx_1\sim q_{1\mid t}(\rvx_1\mid x)}\brac{v_t(x\mid\rvx_1)}\\
    &= \E_{\rvx_1\sim p_{1\mid t}(\rvx_1\mid x)}
       \brac{v_t(x\mid\rvx_1)
       \frac{q_{1\mid t}(\rvx_1\mid x)}{p_{1\mid t}(\rvx_1\mid x)}}\\
    &= \E_{\rvx_1\sim p_{1\mid t}(\rvx_1\mid x)}
       \brac{v_t(x\mid\rvx_1)
       \frac{r(\rvx_1)}{\E_{\rvx_1\sim p_{1\mid t}(\rvx_1\mid x)}[r(\rvx_1)]}}\\
    &= v_t^p(x)
       + \int_{x_1}
         \parenBig{
         \frac{r(x_1)}{\E_{\rvx_1\sim p_{1\mid t}(\rvx_1\mid x)}[r(\rvx_1)]}
         -1}
         v_t(x\mid x_1)\,
         p_{1\mid t}(x_1\mid x)\,
         \dd x_1.
\end{align*}

which coincides the exact guidance \citep{feng2025guidance} for continuous flow matching.

\subsubsection{Discrete Diffusion Models}
Denote the concrete score \citep{meng2022concrete,lou2023discrete} $s_t(x)=\parenBig{\frac{p_t(z)}{p_t(x)}}_{z^d\neq x^d,z^{\bsl d}=x^{\bsl d}}$, which is a $\mc{D}(\abs{\mc{S}}-1)$-dimensional vector. Then, for any $z,x\in\mc{S}^\mc{D}$ satisfying $z^d\neq x^d,z^{\bsl d}=x^{\bsl d}$, we have
\begin{align*}
    s_t^q(x)_z=&~\E_{\rvx_1\sim q_{1|t}(\rvx_1|x)}\bracBig{\frac{q_{t|1}(z|\rvx_1)}{q_{t|1}(x|\rvx_1)}}\\
    =&~\E_{\rvx_1\sim p_{1|t}(\rvx_1|x)}\bracBig{\frac{p_{t|1}(z|\rvx_1)q_{1|t}(\rvx_1|x)}{p_{t|1}(x|\rvx_1)p_{1|t}(\rvx_1|x)}}\\
    =&~s_t^p(x)_z+\sum_{x_1}\parenBig{\frac{r(x_1)}{\E_{\rvx_1\sim p_{1|t}(\rvx_1|x)}\brac{r(\rvx_1)}}-1}\prod_{d=1}^\mc{D}\frac{p^d_{t|1}(z^d|x_1^d)}{p^d_{t|1}(x^d|x_1^d)}p_{1|t}(x_1|x)\\
    =&~s_t^p(x)_z+\sum_{x_1^d}\parenBig{\frac{\E_{\rvx^{\bsl d}_1\sim p(\rvx^{\bsl d}_1|\rvx^d_1=x_1^d,\rvx_t=x)}\brac{r(\rvx_1)}}{\E_{\rvx_1\sim p_{1|t}(\rvx_1|x)}\brac{r(\rvx_1)}}-1}\frac{p^d_{t|1}(z^d|x_1^d)}{p^d_{t|1}(x^d|x_1^d)}p^d_{1|t}(x_1^d|x),
\end{align*}
where we use $t=1$ for the data distribution. Following the discussion in \cref{appendix:comparison of guidance}, the formulation above can be viewed as a special case of discrete flow matching.

\subsubsection{Discrete Flow Matching}
Similar to the previous formulation, we can obtain that
\begin{align*}
    u_t^{q,d}(z^d,x)=u_t^{p,d}(z^d,x)+\sum_{x_1^d}\parenBig{\frac{\E_{\rvx^{\bsl d}_1\sim p(\rvx^{\bsl d}_1|\rvx^d_1=x_1^d,\rvx_t=x)}\brac{r(\rvx_1)}}{\E_{\rvx_1\sim p_{1|t}(\rvx_1|x)}\brac{r(\rvx_1)}}-1}u^{p,d}_t(z^d,x^d|x_1^d)p^d_{1|t}(x_1^d|x).
\end{align*}
To give a specific example, we consider the $x_1$-independent mixture path and its kinetic-optimal transition rate \citep[see Appendix C of][]{shaul2024flow}:
\begin{align*}
    p^d_{t|1}(x^d|x_1^d)=&~q^d_{t|1}(x^d|x_1^d)=\kappa_t\delta_{x_1^d}(x^d)+(1-\kappa_t)p_0(x^d);\\
    u_t^{q,d}(z^d,x^d|x_1^d)=&~u_t^{p,d}(z^d,x^d|x_1^d)=\frac{\dot{\kappa}_t}{1-\kappa_t}(\delta_{x_1^d}(z^d)-\delta_{x^d}(z^d)).
\end{align*}
Then, we have
\begin{align*}
    u_t^{q,d}(z^d,x)=&~u_t^{p,d}(z^d,x)+\frac{\dot{\kappa}_t}{1-\kappa_t}\parenBig{\frac{\E_{\rvx^{\bsl d}_1\sim p(\rvx^{\bsl d}_1|\rvx^d_1=z^d,\rvx_t=x)}\brac{r(\rvx_1)}}{\E_{\rvx_1\sim p_{1|t}(\rvx_1|x)}\brac{r(\rvx_1)}}-1}p^d_{1|t}(z^d|x)\\
    =&~\frac{\E_{\rvx^{\bsl d}_1\sim p(\rvx^{\bsl d}_1|\rvx^d_1=z^d,\rvx_t=x)}\brac{r(\rvx_1)}}{\E_{\rvx_1\sim p_{1|t}(\rvx_1|x)}\brac{r(\rvx_1)}}u_t^{p,d}(z^d,x)\\
    &~+\frac{\dot{\kappa}_t}{1-\kappa_t}\parenBig{\frac{\E_{\rvx^{\bsl d}_1\sim p(\rvx^{\bsl d}_1|\rvx^d_1=z^d,\rvx_t=x)}\brac{r(\rvx_1)}}{\E_{\rvx_1\sim p_{1|t}(\rvx_1|x)}\brac{r(\rvx_1)}}-1}\delta_{x^d}(z^d),
\end{align*}
which implies that $u_t^{q,d}(z^d,x)$ is an affine combination of $u_t^{p,d}(z^d,x)$ and $-\frac{\dot{\kappa}_t}{1-\kappa_t}\delta_{x^d}(z^d)$ in this case.

\subsection{Discrete Guidance for Discrete-Time Framework}\label{discuss: discrete-time discrete guidance}
In this subsection, we discuss the discrete guidance in the settings of the discrete-time Markov chain \citep[e.g., D3PM,][]{austin2021structured}, which is developed by \cite{vignac2022digress,schiff2024simple,li2024derivative}. For convenience, we also use $t=1$ for data distribution and $t=0$ for initial distribution. Assume that $p_{s|t}=q_{s|t}$ for each $s<t$. By \Eqref{eq:rate basic formula}, the transition probability of $q$ in the sampling process is
\begin{align*}
    q_{t+h|t}(z|x)=\frac{q_{t+h}(z)/p_{t+h}(z)}{q_t(x)/p_t(x)}p_{t+h|t}(z|x)=\frac{\E_{\rvx_1\sim p_{1|t+h}(\rvx_1|z)}\brac{r(\rvx_1)}}{\E_{\rvx_1\sim p_{1|t}(\rvx_1|x)}\brac{r(\rvx_1)}}p_{t+h|t}(z|x),
\end{align*}
which is similar to the soft optimal policy introduced in \cite{li2024derivative}. By taking $q_{t+h|t}(z|x)=p_{t+h|t}(z|x,y)$, we recover the classifier guidance in \cite{vignac2022digress} and \cite{schiff2024simple} with unit guidance strength.

Following the proof of \cref{thm:diffusion guidance}, the rate-based guidance can be viewed as a limiting form of the discrete guidance in discrete time settings. In addition, the discrete-time guidance has a similar challenge to the rate-based guidance; that is, it requires multiple function evaluations in each sampling step. For solving this problem, \cite{vignac2022digress,schiff2024simple} employ the first-order approximation in practice. As pointed out in \cref{remark:first order approximation}, such approximation-based approaches are inappropriate for discrete data intuitively.


\subsection{Comparison between Posterior-Based Guidance and Rate-Based Guidance}
\label{appendix:comparison of guidance}
\cref{thm:guidance} and \cref{thm:diffusion guidance} present the discrete guidance based on the posterior and the transition rate, respectively. The later one generalizes the discrete guidance of \cite{nisonoff2024unlocking}, which requires a stronger condition than the posterior-based guidance in \cref{thm:guidance}. To better understand the conditions in both theorems, we will analyze the relationship between discrete diffusion and discrete flow models. In the following, we consider two different states $x\neq z$.

\noindent{\bf Both Guidance Yield the Same Transition Rates in Discrete Diffusion Models.}
Consider a discrete diffusion model with the corruption transition rate $Q^q_t(x,z)$. Here, we assume that $q_1$ is the density of the data distribution. Following \cite{campbell2022continuous}, the time reversal of $Q^q_t(x,z)$ is
\begin{align}\label{eq:relationship between discrete diffusion and flow models}
    u^q_t(z,x)=\sum_{x_1}Q^q_t(x,z)\frac{q_{t|1}(z|x_1)}{q_{t|1}(x|x_1)}q_{1|t}(x_1|x)=\E_{\rvx_1\sim q_{1|t}(\rvx_1|\rvx_t=x)}\bracBig{Q^q_t(x,z)\frac{q_{t|1}(z|\rvx_1)}{q_{t|1}(x|\rvx_1)}}.
\end{align}
Since $u_t^q(z,x|x_1)\overset{\triangle}{=}Q^q_t(x,z)\frac{q_{t|1}(z|x_1)}{q_{t|1}(x|x_1)}$ is a
conditional transition rate that can generate the conditional path $q_{t|1}$ \citep[see Appendix H.1 of][]{campbell2024generative}, the discrete diffusion models can be viewed as a special case of discrete flow models. 
Therefore, the condition $Q^q_t(x,z)=Q^p_t(x,z)$ as in \cref{thm:diffusion guidance} actually implies that we take the same specific conditional transition rate $u_t^p(z,x|x_1)=u_t^q(z,x|x_1)=Q^q_t(x,z)\frac{q_{t|1}(z|x_1)}{q_{t|1}(x|x_1)}$ in discrete flow models. 

\noindent{\bf They Might Offer Different Marginal Transition Rates for Sampling from Target Distribution.}
Suppose that we are given conditional probability path $q_{t|1}=p_{t|1}$ and conditional transition rate $u_t^q(z,x|x_1)=u_t^p(z,x|x_1)$. What we want to know is whether there exists a common corruption transition rate $Q_t(x,z)$ such that $u^q_t(z,x)$ and $u^p_t(z,x)$ are the reverse transition rates of $Q_t(x,z)$ under the data distributions $q_1$ and $p_1$, respectively. If exists, by the time reversal formula \citep[see, e.g.][]{sun2022score,lou2023discrete}, we have
\begin{align*}
    Q_t(x,z)=&~\frac{p_t(x)}{p_t(z)}u^p_t(z,x)=\sum_{x_1}\frac{p_t(x)}{p_t(z)}u^p_t(z,x|x_1)p_{1|t}(x_1|x)\\
    =&~\sum_{x_1}\frac{p_t(x)}{p_t(z)}Q^p_t(x,z|x_1)\frac{p_{t|1}(z|x_1)}{p_{t|1}(x|x_1)}p_{1|t}(x_1|x)\\
    =&~\frac{q_t(x)}{q_t(z)}u^q_t(z,x)=\sum_{x_1}\frac{q_t(x)}{q_t(z)}u^q_t(z,x|x_1)q_{1|t}(x_1|x)\\
    =&~\sum_{x_1}\frac{q_t(x)}{q_t(z)}Q^q_t(x,z|x_1)\frac{q_{t|1}(z|x_1)}{q_{t|1}(x|x_1)}q_{1|t}(x_1|x),
\end{align*}
which implies that
\begin{align*}
    \sum_{x_1}Q^p_t(x,z|x_1)p_{1|t}(x_1|z)=\sum_{x_1}Q^q_t(x,z|x_1)q_{1|t}(x_1|z),
\end{align*}
where $Q^p_t(x,z|x_1)=Q^q_t(x,z|x_1)=u^q_t(z,x|x_1)\frac{q_{t|1}(x|x_1)}{q_{t|1}(z|x_1)}$ is the time reversal of $u_t^q(z,x|x_1)=u_t^p(z,x|x_1)$. However, since the posteriors $p_{1|t}$ and $q_{1|t}$ are different, the above equation generally does not hold, unless $Q^p_t(x,z|x_1)=Q^q_t(x,z|x_1)$ is independent of $x_1$. If we use the rate-based guidance of \cref{thm:diffusion guidance} with the marginal transition rate $u_t^p(z,x)=\E_{\rvx_1\sim p_{1|t}(\rvx_1|x)}\brac{u_t^p(z,x|\rvx_1)}$ that is used to sample from source data distribution, then we implicitly set the corruption rate $Q_t^q(x,z)=Q_t^p(x,z)=\E_{\rvx_1\sim p_{1|t}(\rvx_1|x)}\brac{u_t^p(z,x|\rvx_1)}\frac{p_t(x)}{p_t(z)}$ in discrete diffusion models. In this scenario, the reversed transition rate for generating probability path $q_t$ is
\begin{align*}
    \E_{\rvx_1\sim p_{1|t}(\rvx_1|x)}\brac{u_t^p(z,x|\rvx_1)}\frac{p_t(x)q_t(z)}{p_t(z)q_t(x)},
\end{align*}
which is not equal to the desired transition rate $\E_{\rvx_1\sim q_{1|t}(\rvx_1|x)}\brac{u_t^q(z,x|\rvx_1)}$ which we obtained based on \cref{thm:guidance} in general. The main reason is that $\E_{\rvx_1\sim q_{1|t}(\rvx_1|x)}\brac{u_t^q(z,x|\rvx_1)}$ and $\E_{\rvx_1\sim p_{1|t}(\rvx_1|x)}\brac{u_t^p(z,x|\rvx_1)}$ do not share the same time reversal as mentioned above. Moreover, if we choose a conditional path $u^q_t(z,x|x_1)$ such that $u^q_t(z,x|x_1)\frac{q_{t|1}(x|x_1)}{q_{t|1}(z|x_1)}$ is independent of $x_1$, then these two methods are equivalent. Formally, we have the following proposition.
\begin{proposition}
Assume that $u_t^q(z,x|x_1)\frac{q_{t|1}(x|x_1)}{q_{t|1}(z|x_1)}$ is unrelated to $x_1$ for two states $x\neq z$. Let $Q_t(x,z)=u_t^q(z,x|x_1)\frac{q_{t|1}(x|x_1)}{q_{t|1}(z|x_1)}$ be the transition rate of noising process. Under the same conditions of \cref{thm:guidance}, we have $Q_t(x,z)\frac{p_t(z)}{p_t(x)}=\E_{\rvx_1\sim p_{t|1}(\rvx_1|x)}\brac{u_t^p(z,x|\rvx_1)}$ and $Q_t(x,z)\frac{q_t(z)}{q_t(x)}=\E_{\rvx_1\sim q_{t|1}(\rvx_1|x)}\brac{u_t^q(z,x|\rvx_1)}$, which means that two methods are equivalent.
\end{proposition}
\begin{proof}
    We only consider the case of $p_t$. Note that
    \begin{align*}
       \E_{\rvx_1\sim p_{t|1}(\rvx_1|x)}\brac{u_t^p(z,x|\rvx_1)}
        =&~\sum_{x_1}u_t^p(z,x|x_1)p_{1|t}(x_1|x)\\
        =&~\sum_{x_1}Q_t^p(x,z)\frac{p_{t|1}(z|x_1)}{p_{t|1}(x|x_1)}p_{1|t}(x_1|x)\\
        =&~Q_t^p(x,z)\frac{p_t(z)}{p_t(x)},
    \end{align*}
    which completes the proof.
\end{proof}

In summary, the condition $Q^q_t(x,z)=Q^p_t(x,z)$ in \cref{thm:diffusion guidance} is more restrictive than \cref{thm:guidance} if our goal is to use the desired transition rate $u_t^q(z,x)=\E_{\rvx_1\sim q_{1|t}(\rvx_1|x)}\brac{u_t^q(z,x|\rvx_1)}$ in the sampling stage. 

\noindent{\bf From Posterior-based Guidance to Rate-based Guidance.}
For completeness, we also give a simple proof of \cref{thm:diffusion guidance} based on the result of \cref{thm:guidance}. Notice that, if $Q^q_t(x,z)=Q^p_t(x,z)$ (which implies that $p_{t|1}=q_{t|1}$), by \eqref{eq:relationship between discrete diffusion and flow models}, we have
\begin{align*}
    \frac{u_t^q(z,x)}{u_t^p(z,x)}=&~\frac{\E_{\rvx_1\sim q_{1|t}(\rvx_1|x)}\bracBig{\frac{q_{t|1}(z|\rvx_1)}{q_{t|1}(x|\rvx_1)}}}{\E_{\rvx_1\sim p_{1|t}(\rvx_1|x)}\bracBig{\frac{p_{t|1}(z|\rvx_1)}{p_{t|1}(x|\rvx_1)}}}=\frac{\E_{\rvx_1\sim p_{1|t}(\rvx_1|x)}\bracBig{\frac{q_{1|t}(\rvx_1|x)q_{t|1}(z|\rvx_1)}{p_{1|t}(\rvx_1|x)q_{t|1}(x|\rvx_1)}}}{\E_{\rvx_1\sim p_{1|t}(\rvx_1|x)}\bracBig{\frac{p_{t|1}(z|\rvx_1)}{p_{t|1}(x|\rvx_1)}}}\\
    =&~\frac{\E_{\rvx_1\sim p_{1|t}(\rvx_1|x)}\bracBig{\frac{q_t(z)q_{1|t}(\rvx_1|z)}{q_t(x)p_{1|t}(\rvx_1|x)}}}{\E_{\rvx_1\sim p_{1|t}(\rvx_1|x)}\bracBig{\frac{p_{t|1}(z|\rvx_1)}{p_{t|1}(x|\rvx_1)}}}=\dfrac{q_t(z)/q_t(x)}{p_t(z)/p_t(x)}=\dfrac{q_t(z)/p_t(x)}{q_t(z)/p_t(x)}\\
    =&~\frac{\E_{\rvx_1\sim p(\rvx_1|z)}\brac{r(\rvx_1)}}{\E_{\rvx_1\sim p(\rvx_1|x)}\brac{r(\rvx_1)}},
\end{align*}
where the third equation we use $\frac{q_{1|t}(\rvx_1|x)}{p_{1|t}(\rvx_1|x)}=\frac{q_1(\rvx_1)p_t(x)}{p_1(\rvx_1)q_t(x)}$ in \eqref{eq:basic formula} and the fourth equation we use the marginalization trick in the concrete score matching \citep{meng2022concrete,lou2023discrete}. This completes the proof.

\subsection{Classification of Discrete Guidance}
As discussed in \cref{discuss: discrete-time discrete guidance}, the rate-based guidance can be interpreted as the limiting case of the discrete guidance in the discrete-time framework, since all of them require the condition: $p_{s|t}(x_s|x_t)=q_{s|t}(x_s|x_t)$ for any $0\leq s<t\leq 1$ (here, we use $t=1$ for data distribution). In contrast, the posterior-based guidance (\cref{thm:guidance}) only imposes the condition of conditional probability paths, offering greater flexibility and interpretability for discrete flow models. Accordingly, discrete guidance can be divided into two categories, as summarized in \cref{table:classification of discrete guidance}.

\begin{table}[htbp]
    \centering
    \caption{\small {Classification of discrete guidance. (Data distribution: $t=1$)}}\label{table:classification of discrete guidance}
    \vskip 0.1in
    \begin{small}
    \resizebox{0.99\textwidth}{!}{%
    \begin{tabular}{lllc}
    \toprule
    Assumption & Time Framework & Discrete Guidance & Target Transition Rate for Discrete Flow Models \\
    \midrule
    \vspace{0.03in}
    $p_{t|1}(x_t|x_1)=q_{t|1}(x_t|x_1)$ for any $0\leq t\leq 1$ 
        &  Continuous
        & Posterior-Based Guidance (\cref{thm:guidance}) 
        & $\E_{\rvx_1\sim q_{1|t}(\rvx_1|x)}\brac{u_t(z,x|\rvx_1)}$\\
    \vspace{0.03in}
    \multirow{2}{*}{$p_{s|t}(x_s|x_t)=q_{s|t}(x_s|x_t)$ for any $0\leq s<t\leq 1$}
        & Continuous 
        & Rate-Based Guidance (\cref{thm:diffusion guidance} \& \cite{nisonoff2024unlocking})
        & $\E_{\rvx_1\sim p_{1|t}(\rvx_1|x)}\brac{u_t(z,x|\rvx_1)}\frac{p_t(x)q_t(z)}{p_t(z)q_t(x)}$\\
    & Discrete 
        & \cite{vignac2022digress}\& \cite{schiff2024simple}\& \cite{li2024derivative} 
        & \textemdash \\
    \bottomrule
    \end{tabular}}
    \end{small}
\end{table}


\subsection{Discussion on the Guidance Strength}\label{discuss:guidance strenth}
In this subsection, we focus on the settings considered in \cite{zheng2023guided,zhang2025energy}. Our goal is to generate data from the guided distribution $p_1^{(\gamma)}(x|y)\propto p_1(x)p^\gamma(y|x)$, where $\gamma\ge 0$ is the guidance strength. Lemma 4.10 in \cite{zhang2025energy} provides the comparison between contrastive energy prediction \cite{lu2023contrastive} and classifier guidance \cite{dhariwal2021diffusion,ho2022classifier}. Here, we compare the rate-based guidance with the predictor guidance \cite{nisonoff2024unlocking} under different guidance strengths.

The guided transition rate with predictor guidance strength $\gamma$ is
\begin{align}\label{eq:predictor guidance with strength}
    u_t^{(\gamma)}(z,x|y)=\bracBig{\frac{p_t(y|z)}{p_t(y|x)}}^\gamma u_t(z,x)
\end{align}
for $z\neq x$. Here, $u_t$ is a transition rate that can generate $p_t(x)$. In general, if $\gamma\neq 1$, the above guided rate $u_t^{(\gamma)}$ cannot generate $p_t^{(\gamma)}(x|y)= \frac{p_t(x)p_t^\gamma(y|x)}{\mc{Z}_t(y; \gamma)}$, where $\mc{Z}_t(y; \gamma)=\sum_x p_t(x)p_t^\gamma(y|x)$. To see this, we try to verify the Kolmogorov forward equation:
\begin{align*}
    &~\sum_{x\neq z}u_t^{(\gamma)}(z,x|y)\frac{p_t(x)p_t^\gamma(y|x)}{\mc{Z}_t(y; \gamma)}-\sum_{x\neq z}u_t^{(\gamma)}(x,z|y)\frac{p_t(z)p_t^\gamma(y|z)}{\mc{Z}_t(y; \gamma)}\\
    =&~\sum_{x\neq z}\bracBig{\frac{p_t(y|z)}{p_t(y|x)}}^\gamma u_t(z,x)\frac{p_t(x)p_t^\gamma(y|x)}{\mc{Z}_t(y; \gamma)}-\sum_{x\neq z}\bracBig{\frac{p_t(y|x)}{p_t(y|z)}}^\gamma u_t(x,z)\frac{p_t(z)p_t^\gamma(y|z)}{\mc{Z}_t(y; \gamma)}\\
    =&~\sum_{x\neq z}u_t(z,x)\frac{p_t(x)p_t^\gamma(y|z)}{\mc{Z}_t(y; \gamma)}-\sum_{x\neq z}u_t(x,z)\frac{p_t(z)p_t^\gamma(y|x)}{\mc{Z}_t(y; \gamma)}\\
    =&~\frac{p_t^\gamma(y|z)}{\mc{Z}_t(y; \gamma)}\partial_t p_t(z)+p_t(z)\sum_{x\neq z}u_t(x,z)\frac{\brac{p_t^\gamma(y|z)-p_t^\gamma(y|x)}}{\mc{Z}_t(y; \gamma)}\\
    \neq&~\frac{p_t^\gamma(y|z)}{\mc{Z}_t(y; \gamma)}\partial_t p_t(z)+p_t(z)\partial_t \bracBig{\frac{p_t^\gamma(y|z)}{\mc{Z}_t(y; \gamma)}}\\
    =&~\partial_t p_t^{(\gamma)}(z|y).
\end{align*}
Therefore, in this case, we cannot sample from $p_1^{(\gamma)}$ with the transition rate $u_t^{(\gamma)}$; for empirical results, see the simulations in \cref{sec:simulation}. When $\gamma=1$, the Kolmogorov forward equation holds, since
\begin{align*}
    &~p_t(z)\partial_t \bracBig{\frac{p_t^\gamma(y|z)}{\mc{Z}_t(y; \gamma)}}=p_t(z)\partial_t\bracBig{\frac{p_t(z|y)}{p_t(z)}}\\
    =&~\partial_tp_t(z|y)-p_t(z|y)\frac{\partial_tp_t(z)}{p_t(z)}\\
    =&~-\sum_xQ_t(z,x)p_t(x|y)+\frac{p_t(z|y)}{p_t(z)}\sum_x Q_t(z,x)p_t(x)\\
    =&~-\sum_{x\neq z}\bracBig{Q_t(z,x)p_t(x|y)-\frac{p_t(z|y)}{p_t(z)} Q_t(z,x)p_t(x)}\\
    =&~p_t(z)\sum_{x\neq z}u_t(x,z)\frac{\brac{p_t(y|z)-p_t(y|x)}}{p(y)},
\end{align*}
where $Q_t$ is the transition rate that can generate $p_t(x|y)$ and $p_t(x)$ from $t=1$ to $t=0$.

In contrast, as demonstrated in \cref{thm:diffusion guidance}, the rate-based guidance
\begin{align}\label{eq:proposed rate-based with strength}
    u_t^{(\gamma)}(z,x|y)=\bracBig{\frac{\E_{p_{1|t}(\rvx_1|z)} p_1^\gamma(y|\rvx_1)}{\E_{p_{1|t}(\rvx_1|x)} p_1^\gamma(y|\rvx_1)}}u_t(z,x)
\end{align}
allows us to sample from the guided distribution $p_1^{(\gamma)}$ with the probability path $p_t^{(\star,\gamma)}(x|y)\propto\sum_{x_1}p_1(x_1)p_1^\gamma(y|x_1)p_{t|1}(x|x_1)$. The difference between the proposed rate-based guidance (\eqref{eq:proposed rate-based with strength}) and the predictor guidance (\eqref{eq:predictor guidance with strength}) in discrete state space is similar to that between the exact guidance and the classifier guidance in continuous state space; see Lemma 4.10 of \cite{zhang2025energy}.

\section{Training and Sampling for rate-based guidance}\label{ap:rate_based}
\begin{figure}[t]
    \centering    \includegraphics[width=\linewidth]{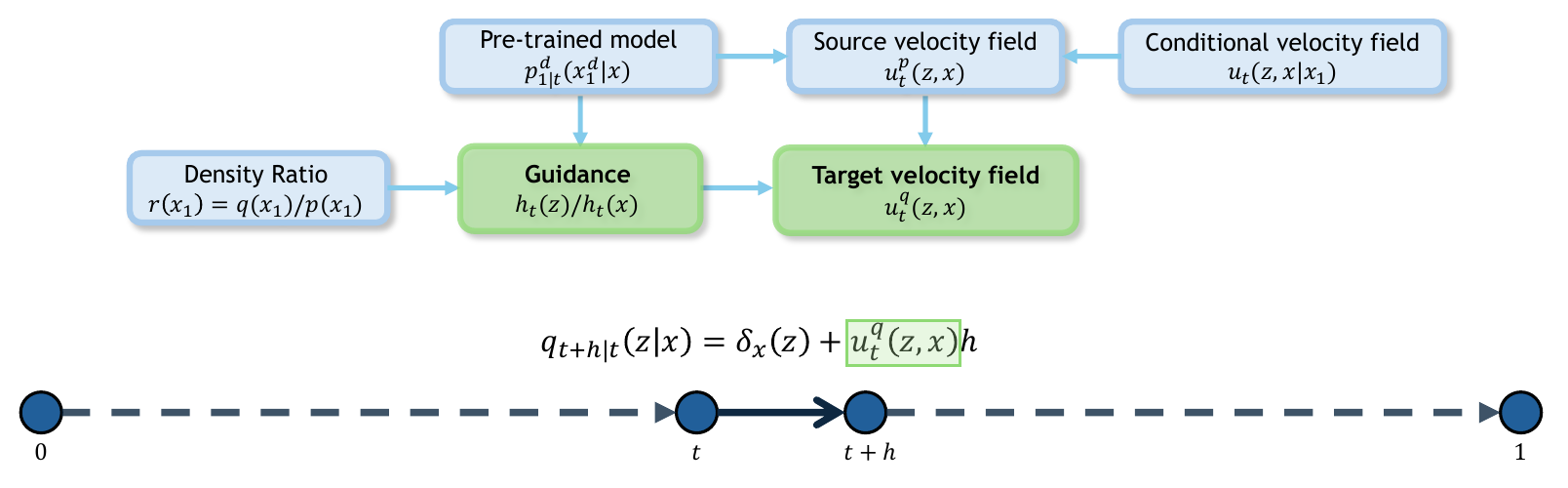}
    \vspace{-0.25in}
    \caption{Framework of the improved rate-based guidance on discrete flow matching. }
    \label{fig:sampling_framework_rate}
\end{figure} 
\subsection{Training Objective}
The rate-based guidance is the following scaler-valued function
\begin{align*}
    h_t(x)\overset{\triangle}{=}\E_{\rvx_1\sim p_{1|t}(\rvx_1|x)}\brac{r(\rvx_1)}=\sum_{s\in\mc{S}}h_t^{d}(s,x)p_{1|t}^d(s,x),~\text{ for any }d\in\brac{\mc{D}},
\end{align*} 
which is also the minimizer of the following objective:
\begin{align}\label{eq:training objective for rate-based guidance}
    \mc{L}_{h,p}^{\text{rate}}(\theta)=\E_{t\sim \mc{U}([0,1]),\rvx_1\sim p_1(\rvx_1),\rvx_t\sim p_{t|1}(\rvx_t|\rvx_1)}\bracBig{h_t^{\theta}(\rvx_t)-r(\rvx_1)\log h_t^{\theta}(\rvx_t)}.
\end{align}
In practice, we can directly use this training objective to avoid additional summation operations.

\subsection{Sampling}
We can construct a new conditional quantity $\check{u}_t^{q,d}(z^d,\rvx_t^d|\rvx_1^d)=\frac{h_t^\theta(z)}{h_t^\theta(\rvx_t)}u_t^p(z^d,\rvx_t^d|\rvx_1^d)$ for $z^d\neq\rvx_t^d$, $z^{\bsl d}=\rvx^{\bsl d}$ and $\check{u}_t^{q,d}(\rvx_t^d,\rvx_t^d|\rvx_1^d)=-\sum_{z^d\neq\rvx_t^d}\check{u}_t^{q,d}(z^d,\rvx_t^d|\rvx_1^d)$, and then use it instead of $u_t^{q,d}$ in \eqref{eq:valid sampling}. The sampling algorithm can be found in  \cref{alg:sampling with rate-based guidance}, and the overall framework is illustrated in \cref{fig:sampling_framework_rate}.

\begin{algorithm}[htbp]
\caption{Sampling with Rate-Based Guidance}
\label{alg:sampling with rate-based guidance}
\begin{algorithmic}[1]
\Require pretrained posterior $p_{1|t}$, conditional transition rate $u_t^p(z,x|x_1)$, initial value $x_0$, rate-based guidance $h_t^\theta$, step size $h$
\State $t \gets 0$
\State $\rvx_t \gets x_0$
\While{$t+h < 1$}
  \For{$d = 1,\dotsc,\mc{D}$} \Comment{\textcolor{brown}{in parallel}}
    \State Sample $\rvx^{d}_{1} \sim p^{d}_{1\mid t}\!\left(\,\cdot \mid \rvx_t\right)$
    \State \textcolor{orange}{Calculate $\hat{u}_t^{d,\theta}(s,\rvx_t^d|\rvx_1^d)=\dfrac{h_t^\theta(\rvx_t^{1},\dots,\rvx_t^{d-1},s,\rvx_t^{d+1},\dots,\rvx_t^\mc{D})}{h_t^\theta(\rvx_t)}u_t^{p,d}(s,\rvx_t^d|\rvx_1^d), s\neq \rvx_t^d$}
    \State $\lambda^{d} \gets \sum_{s\neq \rvx_t^d}\hat{u}_t^{d,\theta}(s,\rvx_t^d|\rvx_1^d)$
    \State Sample $Z^{d}_{\text{jump}} \sim U[0,1]$
    \If{$Z^{d}_{\text{jump}} \le 1 - e^{-h\lambda^{d}}$}
      \State Sample $\rvx^{d}_{t} \sim
      \dfrac{\hat{u}_t^{d,\theta}(\cdot,\rvx^{d}_{t}| \rvx^{d}_{1})}{\lambda^{d}}
      \left(1-\delta_{\rvx^{d}_{t}}(\cdot)\right)$
    \EndIf
  \EndFor
  \State $t \gets t + h$
\EndWhile
\State $t \gets t - h$
\For{$d = 1,\dotsc,\mc{D}$} \Comment{\textcolor{brown}{in parallel}}
    \State Sample $\rvx^{d}_{1} \sim p^{d,\theta}_{1\mid t}\!\left(\,\cdot \mid \rvx_t\right)$
\EndFor
\State \Return $\rvx_1$
\end{algorithmic}
\end{algorithm}
\section{Application to the RL Stage in RLHF}
\label{ap:rlhf}
\subsection{Background and Motivation}
In the diffusion language models \citep[e.g. LLaDa,][]{nie2025large} and flow-based language models \citep[e.g. Fudoki,][]{wang2025fudoki}, we can obtain the pretrained model. Recently, existing works have developed post-training methods to improve preference through DPO \citep{rafailov2023direct} or RLHF \citep{ouyang2022training}. The main challenge is that the joint posterior predictor cannot be easily decoupled like AR models. To overcome this issue, some approximation strategies have been proposed for estimating log probability, such as empirical ELBO approximation \citep{zhu2025llada} and mean-field approximation \citep{zhao2025d1,yang2025mmada}. However, these methods might introduce large approximation errors when the length of the sentence is large. Fortunately, our discrete guidance can provide a novel framework for alignment of diffusion and flow-based language models without approximation.


\subsection{Roadmap of Alignment for Flow-based Language Models}

\noindent {\bf Designing conditional path and rate.} Similar to \cite{wang2025fudoki}, in practice, we take metric-induced probability path with kinetic-optimal conditional transition rate \cite{shaul2024flow}:
\begin{equation}\label{eq:metric-induced path and rate}
    \begin{aligned}
    p_{t|1}(\rvo_t^d|\rvo_1^d)=&~q_{t|1}(\rvo_t^d|\rvo_1^d)=\text{softmax}(-\beta_t\dd(\rvo_t^d,\rvo_1^d));\\
    u_t^{p,d}(z^d,x^d|\rvo_1^d)=&~u_t^{q,d}(z^d,x^d|\rvo_1^d)=p_{t|1}(x|\rvo_1^d)\dot{\beta}_t\brac{\dd(x^d,\rvo_1^d)-\dd(z^d,\rvo_1^d)}_+ \text{ ~for~ }z^d\neq x^d,
\end{aligned}
\end{equation}

where $\beta_t$ is an increasing function of $t$ with $\beta_0=0,\beta_1=\infty$, and $\dd(\cdot,\cdot)$ is a metric.

\noindent {\bf Training guidance model.} We can simply extend the unconditional training objective in the previous section to the following conditional version:
\begin{align}\label{eq_for_RLHF}
        &~\mc{L}_{h,\pi_{ref}}(\theta)+\lambda\mc{L}_{h,\pi^*}(\theta) \notag\\
        =&~\E_{t\sim \mc{U}([0,1]),\rvc\sim p_\rvc,\rvo_1|\rvc\sim\pi_{ref},\rvo_t|\rvo_1\sim p_{t|1}}\bracBig{\sum_{d=1}^\mc{D}h_t^{d,\theta}(\rvo_1^d,\rvo_t,\rvc)-\exp\parenBig{\frac{\mc{R}(\rvc,\rvo_1)}{\tau}}\log h_t^{d,\theta}(\rvo_1^d,\rvo_t,\rvc)}\notag\\
        &~+\lambda\E_{t\sim \mc{U}([0,1]),\rvc\sim p_\rvc,\rvo_1|\rvc\sim\pi^*,\rvo_t|\rvo_1\sim q_{t|1}}\bracBig{\sum_{d=1}^\mc{D}\log \parenBig{\sum_{s\in\mc{S}}h^{d,\theta}_t(s,\rvo_t,\rvc)p_{1|t}^d(s|\rvc,\rvo_t)}-\log h^{d,\theta}_t(\rvo_1^d,\rvo_t,\rvc)}.
\end{align}
\noindent {\bf Sampling.} Given the prompt $\rvc$, at the current time $t$, for each $d\in\brac{\mc{D}}$, we first sample $\rvo_1^d$ from
\begin{align*}
    q^{d,\theta}_{1|t}(\rvo_1^d|\rvc,\rvo_t)\propto h^{d,\theta}_t(\rvo_1^d,\rvo_t,\rvc)p^{d}_{1|t}(\rvo_1^d|\rvc,\rvo_t),
\end{align*}
and then sample $\rvo_{t+h}^d$ similar to \eqref{eq:valid sampling}.

\section{Additional Results}

\subsection{Implementation Details of Energy-Guided Sampling}\label{ap:sim_detail}
For a fair comparison, we train the guidance models (conditional expectation) using the Bregman divergence (\Eqref{eq:training objective} and \Eqref{eq:training objective for rate-based guidance}) without regularization. For the rate-based guidance and predictor guidance, we take the initial distribution $p_0=p_0^{(\gamma)}=\delta_{\mask}$ for efficient sampling, where $\mask$ is the mask state. For sampling with the posterior-based guidance, we consider both masked and uniform initial distributions. We used the mixture probability path and the associated transition rate (\Eqref{eq:mixture path}) with the cosine time schedule $\kappa_t=\cos^2\brac{\frac{\pi}{2}(1-t)}$ in our 2-D experiments, which is kinetic optimal as mentioned in \cite{shaul2024flow}.

Our pretrained model and guidance models are SiLU networks with 3 hidden layers with dimension 256. For training, we use 100,000 datapoints from source data distribution. The optimizer is Adam with learning rate 1e-4.

{\subsection{Implementation Details of Multimodal Tasks}\label{sec:ap_imple_multimodal}
To match the conditional probability path used in \cite{wang2025fudoki}, we use metric-induced path with metric $d(f(x),f(z))=\norm{\tilde{x}-\tilde{z}}^4$, where $\tilde{x}, \tilde{z}$ are normalized token embeddings which are taken from the original text embedding layer of Janus-Pro-7B and the image embeddings of LlamaGen; for time schedule $\beta_t$, we set $\beta_t=3(\frac{t}{1-t})^{0.9}$ as suggested in \cite{shaul2024flow}. }

{Our guidance model is a network composed of 6 LLaDA blocks \citep{Nie2025LargeLD} with hidden dimension 768, a time embedding mapping $\R$ to a hidden dimension of 768, a token embedding mapping each token into a 768-dimensional space. }

\subsection{Effectiveness of Regularization in a 2-D Experiment}\label{ap:more_exp}
In this subsection, we examine the effectiveness of the regularization term in the training objective. Samples of the source and target distributions are shown in \cref{fig:source_target_samples}, and the two distributions differ substantially. We first pretrain a posterior model on the source data, subsequently pretrain a density ratio model by minimizing the following training objective:
\begin{align*}
    \mc{L}_r(\theta)=-\E_{\rvx_1\sim p_1(\rvx_1)}\log\parenBig{\frac{1}{r^\theta(\rvx_1)+1}}-\E_{\rvx_1\sim q_1(\rvx_1)}\log\parenBig{\frac{r^\theta(\rvx_1)}{r^\theta(\rvx_1)+1}}.
\end{align*} We then train posterior-based guidance with the proposed training objective $\mc{L}_{h,p}(\theta)+\lambda\mc{L}_{h,q}(\theta)$, sweeping $\lambda$ from $0$ to $1$ in steps of $0.2$. Here, we take the uniform initial distribution. Finally, we generate samples using \cref{alg:sampling with posterior-based guidance}. As presented in \cref{fig:2-d hyperparameter}, the posterior-based guidance trained with a large hyperparameter offers a distribution closer to the target distribution.

\begin{figure}[htbp]
  \centering
  \begin{minipage}[t]{0.4\linewidth}
  \centering
\includegraphics[width=\linewidth]{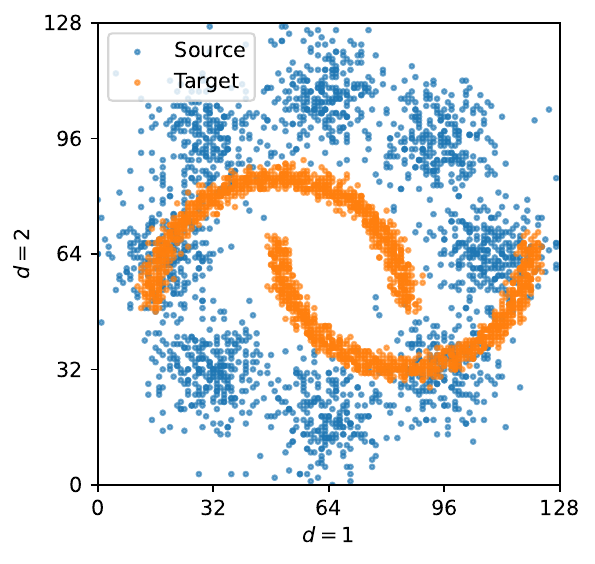}
\end{minipage}
\caption{\small Some samples of source and target distributions.}
  \label{fig:source_target_samples}
\end{figure}

\begin{figure}[htbp]
\centering
\begin{minipage}[t]{0.95\linewidth}
  \centering
  \includegraphics[width=\linewidth]{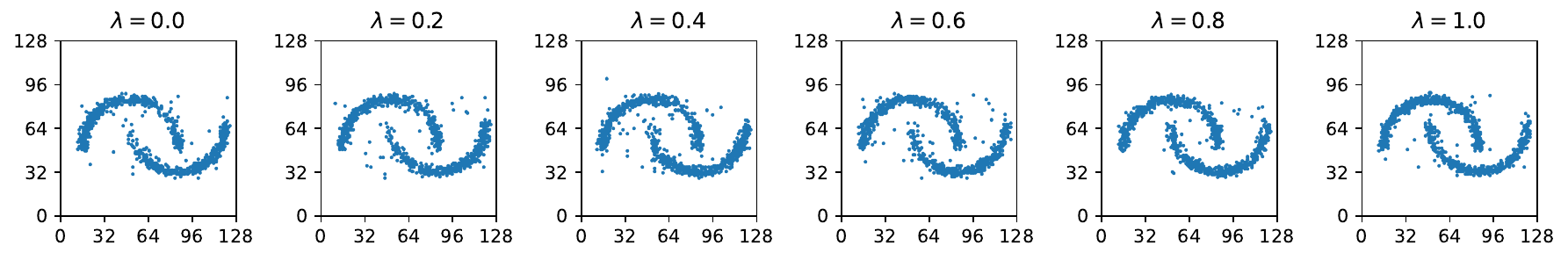}\\[-2pt]
  {\scriptsize }
\end{minipage}

\vspace{3pt}

\begin{minipage}[t]{0.95\linewidth}
  \centering
  \includegraphics[width=\linewidth]{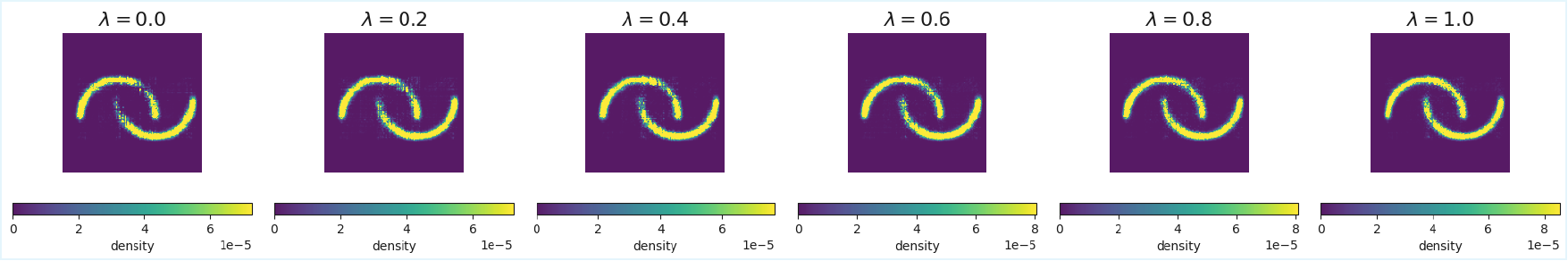}\\[-2pt]
  {\scriptsize }
\end{minipage}

\caption{\small First row: samples generated by posterior-based guidance (64 sampling steps) with different hyperparameters in the training objective. The initial distribution is uniform. Second row: density estimation of the target distribution using posterior-based guidance with different hyperparameters.}
\label{fig:2-d hyperparameter}
\end{figure}


\subsection{Additional 2-D Results}
Additional results on 2-D simulations are provided in \cref{fig:toy main}, \cref{ap:image_moons}, \cref{ap:8gaussians}, \cref{ap:2spirals}, \cref{ap:checkerboard}, and \cref{ap:swissroll}.

\begin{figure}[htbp]
\centering

\begin{minipage}[t]{0.49\linewidth}
  \centering
  \includegraphics[width=\linewidth]{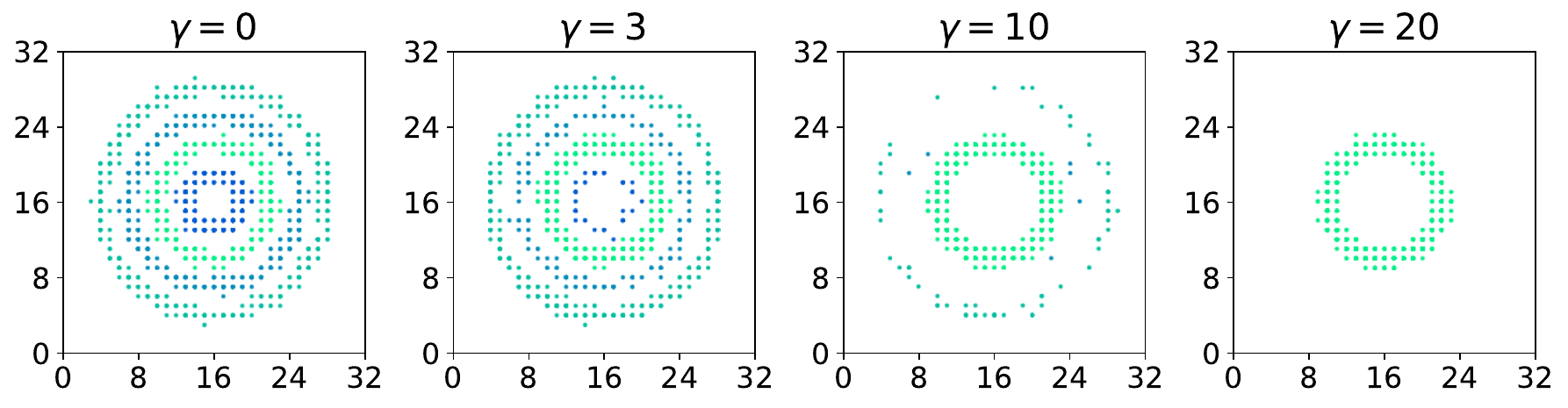}\\[-2pt]
  {\scriptsize (a) \textbf{Ground Truth (rings)}}
\end{minipage}

\vspace{3pt}

\begin{minipage}[t]{0.49\linewidth}
  \centering
  \includegraphics[width=\linewidth]{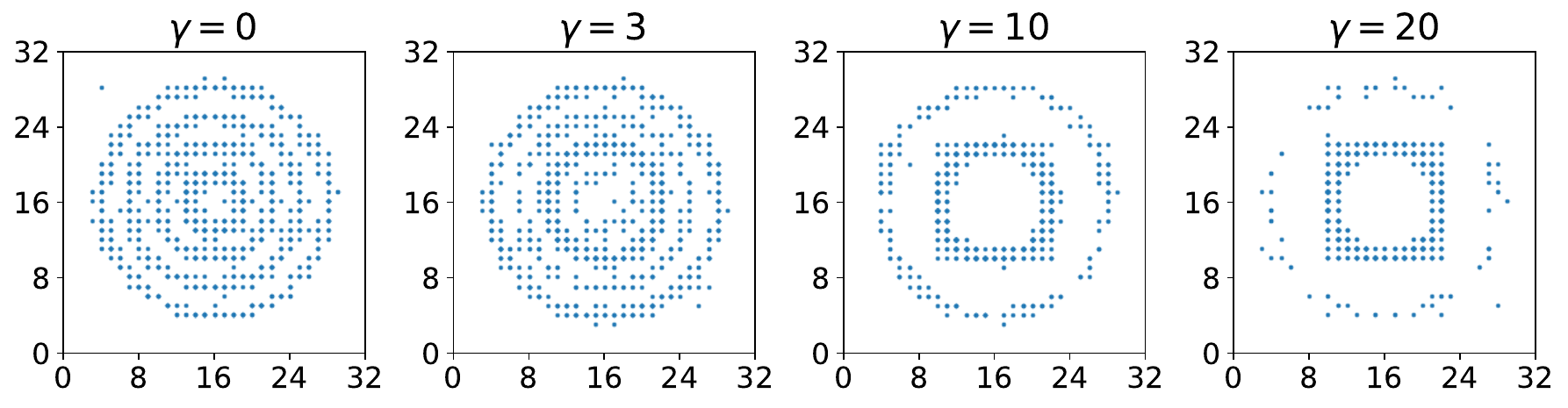}\\[-2pt]
  {\scriptsize (b) Rate-Based (Masked, \cite{nisonoff2024unlocking})}
\end{minipage}\hfill
\begin{minipage}[t]{0.49\linewidth}
  \centering
  \includegraphics[width=\linewidth]{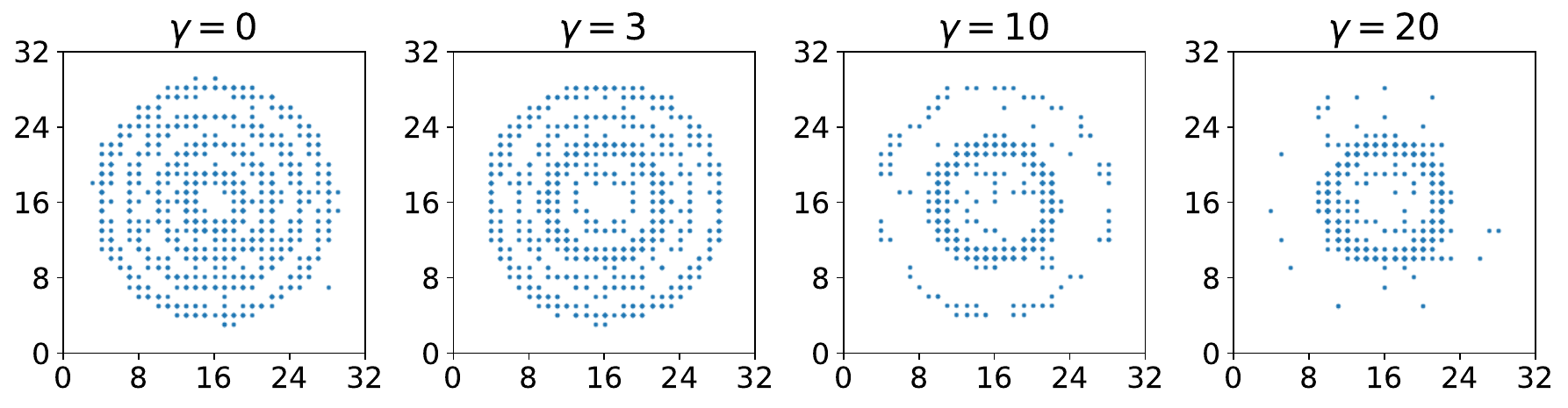}\\[-2pt]
  {\scriptsize (c) Rate-Based (Masked, Ours)}
\end{minipage}

\vspace{3pt}

\begin{minipage}[t]{0.49\linewidth}
  \centering
  \includegraphics[width=\linewidth]{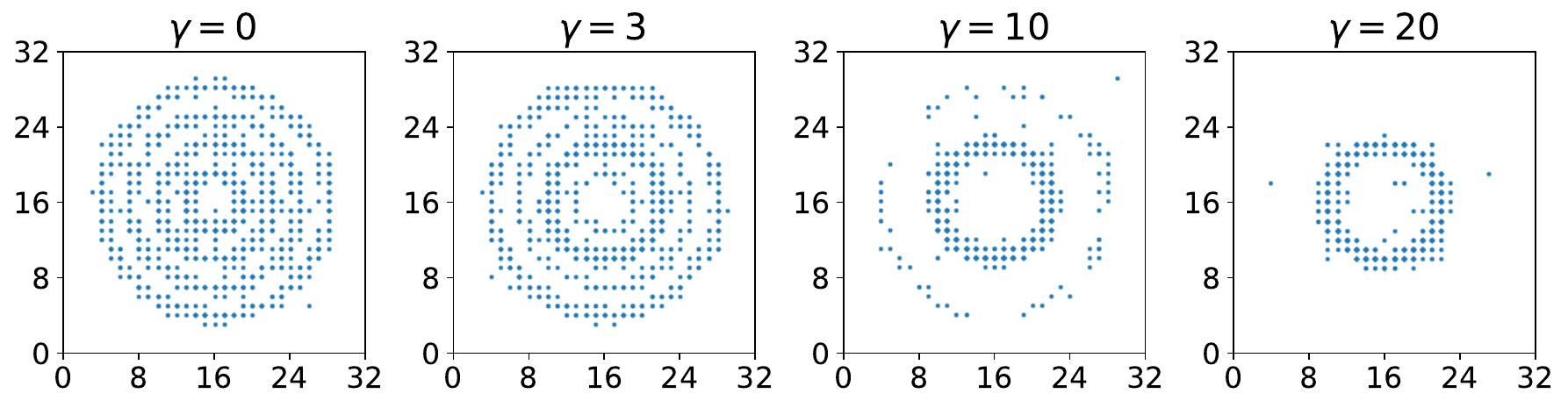}\\[-2pt]
  {\scriptsize (d) Posterior-Based (Masked, Ours)}
\end{minipage}\hfill
\begin{minipage}[t]{0.49\linewidth}
  \centering
  \includegraphics[width=\linewidth]{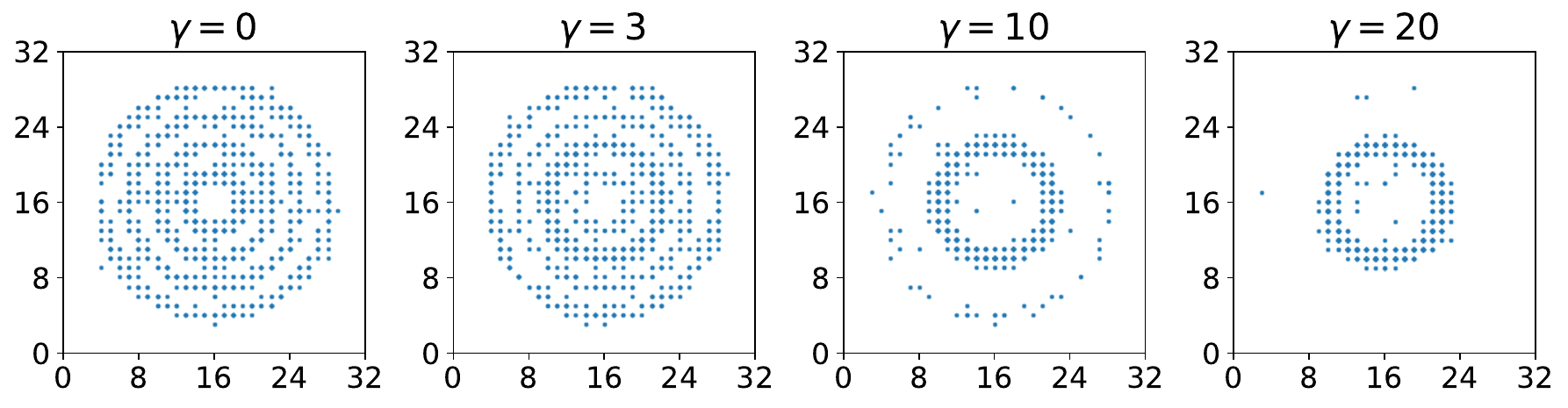}\\[-2pt]
  {\scriptsize (e) Posterior-Based (Uniform, Ours)}
\end{minipage}

\caption{\small Comparison of sampling results with different guidance schemes (64 sampling steps) in 2-D experiments. (a) The samples of the guided distribution $p_1^{(\gamma)}$ with different guidance strength $\gamma$; (b) the data sampled by the predictor guidance \citep{nisonoff2024unlocking} with masked initial distribution; (c) the data sampled by the proposed rate-based guidance with masked initial distribution; (d) the data sampled by the proposed posterior-based guidance with masked initial distribution; (e) the data sampled by the proposed posterior-based guidance with uniform initial distribution.}
\label{fig:toy main}
\end{figure}

\begin{figure}[htbp]
\centering

\begin{minipage}[t]{0.49\linewidth}
  \centering
  \includegraphics[width=\linewidth]{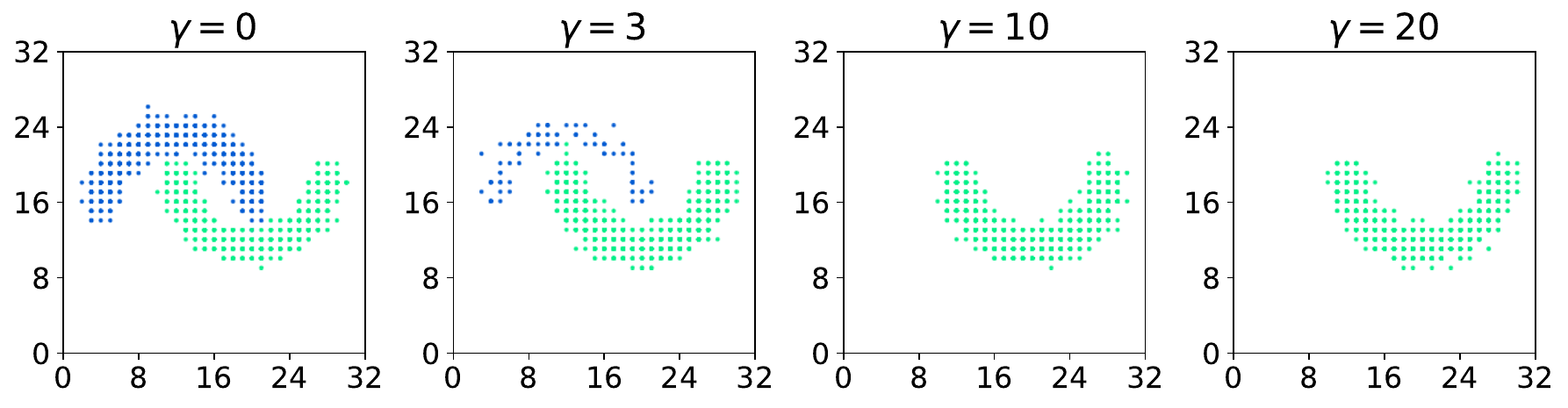}\\[-2pt]
  {\scriptsize \textbf{Ground Truth (moons)}}
\end{minipage}

\vspace{3pt}

\begin{minipage}[t]{0.49\linewidth}
  \centering
  \includegraphics[width=\linewidth]{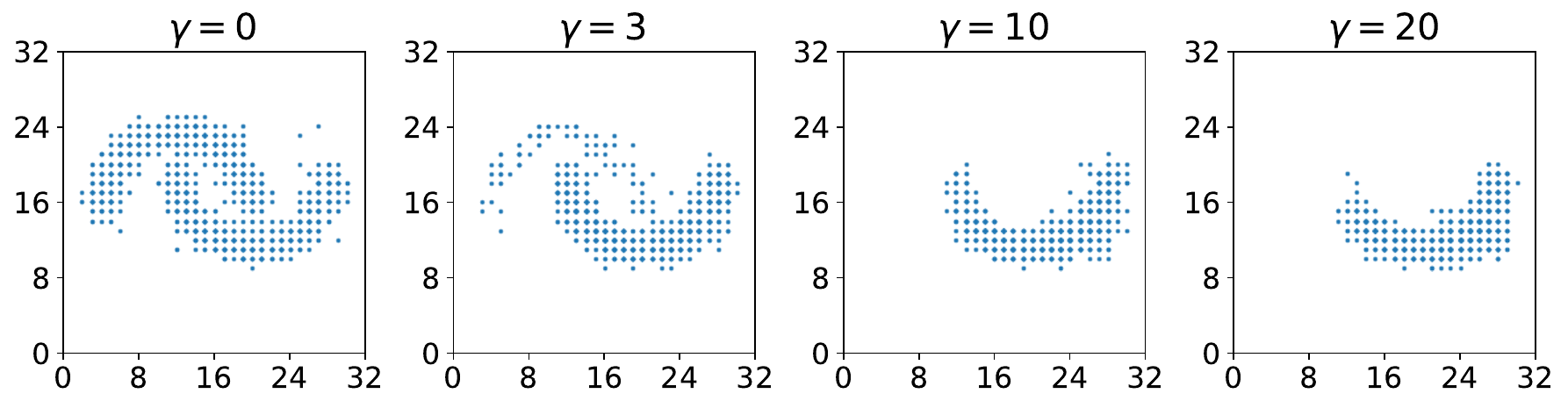}\\[-2pt]
  {\scriptsize Rate-Based (Masked, \cite{nisonoff2024unlocking})}
\end{minipage}\hfill
\begin{minipage}[t]{0.49\linewidth}
  \centering
  \includegraphics[width=\linewidth]{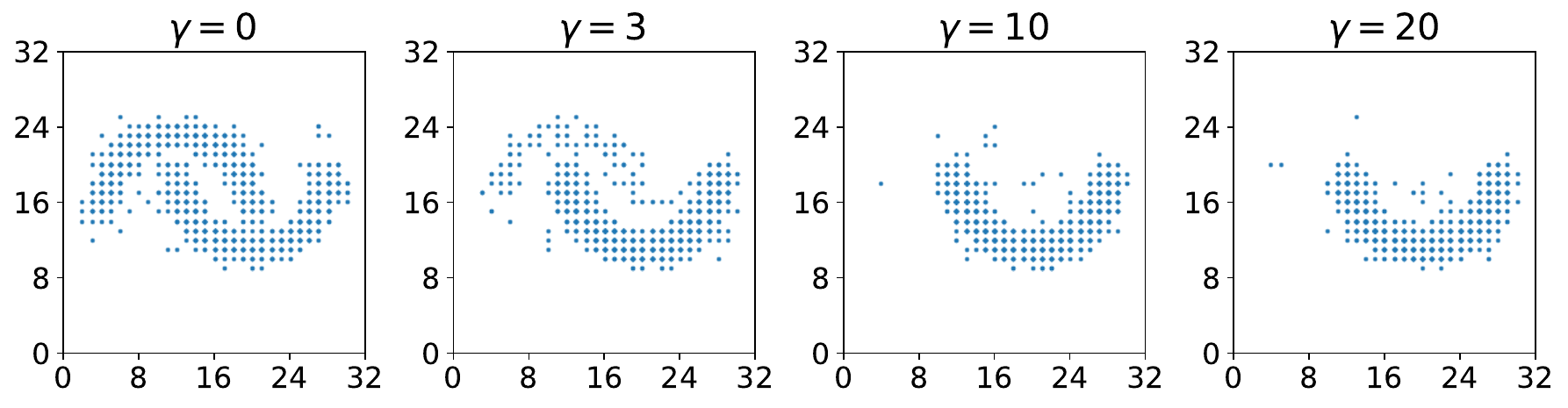}\\[-2pt]
  {\scriptsize Rate-Based (Masked, Ours)}
\end{minipage}

\vspace{3pt}

\begin{minipage}[t]{0.49\linewidth}
  \centering
  \includegraphics[width=\linewidth]{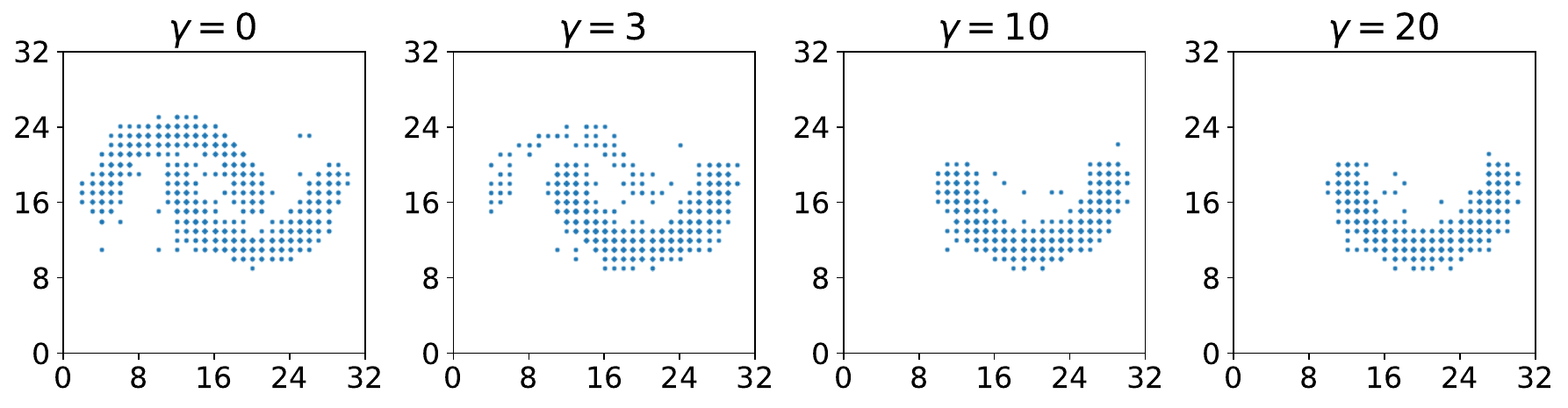}\\[-2pt]
  {\scriptsize Posterior-Based (Masked, Ours)}
\end{minipage}\hfill
\begin{minipage}[t]{0.49\linewidth}
  \centering
  \includegraphics[width=\linewidth]{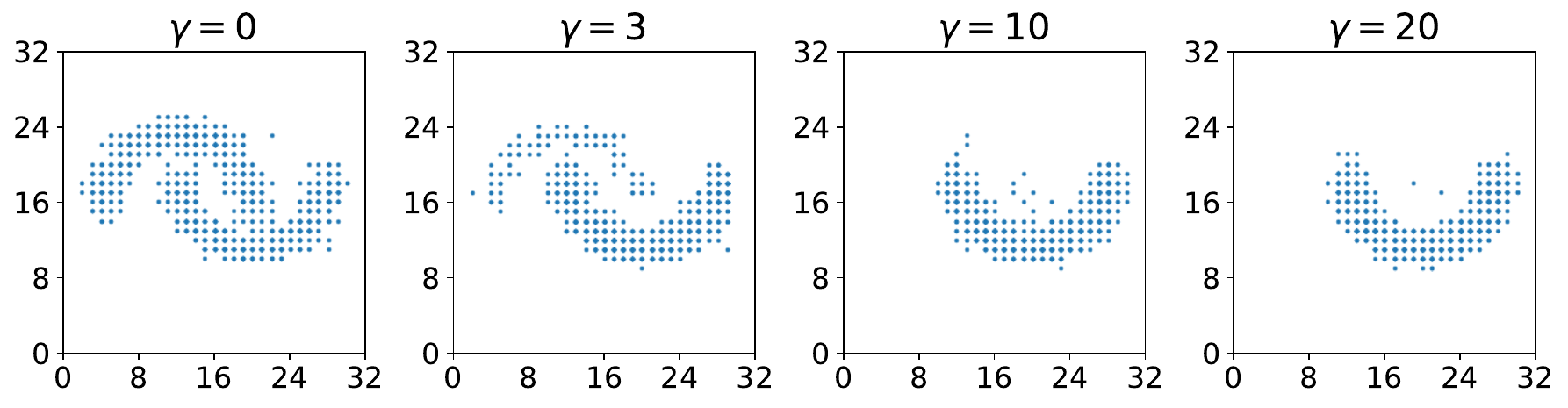}\\[-2pt]
  {\scriptsize Posterior-Based (Uniform, Ours)}
\end{minipage}
\caption{\small Comparison of sampling results with different guidance schemes in moons.}
\label{ap:image_moons}
\end{figure}

\begin{figure}[htbp]
\centering
\begin{minipage}[t]{0.49\linewidth}
  \centering
  \includegraphics[width=\linewidth]{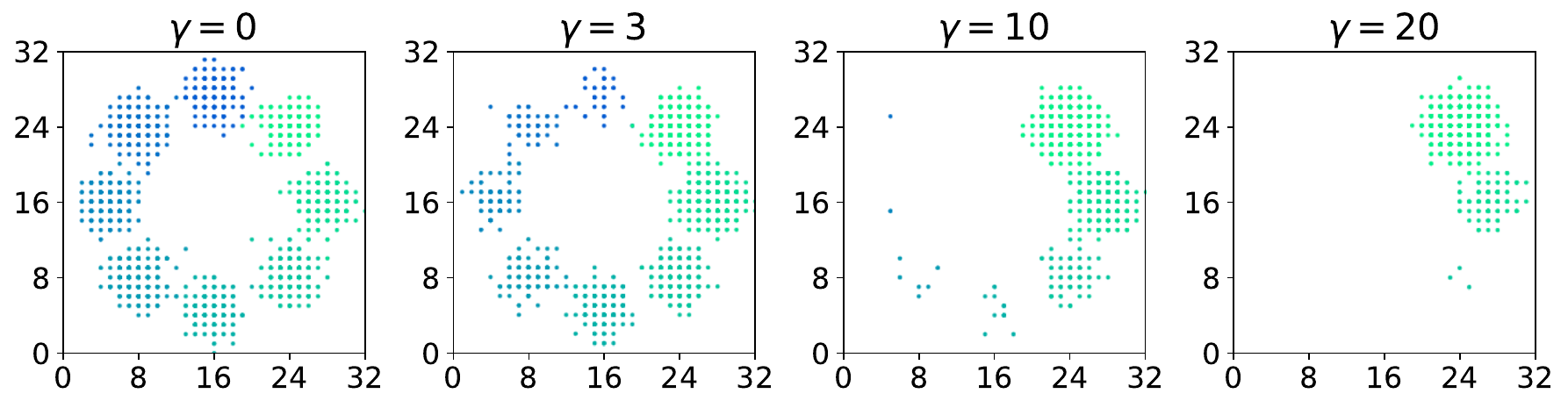}\\[-2pt]
  {\scriptsize \textbf{Ground Truth (8gaussians)}}
\end{minipage}

\vspace{3pt}

\begin{minipage}[t]{0.49\linewidth}
  \centering
  \includegraphics[width=\linewidth]{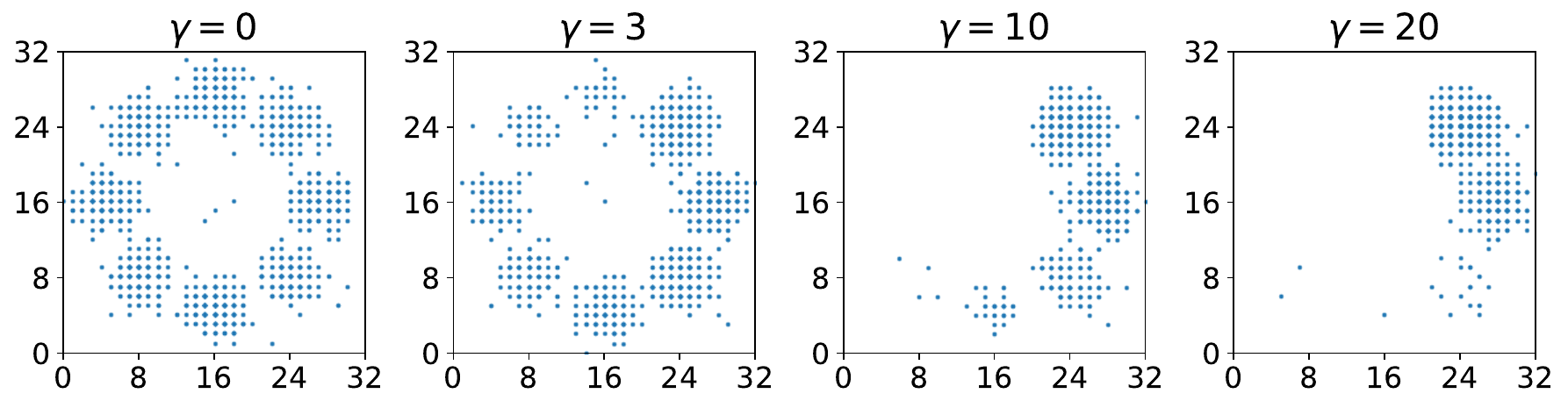}\\[-2pt]
  {\scriptsize Rate-Based (Masked, \cite{nisonoff2024unlocking})}
\end{minipage}\hfill
\begin{minipage}[t]{0.49\linewidth}
  \centering
  \includegraphics[width=\linewidth]{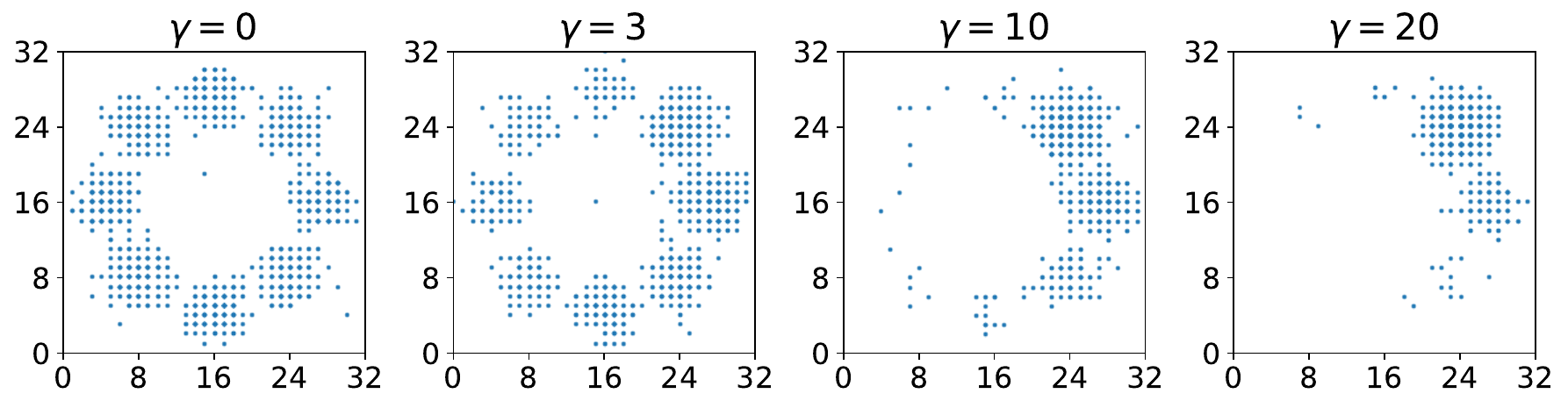}\\[-2pt]
  {\scriptsize Rate-Based (Masked, Ours)}
\end{minipage}

\vspace{3pt}

\begin{minipage}[t]{0.49\linewidth}
  \centering
  \includegraphics[width=\linewidth]{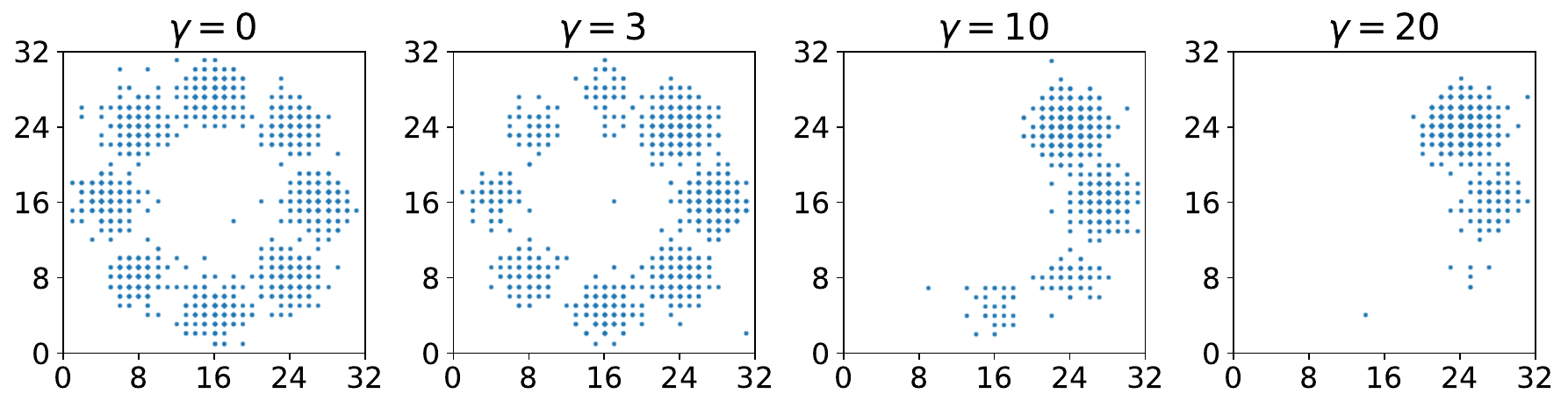}\\[-2pt]
  {\scriptsize Posterior-Based (Masked, Ours)}
\end{minipage}\hfill
\begin{minipage}[t]{0.49\linewidth}
  \centering
  \includegraphics[width=\linewidth]{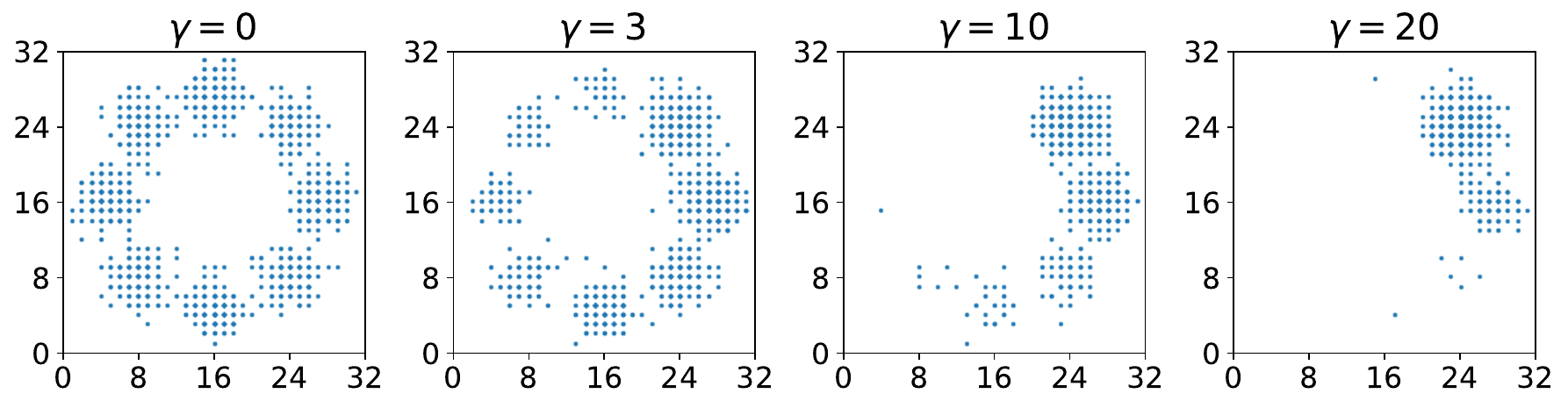}\\[-2pt]
  {\scriptsize Posterior-Based (Uniform, Ours)}
\end{minipage}

\caption{\small Comparison of sampling results with different guidance schemes in 8gaussians.}
\label{ap:8gaussians}
\end{figure}

\begin{figure}[htbp]
\centering

\begin{minipage}[t]{0.49\linewidth}
  \centering
  \includegraphics[width=\linewidth]{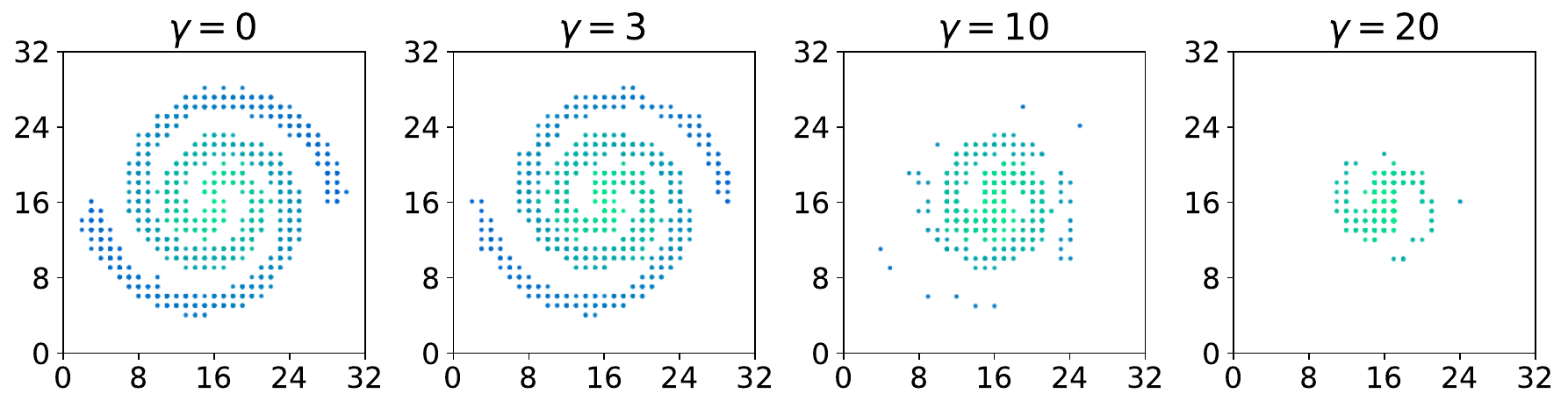}\\[-2pt]
  {\scriptsize \textbf{Ground Truth (2spirals)}}
\end{minipage}

\vspace{3pt}

\begin{minipage}[t]{0.49\linewidth}
  \centering
  \includegraphics[width=\linewidth]{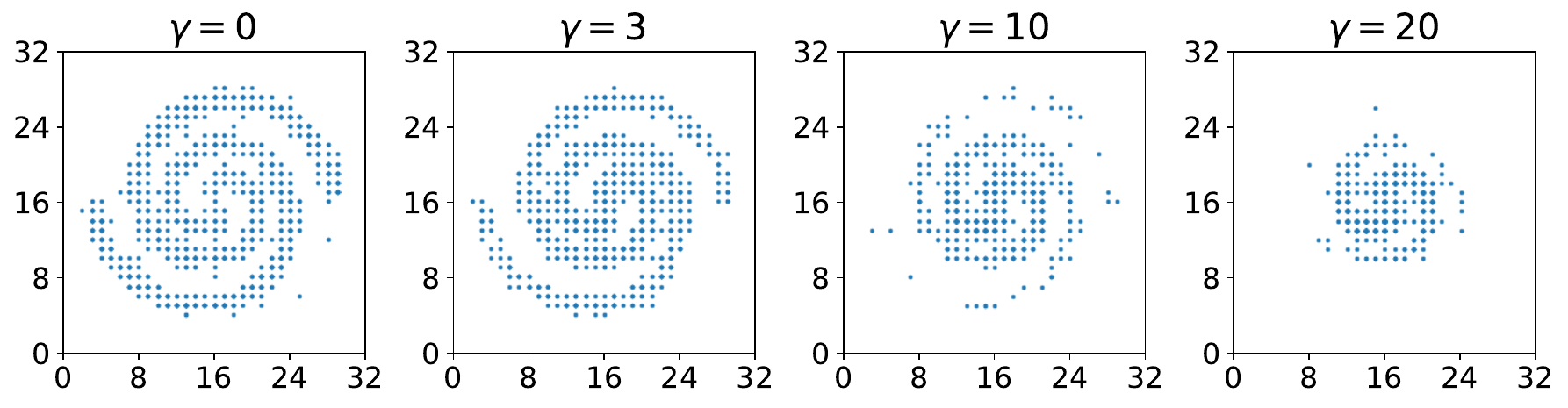}\\[-2pt]
  {\scriptsize Rate-Based (Masked, \cite{nisonoff2024unlocking})}
\end{minipage}\hfill
\begin{minipage}[t]{0.49\linewidth}
  \centering
  \includegraphics[width=\linewidth]{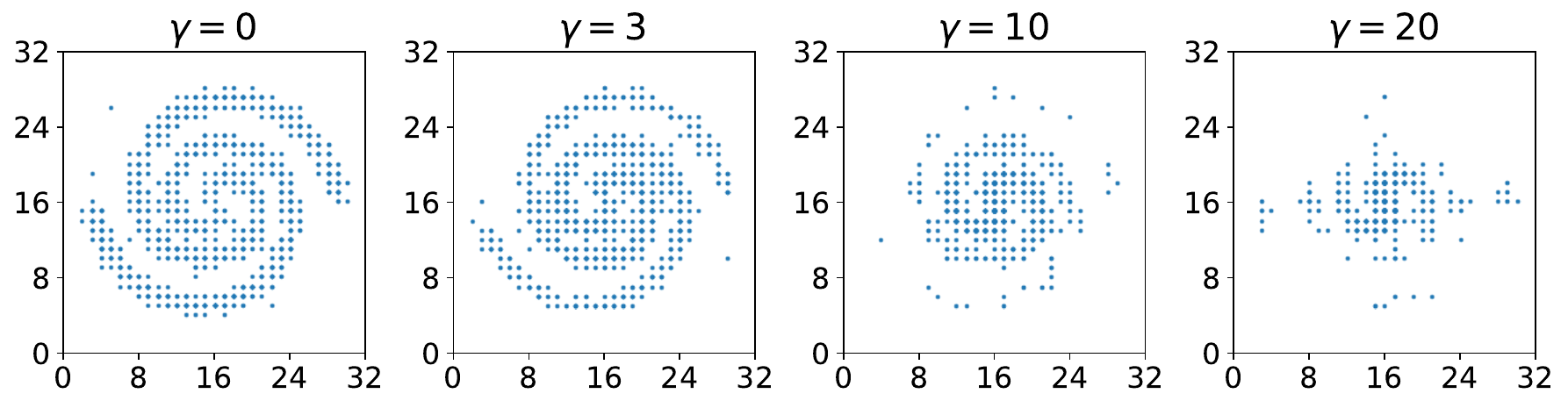}\\[-2pt]
  {\scriptsize Rate-Based (Masked, Ours)}
\end{minipage}

\vspace{3pt}

\begin{minipage}[t]{0.49\linewidth}
  \centering
  \includegraphics[width=\linewidth]{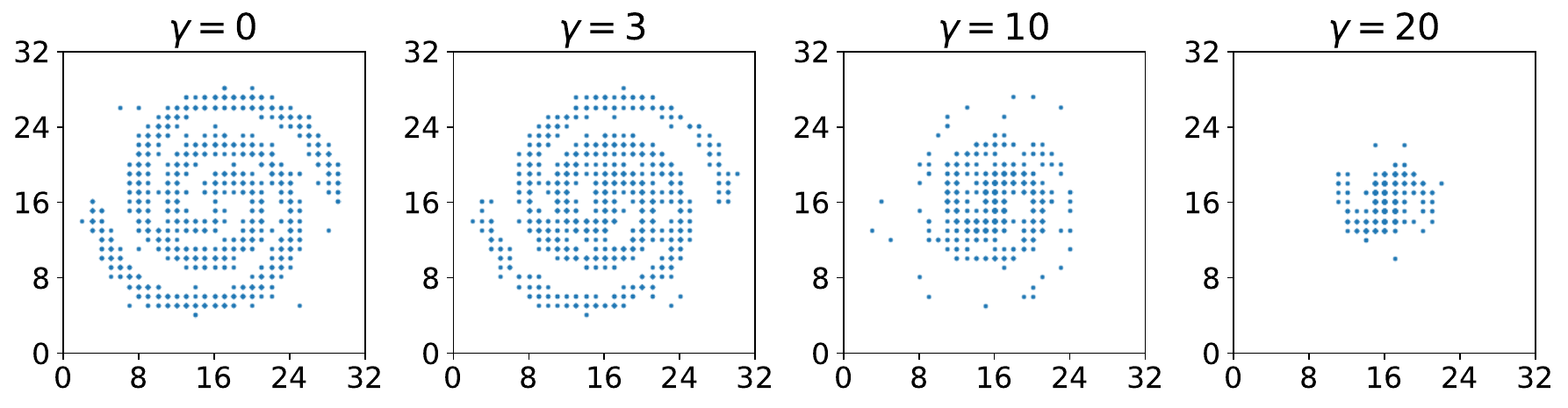}\\[-2pt]
  {\scriptsize Posterior-Based (Masked, Ours)}
\end{minipage}\hfill
\begin{minipage}[t]{0.49\linewidth}
  \centering
  \includegraphics[width=\linewidth]{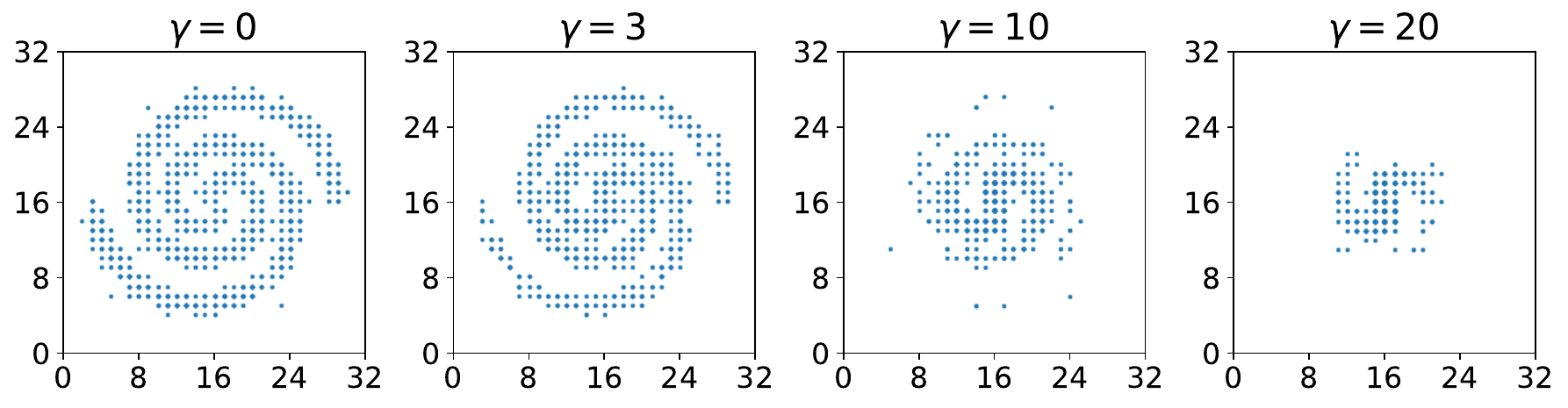}\\[-2pt]
  {\scriptsize Posterior-Based (Uniform, Ours)}
\end{minipage}

\caption{\small Comparison of sampling results with different guidance schemes in 2spirals.}
\label{ap:2spirals}
\end{figure}

\begin{figure}[htbp]
\centering
\begin{minipage}[t]{0.49\linewidth}
  \centering
  \includegraphics[width=\linewidth]{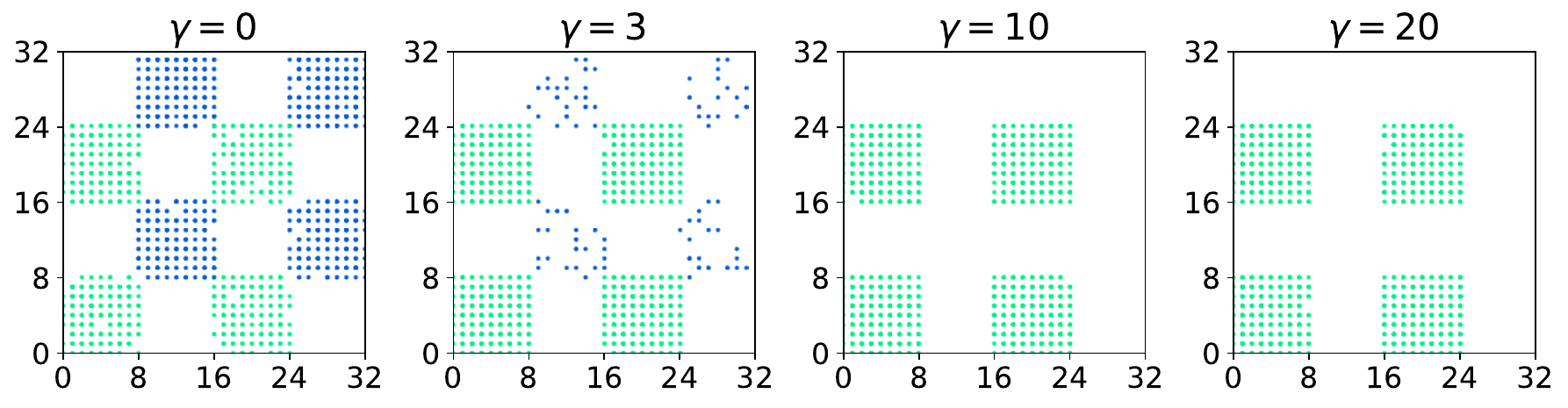}\\[-2pt]
  {\scriptsize \textbf{Ground Truth (checkerboard)}}
\end{minipage}

\vspace{3pt}

\begin{minipage}[t]{0.49\linewidth}
  \centering
  \includegraphics[width=\linewidth]{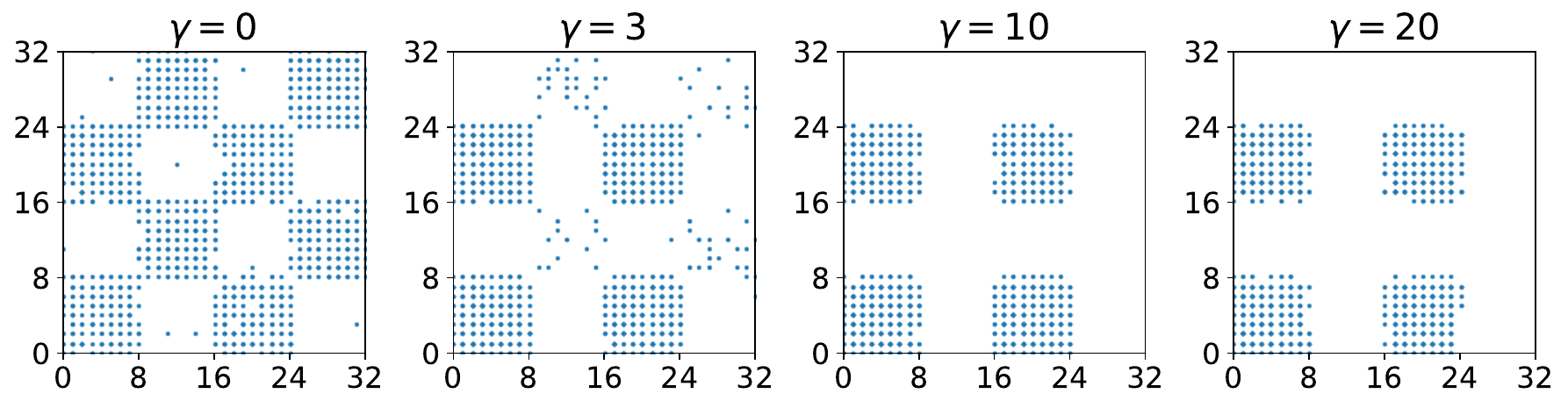}\\[-2pt]
  {\scriptsize Rate-Based (Masked, \cite{nisonoff2024unlocking})}
\end{minipage}\hfill
\begin{minipage}[t]{0.49\linewidth}
  \centering
  \includegraphics[width=\linewidth]{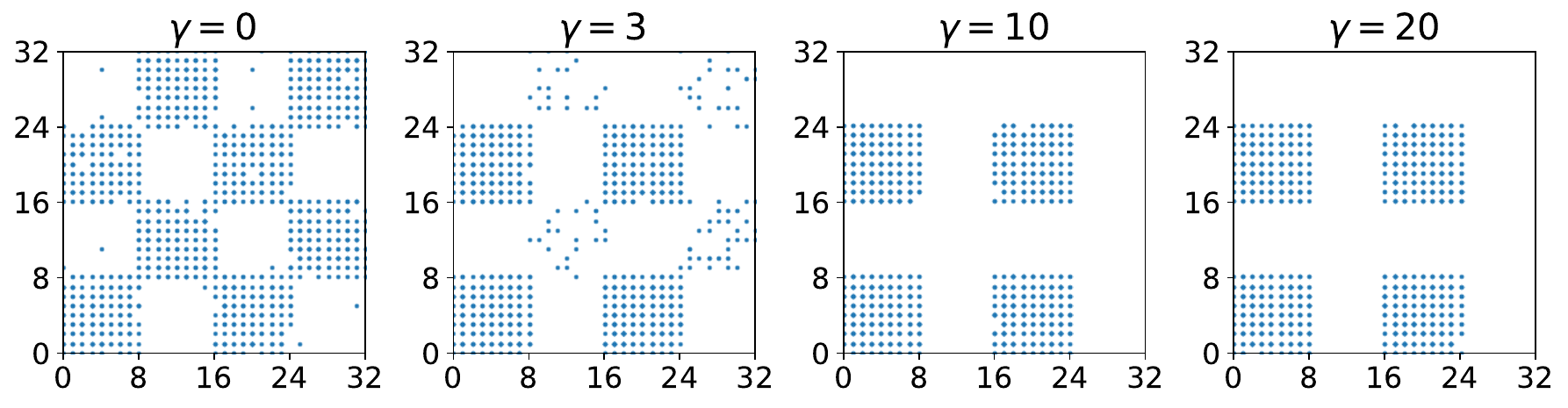}\\[-2pt]
  {\scriptsize Rate-Based (Masked, Ours)}
\end{minipage}

\vspace{3pt}

\begin{minipage}[t]{0.49\linewidth}
  \centering
  \includegraphics[width=\linewidth]{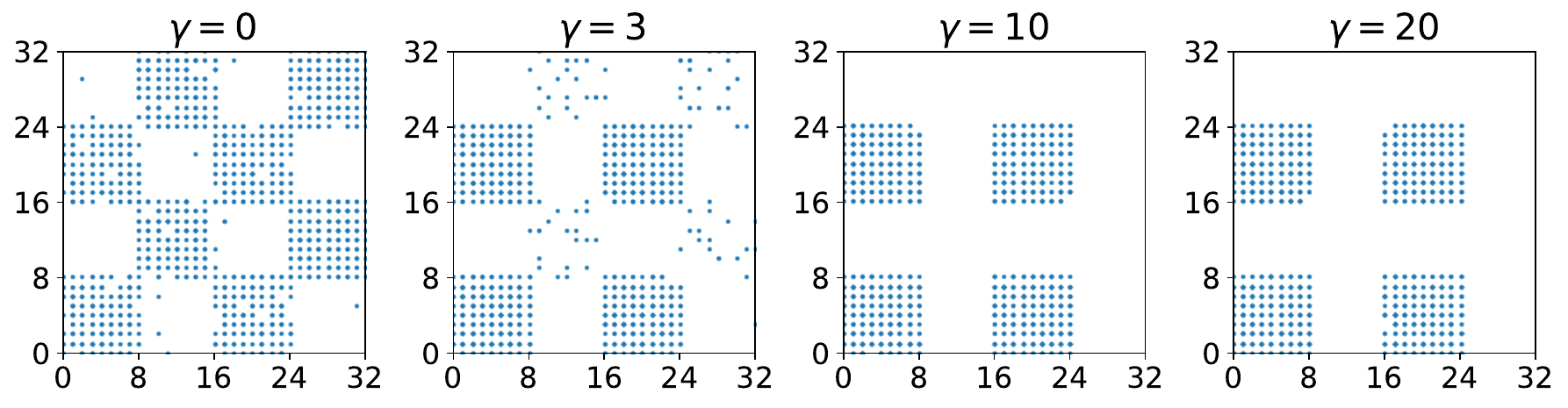}\\[-2pt]
  {\scriptsize Posterior-Based (Masked, Ours)}
\end{minipage}\hfill
\begin{minipage}[t]{0.49\linewidth}
  \centering
  \includegraphics[width=\linewidth]{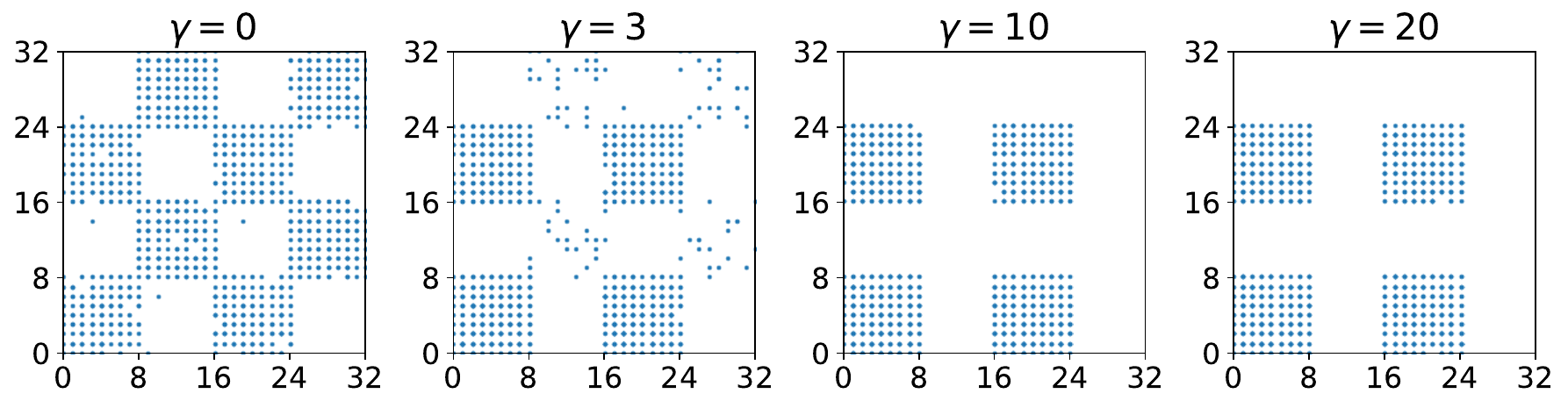}\\[-2pt]
  {\scriptsize Posterior-Based (Uniform, Ours)}
\end{minipage}

\caption{\small Comparison of sampling results with different guidance schemes in Checkerboard.}
\label{ap:checkerboard}
\end{figure}

\begin{figure}[htbp]
\centering
\begin{minipage}[t]{0.49\linewidth}
  \centering
  \includegraphics[width=\linewidth]{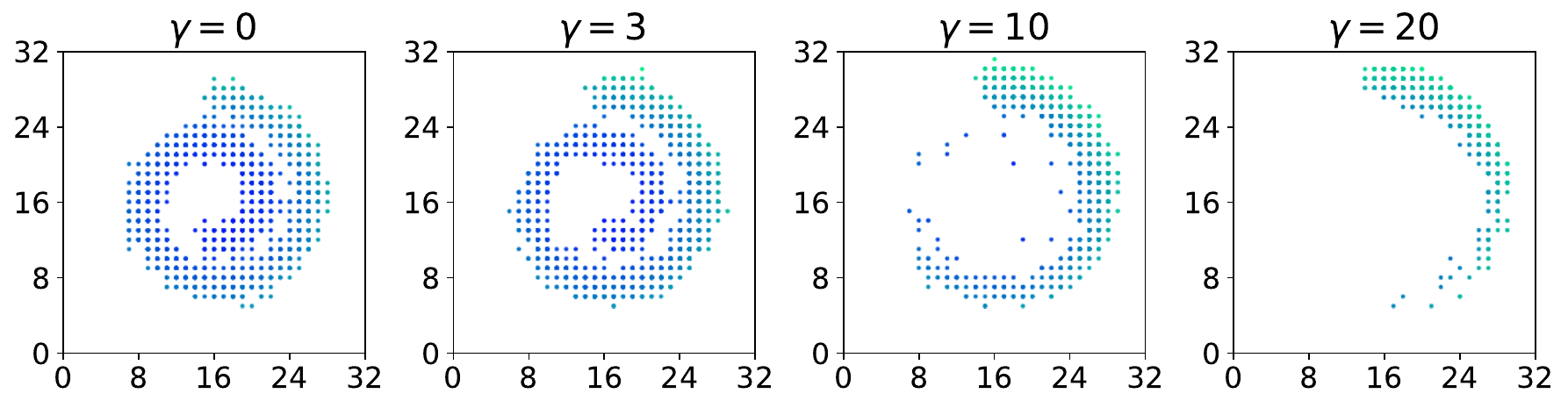}\\[-2pt]
  {\scriptsize \textbf{Ground Truth (swissroll)}}
\end{minipage}

\vspace{3pt}

\begin{minipage}[t]{0.49\linewidth}
  \centering
  \includegraphics[width=\linewidth]{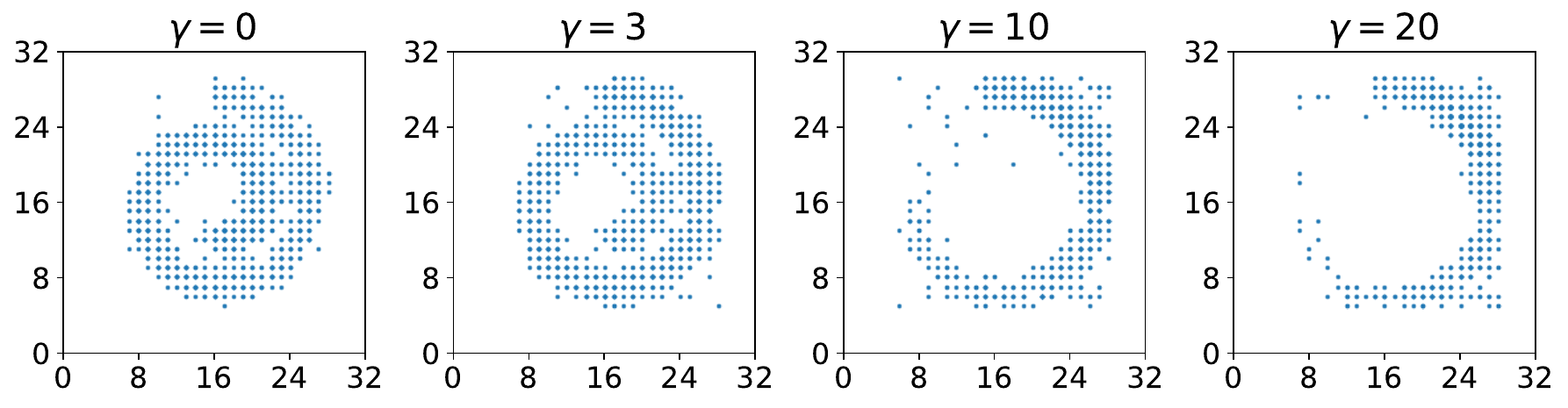}\\[-2pt]
  {\scriptsize Rate-Based (Masked, \cite{nisonoff2024unlocking})}
\end{minipage}\hfill
\begin{minipage}[t]{0.49\linewidth}
  \centering
  \includegraphics[width=\linewidth]{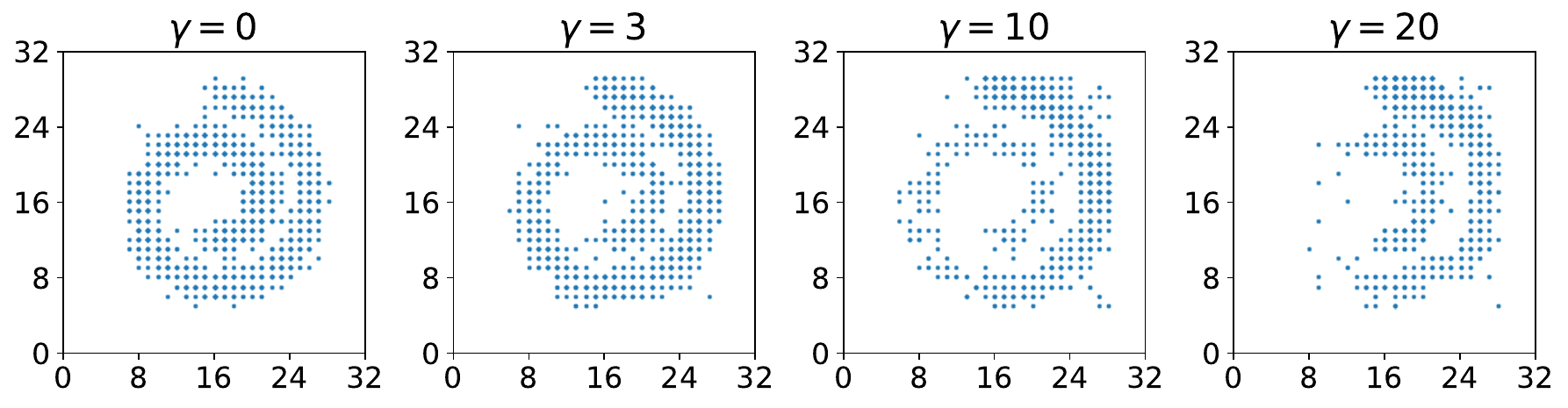}\\[-2pt]
  {\scriptsize Rate-Based (Masked, Ours)}
\end{minipage}

\vspace{3pt}

\begin{minipage}[t]{0.49\linewidth}
  \centering
  \includegraphics[width=\linewidth]{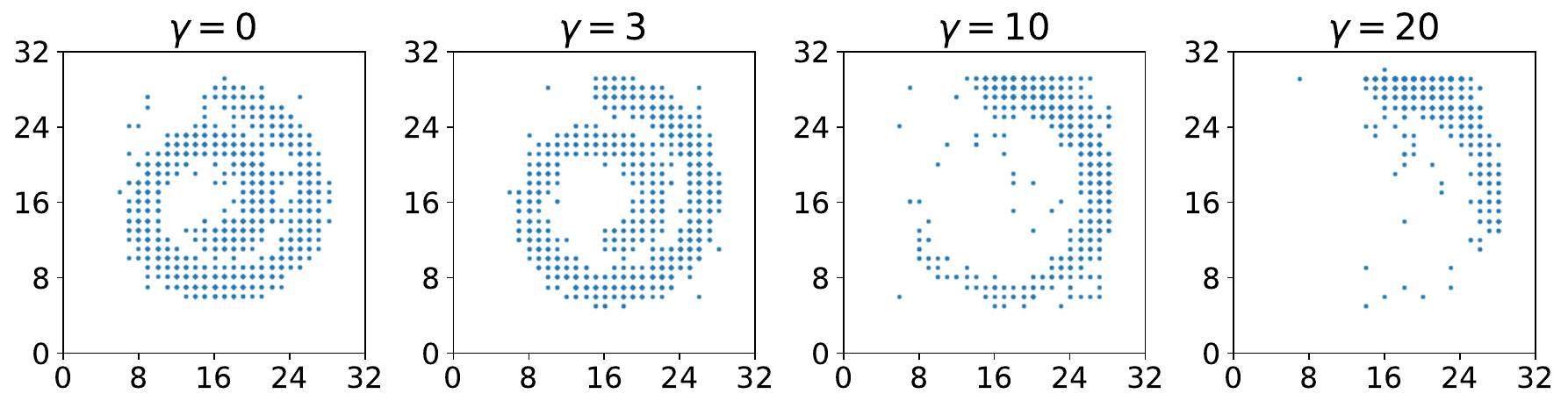}\\[-2pt]
  {\scriptsize Posterior-Based (Masked, Ours)}
\end{minipage}\hfill
\begin{minipage}[t]{0.49\linewidth}
  \centering
  \includegraphics[width=\linewidth]{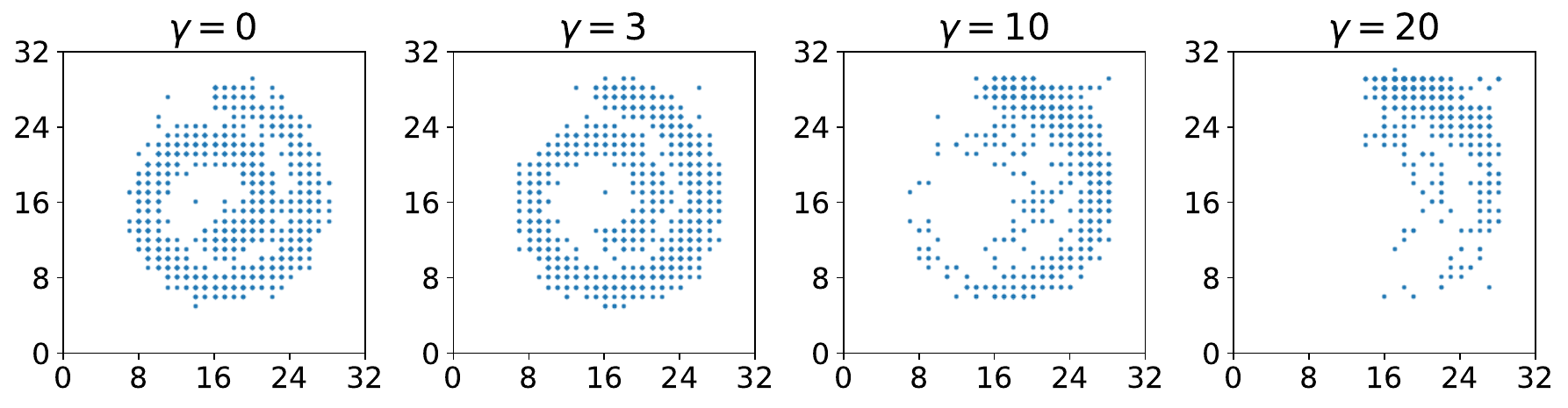}\\[-2pt]
  {\scriptsize Posterior-Based (Uniform, Ours)}
\end{minipage}

\caption{\small Comparison of sampling results with different guidance schemes in Swissroll.}
\label{ap:swissroll}
\end{figure}

\subsection{Ablation study of the regularization strength}\label{ap:abl}
In \cref{tab:abl}, we present an ablation study on the effect of the regularization strength introduced in \cref{sec:training}.
\begin{table}[htbp]\label{tab:abl}
\centering
\begin{threeparttable}
\caption{Ablation study of the regularization strength.} 
\begin{tabular}{@{}lccccccc@{}} 
\toprule
 Method & 
{\small Single Obj.} & 
{\small Two Obj.} & 
{\small Counting} & 
{\small Colors} & 
{\small Position} & 
{\small Color Attri.} & 
Overall $\uparrow$ \\
\midrule
FUDOKI   & 96.25 & 83.84 & 47.50 & 91.49 & 71.00 & 74.00 & 77.35 \\
$\eta=0.1$    & 96.25 & 90.91 & 45.00 & 91.49 & 68.00 & 71.00 & 77.11 \\
$\eta=0.5$   & 93.75 & 85.86 & 52.50 & 89.36 & 70.00 & 77.00 & 78.08 \\
$\eta=1.0$    & 92.50 & 87.88 & 52.50 & 92.55 & 67.00 & 74.00 & 77.74 \\

\bottomrule
\end{tabular}
\begin{tablenotes}
\item \textit{Note:} Results are percentages.
\end{tablenotes}
\end{threeparttable}
\end{table}

\section*{LLM Usage}
In our multimodal understanding experiments, we follow the standard practice of using LLM-as-a-judge to assess rewards and evaluations based on the question, ground-truth answer, and model response. LLMs are further used solely for polishing the writing. Importantly, no novel ideas, analyses, or discoveries are contributed by LLMs.
\end{document}